\documentclass{article}

\usepackage[preprint, nonatbib]{neurips_2026}

\usepackage[utf8]{inputenc} 
\usepackage[T1]{fontenc}    
\usepackage{hyperref}       
\usepackage{url}            
\usepackage{booktabs}       
\usepackage{amsfonts}       
\usepackage{nicefrac}       
\usepackage{microtype}      

\usepackage{graphicx} 
\usepackage{amsmath}
\usepackage{amssymb}
\usepackage{amsthm}
\usepackage{bbm}
\usepackage[table]{xcolor}
\usepackage{appendix}
\usepackage{algorithm}
\usepackage{algpseudocode}  
\usepackage{hyperref}
\usepackage{multirow}
\usepackage{multicol}
\usepackage[backend=biber, style=numeric,
            maxbibnames=99, minbibnames=99]{biblatex}
\usepackage{titlesec}
\usepackage{placeins}
\usepackage{enumitem}
\usepackage{subcaption}
\usepackage{tikz}
\usepackage{mathtools}
\usepackage{wrapfig}
\usetikzlibrary{arrows.meta, positioning, calc}

\newcommand{\pdata}{p_{\normalfont{data}}}
\newcommand{\pbase}{p_{\normalfont{base}}}
\newcommand{\PP}{\mathbb{P}}
\newcommand{\QQ}{\mathbb{Q}}
\newcommand{\EE}{\mathbb{E}}
\newcommand{\RR}{\mathbb{R}}
\newcommand{\FF}{\mathcal{F}}

\titlespacing*{\paragraph}{0pt}{0pt}{1em}

\title{ABC: \textit{A}ny-Subset Autoregression via Non-Markovian Diffusion \textit{B}ridges in \textit{C}ontinuous Time and Space}

\author{%
  Gabe Guo \thanks{Corresponding author: \texttt{gabeguo@stanford.edu}. All authors affiliated with Stanford University.} \\
  \And
  Thanawat Sornwanee \\
  \And
  Lutong Hao \\
  \And
  Elon Litman \\
  \And
  Stefano Ermon \\
  \And
  Jose Blanchet
}

\bibliography{bibliography}

\newtheorem{theorem}{Theorem}

\newtheorem{lemma}{Lemma}

\newtheorem{remark}{Remark}

\begin{document}

\maketitle

\begin{abstract}
Generating continuous-time, continuous-space stochastic processes (\textit{e.g.}, videos, weather forecasts) conditioned on partial observations (\textit{e.g.}, first and last frames) is a fundamental challenge.
Existing approaches, (\textit{e.g.}, diffusion models), suffer from key limitations: (1) noise-to-data evolution fails to capture structural similarity between states close in physical time and has unstable integration in low-step regimes; (2) random noise injected is insensitive to the physical process's time elapsed, resulting in incorrect dynamics; (3) they overlook conditioning on arbitrary subsets of states (\textit{e.g.}, irregularly sampled timesteps, future observations).
We propose \textbf{ABC}: \textit{\textbf{A}ny-Subset Autoregressive Models via Non-Markovian Diffusion \textbf{B}ridges in \textbf{C}ontinuous Time and Space}. Crucially, we model the process with \textit{one continual SDE} whose time variable and intermediate states track the \textit{real time and process states}. 
This has provable advantages: (1) the starting point for generating future states is the already-close previous state, rather than uninformative noise; (2) random noise injection scales with physical time elapsed, encouraging physically plausible dynamics with similar time-adjacent states.
We derive SDE dynamics via changes-of-measure on path space, yielding another advantage: (3) path-dependent conditioning on arbitrary subsets of the state history and/or future.
To learn these dynamics, we derive a path- and time-dependent extension of denoising score matching.
Our experiments show \textbf{ABC}'s superiority to competing methods on multiple domains, including video generation and weather forecasting.

\end{abstract}

\section{Introduction}

\paragraph{Goal} Generating continuous-time, continuous-space stochastic processes is a fundamental problem. 
Examples include: videos, turbulent flows, stock prices, EEG readings, and weather forecasts.
We are particularly interested in probabilistically infilling missing states, given only partial observations of the process, which may be irregularly sampled and even from the future (\textit{e.g.}, sporadically synced GPS readings, videos with low frame rate, first minute of a stock price). 

\paragraph{Existing Approaches Lack Inductive Bias} However, popular generative modeling techniques often do not fully leverage the underlying structure of such processes.
Instead, they shoehorn existing techniques (typically some variants of flow \cite{albergo2023stochastic, lipman2022flow} and diffusion \cite{song2020score} models) made for static data. 
For example, naive approaches concatenate the states at all times of interest, and have flow/diffusion models generate them simultaneously; this runs into memory problems on long sequences.

\textbf{\textit{Lack of State Continuity:}} To solve this, autoregressive approaches generate the next few states conditioned on previous observations. 
Typically, each state is generated by a diffusion/flow model that maps Gaussian noise to data \cite{zhang2025frame, zhang2025pretraining}. 
This ignores a crucial fact: adjacent states have minimal difference (\textit{e.g.}, in a few seconds, videos barely evolve). So, starting from uninformative noise causes many issues, \textit{e.g.}, poor structural preservation, stiff SDE integration \cite{zhou2023denoising}; also, the jump from noise to data is a physically implausible transition. Section \ref{sec:results} shows this gives incoherent generations.






\textit{\textbf{Lack of Time-Scaled Volatility:}} As a partial fix, diffusion bridges and stochastic interpolants \cite{zhou2023denoising, wang2024framebridge, lyu2024frame, li2023bbdm, liu2022let, peluchetti2023diffusion, peluchetti2023non, chen2024probabilistic} define generative SDEs whose endpoints are the states we transition between.\footnote{State-to-state flow models/ODEs would be inappropriate, as they are deterministic, but the task is fundamentally stochastic.} 
Yet, they still do not reflect the underlying process's true dynamics. Regardless of what physical times they interpolate, the SDE always starts at \textit{auxiliary (non-physical) time} $0$ and ends at auxiliary time $1$, with the \textit{same exact volatility (aka noise) schedule}. 
This is inappropriate for processes where change in state relates to physical time (\textit{e.g.}, in movies, 1$^\text{st}$ to 2$^\text{nd}$ frame changes less than 1$^\text{st}$ to 20$^\text{th}$ frame), because these models inject the same volatility (\textit{i.e.}, perturbation size) into both transitions. This causes problems, \textit{e.g.}, flickering frames (Sec. \ref{sec:results}), incorrect trajectories (Thm. \ref{thm:non_bijective_path_measure}). 

\paragraph{Continuous-Time Any-Subset Generation is Underexplored} Furthermore, previous works largely neglect handling arbitrary subsets of conditioning states. This limits applicability in cases where data was irregularly time-sampled (\textit{e.g.}, stocks only record price when a transaction is made, GPS only transmits over strong 5G signal), or has non-causal conditioning (\textit{e.g.}, first and last movie frames).

{
\setlength{\textfloatsep}{0pt}
\setlength{\intextsep}{0pt}
\begin{figure}[t]
    \centering
    \includegraphics[width=\linewidth]{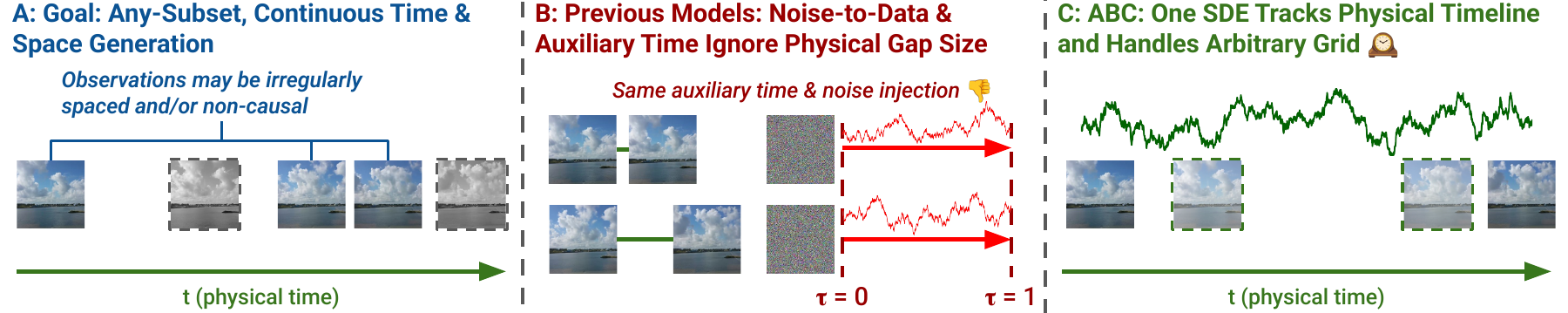}
    \setlength{\abovecaptionskip}{-10pt}
    \caption{\textbf{Overview.}}
\end{figure}
}

\paragraph{Our Approach} We propose \textbf{ABC}: \textit{\textbf{A}ny-Subset Autoregressive Models via Non-Markovian Diffusion \textbf{B}ridges in \textbf{C}ontinuous Time and Space}.
Our key insight is to model the underlying process with \textit{one continual SDE} whose time variable and intermediate states correspond to \textit{actual physical time and process states}. This makes volatility adaptive to times elapsed (Thm. \ref{thm:non_bijective_path_measure} shows baseline methods cannot attain this) and encourages state change to scale with time change.
To guarantee that the SDE matches the finite-dimensional distribution at times of interest, we set its drift via Doob h-transform (Thm. \ref{thm:new_sde}) \cite{girsanov1960transforming, doob1984classical}. 
Via Gaussian process regression \cite{williams2006gaussian}, we derive a simulation-free training objective for this drift that generalizes denoising score matching to the non-Markovian case with timestep-sensitive noising kernels (Thm. \ref{thm:path_dependent_score}) \cite{song2019generative, song2020score, vincent2011connection}. We parameterize this drift with a cross-attention transformer \cite{peebles2023scalable}; it can condition on arbitrary (irregular) subsets of observed states and the future, similar to the re-popularized \textit{any-subset autoregressive models} \cite{guo2025reviving, shih2022training}.
Our work (1) unifies diffusion and any-subset autoregressive models; (2) generalizes any-subset autoregression to continuous time and space; (3) generalizes diffusion models to non-Markovian cases.



\paragraph{Results}

Our experiments show that (1) we address the long-overlooked problem of continuous-time any-subset generative modeling; (2) time-responsive volatility is important for modeling time series; (3) data-to-data bridges beat noise-to-data generation.
Competing methods are essentially ablations on our method, up to minor derivation differences: conditional diffusion bridges (\textit{e.g.}, DDBM) \cite{zhou2023denoising, liu2022let} with fixed volatility schedules ablate (2); noise-to-data diffusion (\textit{e.g.}, SMLD) \cite{song2020score} ablates (3).
We validate ABC's advantages over these baselines on video generation and weather forecasting.
\footnote{Sampled videos: \href{https://abc-diffusion.github.io/}{https://abc-diffusion.github.io/}. 
}

Our contributions are:
\textit{\textbf{(C1)}} We propose \textbf{ABC}: an SDE generative model that generalizes autoregression to continuous time and space, and addresses the long-overlooked any-subset case.
\textit{\textbf{(C2)}} We prove that via its SDE time and state corresponding to physical process time and state, ABC's unique inductive biases of data-to-data generative dynamics and time-adaptive volatility lead to faithful modeling.
\textit{\textbf{(C3)}} We derive a scalable training objective that extends denoising score matching to non-Markovian, time-dependent cases.
\textit{\textbf{(C4)}} We validate ABC on practical tasks.

\section{Preliminaries}

\paragraph{Any-Subset Autoregressive Models}
Consider a time series $x_{t_0}, \ldots, x_{t_{L-1}}$, where $t_0 < t_1 < t_2 < \ldots < t_{L-1}$. Vanilla autoregressive (AR) models predict $p(x_{t_{L-1}}, \ldots, x_{t_i} | x_{t_{i-1}}, \ldots, x_{t_0})$, \textit{i.e.}, the future conditioned on the past. Any-subset autoregressive models predict $p(\mathbf{x}_{t_{\sigma(\geq i)}} | \mathbf{x}_{t_{\sigma(< i)}})$, where $\sigma: \{0, 1, \ldots, L-1\} \rightarrow \{0, 1, \ldots, L-1\}$ is a bijective mapping; that is, they predict the states in a permuted order that may not be causal in time \cite{guo2025reviving}.

Autoregressive models have four categories: (1) discrete space, discrete time; (2) continuous space, discrete time; (3) discrete space, continuous time; (4) continuous space, continuous time. Most existing models (\textit{e.g.}, LLMs) are discrete-space, discrete-time (one token/timestep). There are also some continuous-space, discrete time models for video \cite{ho2022video, chen2024probabilistic, zhang2025frame, zhang2025pretraining} and images \cite{li2024autoregressive}. 

Very few popular generative models (discrete or continuous space) instantiate continuous-time autoregression, in which state-to-state transitions should mimic the physical time evolution of the real-world process. Approaches that could take in continuous time inputs, but construct auxiliary generative SDEs/ODEs to bridge between states (\textit{e.g.}, noise-to-data diffusion/flow), do not count, as they only model a finite subset of process times \cite{chen2024probabilistic}.

\paragraph{Diffusion Bridge Models}
Diffusion bridge models map between two distributions $x_0 \sim p_\text{data}(x_0), x_1 \sim p_\text{data}(x_1 | x_0)$, via an SDE whose drift is trained via denoising score matching \cite{vincent2011connection}. There are multiple related formulations \cite{zhou2023denoising, liu2022let, albergo2023stochastic, peluchetti2023diffusion, peluchetti2023non}; we show Zhou et al's \cite{zhou2023denoising} (switching time convention to generate in forward time), as it generalizes denoising diffusion models \cite{song2020score, song2019generative}: 
\begin{align}
    dx_t &= \left(-a(t)x_t + \sigma(t)^2 \left[\mathbf{s_\theta}(t, x_t, x_0) - \nabla_{x_t}\text{log } p(x_0 | x_t)\right]\right)dt + \sigma(t)dB_t, x_0 \sim p_\text{data}(x_0) \label{eqn:trad_diffusion_bridge} \\
    \mathbf{\theta^*} &= \underset{\mathbf{\theta}}{\text{argmin }}\mathbb{E}_{t \sim \mathcal{U}[0, 1)}\underset{\substack{x_0, x_1 \sim p_\text{data} \\ x_t \sim p_\text{base}(x_t | x_0, x_1)}}{\mathbb{E}}\left[w(t)\|\mathbf{s_\theta}(t, x_t, x_0) - \nabla_{x_t}\text{log } p_\text{base}(x_t | x_0, x_1)\|^2\right] \label{eqn:trad_dsm}
\end{align}
where $p_\text{base}(x_t | x_0, x_1)$ is a Gaussian transition kernel, \textit{i.e.}, distribution of a diffusion bridge pinned at endpoints $x_0, x_1$ \cite{zhou2023denoising, steele2001stochastic}.
To recover the familiar noise-to-data diffusion models, set $x_0 \sim \mathcal{N}(0, I)$ \cite{song2020score, song2019generative}. 
Most diffusion models are Markovian, in that their dynamics are path-independent. A few works explore non-Markovian diffusions \cite{nobis2024generative, nobis2025fractional}, but rely on handcrafted memory kernels, \textit{e.g.}, fBM.

\paragraph{Additional Related Works} Comprehensive view of related works in Section \ref{sec:related_work}.

\section{Generative SDE for Time Series}

\begin{figure}
\centering
\begin{tikzpicture}[
    node distance=2cm,
    every node/.style={font=\normalsize},
    mycircle/.style={circle, draw=black, thick, minimum size=1cm, inner sep=1pt},
    myarrow/.style={-{Stealth[length=3mm, width=2mm]}, thick}
]
\node[mycircle] (start) at (0.025\textwidth + 0.45cm,0) {$x_0$};
\node[mycircle] (end) at (0.975\textwidth-1.45cm, 0) {$x_{t_L}$};
\draw[myarrow] (start) -- (end);
\node[fill=gray!10, draw=black, thick, rounded corners=3pt, inner sep=4pt, align=center] at ($(start)!0.5!(end)$) {
$dx_t = \left[-a(t)x_t + \sigma(t)^2 \ \smash{
    \overbrace{\mathbf{f}_\theta\!\left(t,x_t,t_{i^*+1}, \mathbf{x}_{0:i^*}\right)}^{\mathclap{\substack{
        \text{path-dependent score, minimizer of \eqref{eqn:simplified_dsm_main}}
        }}
    }
} \ \right]dt + \sigma(t)dB_t$
};
\end{tikzpicture}
\includegraphics[width=0.95\linewidth]{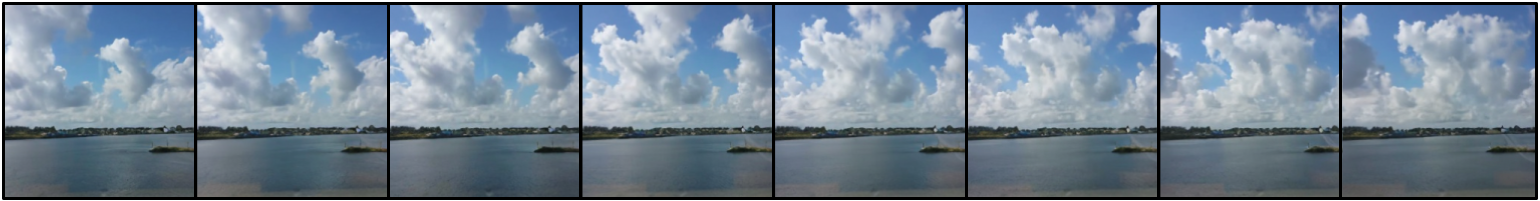}
\caption{\textbf{ABC Schematic:} The time series generation SDE from Thm. \ref{thm:new_sde} is learned via Thm. \ref{thm:path_dependent_score}. $\mathbf{f}_\theta\!\left(t,x_t,t_{i^*+1}, \mathbf{x}_{0:i^*}\right)$ is the path-dependent score (parameterized by neural net), $\mathbf{x}_{0:i^*}$ is the conditioning frame history, $t$ is the current \textit{physical} time, $t_{i^*+1}$ is the next frame's physical time.}
\end{figure}
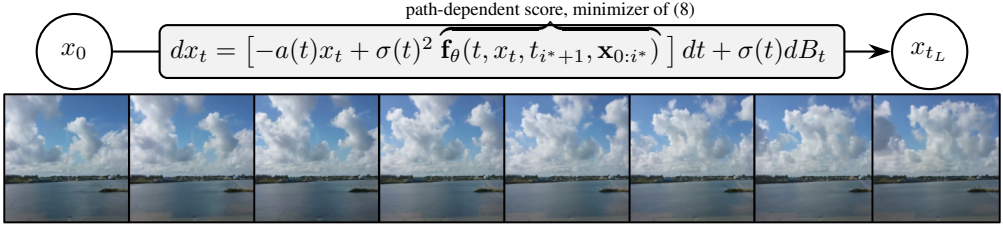

\noindent\textbf{Intutition:} To model time series, we have three major criteria:
(1) The transition to the next state should interpolate from the previous state, rather than restarting from noise. These continual transitions are faithful to the physical nature of the process.
(2) The trajectory should depend on the observed history (and future states that we have knowledge of).
(3) The change from state-to-state should scale with the amount of physical time elapsed: smaller time differences mean smaller changes in the state.

We construct a solution based on conditional diffusion bridges~\cite{zhou2023denoising}. Towards (1), diffusion bridges are specifically designed to transition from data-to-data, rather than noise-to-data. Towards (2), we can combine this with a continuous-time extension of any-subset autoregressive modeling~\cite{guo2025reviving}, and feed the state history into the diffusion bridge network via cross-attention. Finally, towards (3), we should give the conditional diffusion bridge a volatility coefficient (noise injection scale) that is proportional to the time elapsed, so that larger time gaps lead to more variation in state outcome.

\subsection{Desired Time Series}
Here, \textit{time series} means continuous-space, continuous-time stochastic process in $\mathbb{R}^d$. Abusing notation, we use $x_t$ and $X(t)$ interchangeably to denote process state at time $t$. 
WLOG, we want to generate the finite-dimensional 
state distribution $\pdata(x_{t_1}, \ldots, x_{t_L} | x_0)$ of a time series at times $0 < t_1 <\ldots < t_{L}$, given initial condition $x_{0}$. We assume $\pdata$ is only observable via samples, \textit{i.e.}, dataset. As an example, $x_{0}$ is a video's first frame, and $x_{t_1} \ldots x_{t_{L}}$ are the remaining frames (maybe irregularly spaced). Section \ref{subsec:any_subset} addresses non-causal conditioning and arbitrary time grids.

\subsection{Data-Generating Dynamics via Change-of-Measure}

Since we want our diffusion process to share the same physical time as the data generating process, we start by defining a data-independent base process:
\begin{align}
    dX_t &= -a(t)X_tdt + \sigma(t)dB_t \label{eqn:base_process_main}
\end{align}
where $a(t), \sigma(t)$ are appropriately chosen to represent the volatility structure of the data generating process, and $B$ is a Brownian motion under the base measure $\mathbb{P}$. This is a Gaussian process conditioned on the initial point $X_0$. Since this process is data-independent, it can differ from the data generating process. However, we can use path-dependent change-of-measure (via Doob-h transform) \cite{doob1984classical, girsanov1960transforming} to change the base measure $\mathbb{P}$ into $\mathbb{Q}$ \textit{such that the finite-dimensional distribution matches with that of the data}. This allows the process $X$ under the new measure $\mathbb{Q}$ to become a data-generating, non-Gaussian, non-Markov process whose diffusion time corresponds to physical process time. \footnote{The importance of non-Markov continuous-time processes is also present even when the finite-dimensional (therefore discrete-time) observation is Markov. This happens because continuous-time Markov does not admit flexible correlation structure. An example and discussions are in Figure~\ref{fig:nonmarkov}.} 

Given a process at time $t$, we know the current path history $(X_s)_{s\in [0,t]}$, and seek to find a drift. For Markov processes, we can write the drift as a function of $(t,X_t)$. In the general non-Markov case, we may require the drift to condition on the entire path. We show that it is sufficient to match the finite-dimensional distribution of the data generating process by choosing the drift to be of the form $s\left(t,X_t; \mathbf{X}_{0:i^*}\right)$, where $\mathbf{X}_{0:i^*}$ is a past observation on the time grid $(0, t_1, \dots, t_L)$: 
\begin{align}
    \mathbf{X}_{0:i^*} := \left[(0, X_0), (t_1, X_1), \dots, (t_{i^*}, X_{t_{i^*}})\right] \label{eqn:all_states}
\end{align}
with $i^* := \max\{i: t_i < t\}$. The non-causal case in Section \ref{subsec:any_subset} will allow the argument of the drift $u$ to include future times.

We denote $\pbase$ as a probability distribution of $(X_t)_{t\in [0, t_L]}$ under $\mathbb{P}$, so
\begin{align}
    p_\text{base}(x_t | x_{t_{i^*}}, x_{t_{i^*+1}}) \sim \mathcal{N}\left(\mu_{t | t_{i^*}, t_{i^*+1}}(x_{t_{i^*}}, x_{t_{i^*+1}}), \Sigma_{t | t_{i^*}, t_{i^*+1}}\right), \label{eqn:sampling_dist_noising_kernel_main}
\end{align}
where the affine function $\mu_{t | t_{i^*}, t_{i^*+1}}$ and the constant matrix $\Sigma_{t | t_{i^*}, t_{i^*+1}}$ are primitives of $(a, \sigma)$ obtained via Gaussian regression as shown in Lemma~\ref{lem:gaussian_kernels} in Appendix~\ref{appendix:B}.

\begin{theorem}
    \label{thm:new_sde}
    Under regularities in Appendix~\ref{appendix:B}, by defining a process $(X^{\text{new}}_t)_{t \in [0, t_L]}$ following
    \allowdisplaybreaks
    \begin{align}
        dX^{\text{new}}_t &= \left[-a(t)X^{\text{new}}_t + \sigma(t)^2 s\left(t,X^{\text{new}}_t; \mathbf{X}^{\text{new}}_{0:i^*}\right)\right]dt + \sigma(t)dB_t \label{eqn:new_sde}
    \end{align}
    with its initial condition $X^{\text{new}}_0 \sim p_{\text{data}}(x_0)$ and Brownian motion $B_t$, and
    where the score
    \begin{align}
        s\left(t,x_t; \mathbf{x}_{0:i^*}\right) &= -\nabla_{x_t}\log \pbase (x_t | x_{t_{i^*}})\notag\\&\hspace{14pt}
        + \nabla_{x_t}\log \int \pdata (x_{t_{i^* + 1}} | x_0, \ldots, x_{t_{i^*}}) \pbase (x_t | x_{t_{i^*}}, x_{t_{i^*+1}}) dx_{t_{i^*+1}}\label{eqn:expanded_score_main},
    \end{align}
    we have that $\left(X_0^{\text{new}}, X_{t_1}^{\text{new}}, \dots, X_{t_L}^{\text{new}}\right)$ will have the same joint distribution as $\pdata(x_0, x_{t_1}, \ldots, x_{t_L})$.
\end{theorem}
Formal proof in Appendix~\ref{appendix:B}. We use Doob's h-transform \cite{doob1984classical, girsanov1960transforming} to find $\mathbb{Q}$ that makes the finite-dimension distribution of $X$ under $\mathbb{Q}$ follow the data distribution.

This theorem suggests that we can transform any data-independent base process $X$ (up to regularities) into a new process $X^{\text{new}}$ whose finite-dimension distribution matches the data process. Intuitively, this comes from the added score term, which "pulls" the SDE towards regions of high data density at the prescribed times, conditioned on the relevant path history.
In contrast to previous works, our score term $s\left(t,X^{\text{new}}_t; \mathbf{X}^{\text{new}}_{0:i^*}\right)$ \textit{allows non-Markov structure, thereby admitting flexible correlations.}


\paragraph{Sampling Algorithm}
See Algorithm \ref{alg:inference}. To generate time series data, we simulate the SDE according to Eqn. \ref{eqn:new_sde}. 
That is, given $a(t), \sigma(t)$, and $s\left(t,x_t; \mathbf{x}_{0:i^*}\right)$, we discretize: $dt \leftarrow \Delta t$. Then, we generate Gaussian increments for $dB_t \sim \mathcal{N}(0, \Delta t)$, and plug all these variables into the equation to get $dX_t$. We iterate by adding $dX_t$ to $X_t$ and $\Delta t$ to $t$. When $t$ is sufficiently close to one of the $t_i$ of interest, we output the corresponding $X_t$ value, stopping after we cross $t_L$.
\textit{This can be interpreted as continuous-time autoregression, which is underexplored in prior work.}

\textit{A distinguishing feature of our sample generation is that the computational budget and volatility injected are determined by the physical time difference between consecutive states}. This encourages the model to make fewer changes in the "easy" transitions (\textit{i.e.}, similar states close in time), and spend more effort on "hard" transitions (\textit{e.g.}, states far apart in time, with more variation in outcomes).

\subsection{Path- and Time-Dependent Bridge Score Matching}

Next, we will show how the score function $s$ in Theorem~\ref{thm:new_sde} can be learned with neural networks.

\begin{theorem}
    \label{thm:path_dependent_score}
    Under regularities in App.~\ref{appendix:B},
    by defining the path-dependent bridge score matching loss
    \begin{align}
        \mathcal{L}_{\text{DSM}}(\hat{s}) := \displaystyle \underset{t \sim \mathcal{U}(0, t_{L-1})}{\mathbb{E}} \mathbb{E}_{\substack{\pdata(x_0, \ldots, x_{t_{i^* + 1}}) \\ \pbase(x_t | x_{t_{i^* + 1}}, x_{t_{i^*}})}} \left[\left\|
        \hat{s}\left(t,x_t; \mathbf{x}_{0:i^*}\right) - \nabla_{x_t} \log \pbase(x_{t_{i^* + 1}} | x_t)\right\|_2^2\right]\label{eqn:simplified_dsm_main},
    \end{align}
    we have that the actual score function $s$ in Theorem~\ref{thm:new_sde} minimizes $\mathcal{L}_{\text{DSM}}$.
\end{theorem}
The proof is in Appendix \ref{sec:proof_thm_2} by a bias-variance decomposition and identifying the unique minimizer as the conditional mean of $\nabla_{x_t}\log\pbase(x_{t_{i^*+1}}\mid x_t)$.

This theorem suggests we can train a score function by considering the empirical distribution and class of functions specified by neural networks. Algorithm \ref{alg:train} shows the training loop. The flexibility of $s$ allows time re-weighting, since the proof is done in a point-wise in time manner. We select weight $w(t, t_{i^*}, t_{i^*+1}) = \frac{C_t(t_{i^*+1}, t_{i^*+1})}{\Phi(t, t_{i^*+1})}$, where $C$ is defined in Eq. \ref{eqn:cov_kernel} and $\Phi$ is defined in Eq. \ref{eqn:phi}.
This loss roughly generalizes score matching losses popularly used for diffusion models \cite{song2020score, zhou2023denoising}, up to minor problem formulation differences (Appendix \ref{sec:comparison}). \textit{Distinct from prior works, our loss is path-dependent, and the target scales with the physical process time.}

\subsection{Any-Subset and Non-Causal Extension}\label{subsec:any_subset}

\textit{This subsection underpins a continuous-time generalization of any-subset autoregression \cite{guo2025reviving, shih2022training}.}

\paragraph{Any-Subset}

We've focused on finite-dimensional  distribution $p_\text{data}(x_0, x_{t_1}, \ldots, x_{t_L})$, but what if we want
$p_\text{data}\left(x_0, x_{t_1'}, \ldots, x_{t_L'}\right)$ of the same process at \textit{other} times $t_{1}' < \ldots < t_{L}'$? Similarly, we can adjust the loss  $\mathcal{L}_{\text{DSM}}(\hat{s})$ to also include different variations of time measurements while sharing the same neural network model $\hat{s} = f_\theta$. 
Note that in this multi-task setting, the next time on the grid is no longer uniquely determined by the past observation times, so we must input this desired future time $t_{i^*+1}$ to the neural network (combined with path history, this indexes the tasks):
\begin{equation}
    \mathbf{f}_{\theta^*}(t, x_t, t_{i^* + 1}, \mathbf{x}_{0:i^*})
\end{equation}
\textit{This multi-task loss over arbitrary subsets of the path is a unique feature of our method.}

\paragraph{Non-Causal}
There is sometimes "peek ahead" knowledge of future states the process will reach. We may want to generate $X^{\text{new}}$ given that some observation
\begin{align}
    \mathcal{O} = [(t_j, x_{t_j})]_{j \in J}
\end{align}
for some $J\subseteq \{1,\dots,L\}$ is specified beforehand. At the current physical time $t$, we denote the already-generated constraint $\mathcal{O}_{\le t} := \{(s,x_s)\in \mathcal{O}: s\le t\}$ and yet-to-be-generated constraint $\mathcal{O}_{> t} := \{(s,x_s)\in\mathcal{O}: s > t\}$. The constraints' values will not change with time $t$ but will change their membership: from $\mathcal{O}_{>t}$ to $\mathcal{O}_{\le t}$. Note that when the physical time $t$ is sufficiently small, $\mathcal{O}_{>t} = \mathcal{O}$, meaning that all specified constraints are yet to be generated.

The already generated $\mathcal{O}_{\le t}$ is already a part of the observed-and-generated $\mathbf{x}_{0:i^*}$, so it will not pose any problem to our framework. However, the yet-to-be-generated constraint $\mathcal{O}_{> t}$ is known but not included in the state $\mathbf{x}_{0:i^*}$ at time $t$. Therefore, we have to conditionally train our score function $s$ to incorporate this peek ahead knowledge: by changing the score from $s\left(t,X_t; \mathbf{X}_{0:i^*}\right)$ to $s\left(t,X_t; \mathbf{X}_{0:i^*}, \mathcal{O}\right)$, the score matching loss becomes
\begin{align}
    \mathcal{L}_{\text{DSM}}(\hat{s}) := \displaystyle \underset{t \sim \mathcal{U}(0, t_{L-1})}{\mathbb{E}} \mathbb{E}_{\substack{\pdata(x_0, \ldots, x_{t_L}) \\ \pbase(x_t | x_{t_{i^* + 1}}, x_{t_{i^*}})\\
    \mathcal{O} \vert x_0, \ldots, x_{t_L}}
    } \left[\left\|
    \hat{s}\left(t,x_t; \mathbf{x}_{0:i^*}, \mathcal{O}\right) - \nabla_{x_t} \log \pbase(x_{t_{i^* + 1}} | x_t)\right\|_2^2\right]\raisetag{12pt}\label{eqn:simplified_dsm_main_non_causal}.
\end{align}
Note that $\mathcal{O}$ is sampled by randomly selecting observation time indices $J \subseteq \{1, \dots, L\}$ before setting $\mathcal{O}$ accordingly. In the full expressivity setting, this is equivalent to training a score function $s_{\mathcal{O}}$ specifically for each observation $\mathcal{O}$; this ensures the approach's validity. The implementation uses a shared network, which processes $\mathbf{x}_{0:t}$ and $\mathcal{O}$ identically, as an additional layer of multitask learning.

One caveat of this approach is that for any observation $(t_j, x_{t_j}) \in \mathcal{O}$, the corresponding process $(X_t)_{t \in [0, t_L]}$ must have $X_{t_j}$ being exactly $x_{t_j}$. This makes the conditional distribution $\pdata(x_{t_1}, \dots, x_{t_L}\vert x_0, \mathcal{O})$ not have full-support, violating the regularity condition for Theorem~\ref{thm:new_sde}. For tractability, observations at future times $t_j$ can be modeled as Gaussian distributions centered on observed values $x_{t_j}$ with vanishingly small variance, instead of Dirac distributions. 

Furthermore, with the learned score (explicitly including future index $t_{i^*+1}$ for multi-task learning)
    $\mathbf{f}_{\theta^*}(t, x_t, t_{i^* + 1}, \mathbf{x}_{0:i^*}, \mathcal{O}) \approx s(t, x_t; \mathbf{x}_{0:i^*}, \mathcal{O})$,
the sampling conditioned on observation $\mathcal{O}$ via 
\begin{align}
    dX^{\text{new}}_t &= \left[-a(t)X^{\text{new}}_t + \sigma(t)^2 \mathbf{f}_{\theta^*}\left(t, X^{\text{new}}_t, t_{i^* + 1}, \mathbf{X}_{0:i^*}, \mathcal{O}\right)\right]dt + \sigma(t)dB_t
\end{align}
may fail to satisfy $X^{\text{new}}_{t_j} = x_{t_j}$ for some $(t_j, x_{t_j}) \in \mathcal{O}$.
One engineering "hack" to circumvent this is to parameterize the score function 
\begin{align}
    \mathbf{f}_{\theta^*}\left(t, X^{\text{new}}_t, t_{i^* + 1}, \mathbf{X}_{0:i^*}, \mathcal{O}\right) = \mathbf{f}_{\theta^*}^\text{res}\left(t, X^{\text{new}}_t, t_{i^* + 1}, \mathbf{X}_{0:i^*}, \mathcal{O}\right) + \nabla_{x_t} \log \pbase(x_{t_{i^{\mathcal{O}_{>t}}}}\vert x_t) \label{eqn:residual_bb}
\end{align}
where $t_{i^{\mathcal{O}_{>t}}} := \min \{s: (s,x_s)\in \mathcal{O}_{>t}\}$ is the next future constraint time to fulfill  
\footnote{Note $t_{i^{\mathcal{O}_{>t}}}$, in general, can be different from $t_{i^*+1}$: $t_{i^*+1}$ is the time for the next state we will generate in the finite-dimensional distribution, while $t_{i^{\mathcal{O}_{>t}}}$ is the time for the next future state whose value we know in advance.}.
When the current physical time $t$ approaches the next constraint time $t_{i^{\mathcal{O}_{>t}}}$, the drift $\nabla_{x_t} \log \pbase(x_{t_{i^{\mathcal{O}_{>t}}}}\vert x_t)$ dominates, and the process behaves 
like a Brownian bridge (in practice, the denominator has a small stability epsilon) whose right endpoint is our desired constraint $(t_{i^{\mathcal{O}_{>t}}}, x_{t_{i^{\mathcal{O}_{>t}}}})$.

\section{Special Cases and Ablations}\label{sec:ablation}

\paragraph{AR Conditional Diffusion Bridges}

When autoregressively\footnote{For any-subset regression, we can analogously parameterize networks as in Section \ref{subsec:any_subset}.} 
modeling $p_\text{data}(x_{t_1, \ldots, }x_{t_L} | x_0)$:
\begin{equation}\label{eqn:cond_bridge_autoregressive_factorization}
    p_\text{data}(x_{t_1, \ldots, }x_{t_L} | x_0) = \prod_{i=1}^{L}p_\text{data}(x_{t_i} | x_0, \ldots, x_{t_{i-1}}),
\end{equation}
why do we need a continual SDE $X^{\text{new}}$ as defined in Theorem~\ref{thm:new_sde}?

Since these are distribution matching problems, instead of sampling a single process ranging from the physical time of $0$ directly to the terminal $t_L$, we could instead \textit{autoregressively sample} $L$ chained conditional diffusion bridges (with their own unscaled generative time variables $\tau \in [0, 1]$) that correspond to the transitions from physical times $t_{i-1}$ to $t_{i}$. 
That is, generating
\begin{align}
    \left(\tilde{X}^{(i)}_\tau\right)_{\tau \in [0,1]}
\end{align}
for each $i \in \{1,2,\dots,L\}$, such that $\tilde{X}^{(i)}_1 \sim p_\text{data}(x_{t_{i}} | x_{0}, \ldots, x_{t_{i-1}})$, with the initial condition coming from the previous bridge: $\tilde{X}^{(i)}_0 = \tilde{X}^{(i-1)}_1$ (edge case $\tilde{X}^{(1)}_0 = x_0$).
Then, define a process $(Y_t)_{t\in [0,t_L]}$ by rescaling and concatenating the processes $\tilde{X}^{(i)}$'s together such that
\begin{align}
\label{eqn:time_shifted_bridge}
    Y_t := \tilde{X}^{(i^*)}_{\frac{t-t_{i^*}}{t_{i^*+1}-t_{i^*}}}
\end{align}
This will map the process $\tilde{X}^{(i)}$ and its corresponding bridge time $\tau \in [0,1]$ into part of the process $Y$ over the physical time $t \in [t_{i-1}, t_i]$.
The path of $Y$ is continuous (due to the initial condition of each bridge being the previous bridge's terminal state), and $Y$ agrees with $\pdata$ on finite-dimensional distribution at time $0, t_1, \dots, t_L$ (due to the autoregressive factorization).

Each of the $L$ conditional diffusion bridge processes $\tilde{X}^{(i)}, \ldots, \tilde{X}^{(L)}$ can \textit{share the same coefficient functions} $\tilde{a}(\tau), \tilde{\sigma}(\tau)$ and \textit{one common conditional network} $\tilde{\mathbf{f}}_\theta$:
\begin{equation}\label{eqn:cond_diffusion_bridge}
    d\tilde{X}^{(i)}_{\tau} = \left[-\tilde{a}(\tau)\tilde{X}^{(i)}_{\tau} + \tilde{\sigma}(\tau)^2\tilde{\mathbf{f}}_\theta\left(\tau, \tilde{X}^{(i)}_{\tau}; \mathbf{x}_{0:t_{i-1}}\right)\right]d\tau + \tilde{\sigma}(\tau)dB^{(i)}_{\tau}
\end{equation}

This fractured SDE formulation is almost the \textit{same as probabilistic forecasting with stochastic interpolants (PFI) \cite{chen2024probabilistic} and autoregressively conditioned denoising diffusion bridge models (DDBM) \cite{zhou2023denoising}} (\textit{e.g.}, volatility doesn't adapt to physical time), save for minor parameterization differences (Sec. \ref{sec:comparison}, \ref{sec:pfi_comparison}).
But, even as this model can be correct on finite-dimensional distributions, it generally cannot capture the correct volatility structure or path measure, as seen in the next theorem.

\begin{theorem}
\label{thm:non_bijective_path_measure}
    For some (generic) time grid $0, t_1, \dots, t_L$, and volatility schedule $\sigma$, for any base drift schedules $a, \tilde{a}$, volatility schedule for the fractured process $\tilde{\sigma}$, score functions $\textbf{f}_{\theta}, \tilde{\textbf{f}}_{\theta}$,
    the process $Y$ defined in~\ref{eqn:time_shifted_bridge}
    generally is not equal in distribution to the process $X^{\text{new}}$ as defined in Theorem~\ref{thm:new_sde}. Moreover, the two processes can be distinguished almost surely.
\end{theorem}
Proof in App.~\ref{sec:proof_hacky_bridge_bad}.
It uses the fact that the fractured processes $\tilde{X}^{(i)}$ share the same volatility schedule $\tilde{\sigma}$. Under generic $0,t_1, \dots,t_L$ and generic $\sigma$, we get mismatch in
quadratic variations of $Y$ and $X^{\text{new}}$. 

Theorem \ref{thm:non_bijective_path_measure} tells us that the shared choice of $\tilde{\sigma}$ across all the bridge processes $\tilde{X}_i$ (as in Equation \ref{eqn:cond_diffusion_bridge}) generally leads to different dynamics in the generation process than achieved with our ABC method's continual SDE (Equation \ref{eqn:new_sde}).
This is important when we care about the dynamics of the process in between waypoints. For instance, stock prices shouldn't have the same total volatility in the first five seconds of the day ($t_1 - t_0 = 5$ seconds) as in the next five hours ($t_2 - t_1 = 5$ hours). 

That being said, if we were to scale the volatility coefficient of each AR conditional diffusion bridge according to $\sqrt{t_{i^*+1} - t_{i^*}}$ (and \textit{crucially, train with these coefficients}), the resultant construction would be equivalent to ABC. This shows that the novel design elements of ABC, as opposed to diffusion bridge models, can be broken into two parts: (1) We construct a new continuous-time formulation of the long-underexplored any-subset autoregressive conditioning~\cite{guo2025reviving}, and introduce it to diffusion bridge models. Diffusion bridge models typically just condition on one endpoint ($O(1)$ possibilities), while we open the possibility to condition on \textit{any subset of the trajectory} ($O(2^N), N\rightarrow\infty$ possibilities, depending on the grid resolution).
(2) We emphasize the importance of physical-time-aware bridge dynamics. As Thm. 3 suggests, without time-scaled volatility, diffusion bridges would traverse physically implausible paths.

\paragraph{Noise-to-Data Diffusion}

We define conditional diffusion bridges, as in Eqn. \ref{eqn:cond_diffusion_bridge}, but \textit{with noise as initial condition}: $\tilde{X}^{(i)}_0 \sim \mathcal{N}(0, I)$. This is closely related to denoising diffusion models (DDPM, SMLD) \cite{song2020score}, which are cases of DDBM \cite{zhou2023denoising}. 
However, noise-to-data diffusion (1) has physically nonsensical dynamics, only giving meaningful samples at $t = 1$; (2) lacks the inductive bias of similarity between adjacent states; (3) has stiff numerical integration at low discretization steps.

\section{Experimental Setup}

\newcounter{rq}
\newcommand{\researchq}[2]{%
  \refstepcounter{rq}%
  \paragraph{\textit{RQ\therq \ (#2)}}\label{rq:#1}
}
\newcommand{\rqref}[1]{RQ\ref{rq:#1}}

\researchq{asarm_continuous}{Task} Conventionally, autoregressive modeling is in discrete time and space; can we generalize it to continuous time and space? We also wish to address signals with non-causal conditioning, further extending the framework to \textit{any-subset} modeling \cite{guo2025reviving, shih2022training}.

\researchq{time_scaling}{Method}
We investigate the impact of adjusting the generative SDE's volatility to the physical time elapsed, as opposed to a fixed volatility schedule that ignores real-world time scaling.

\researchq{data_to_data}{Method}
We assess how effective a generative SDE that maps from data to data is in modeling time series, as opposed to the popular SDEs that map from data to noise \cite{song2020score}.

\paragraph{Dataset}
We have two \textit{video generation} datasets: CelebV-HQ \cite{zhu2022celebv} and Sky Timelapse \cite{zhang2020dtvnet}; see Section \ref{subsec:video_dataset}. 
We evaluate \textit{weather modeling} on the SEVIR VIL dataset \cite{veillette2020sevir}; see Section \ref{sec:sevir_additional_details}. 

\paragraph{Training Details}
For each clip, we sample the number of conditioning frames $L \sim \mathcal{U}[1, 16]$. After selecting $L$, the first frame is always used as conditioning, and all the other frames have uniform probability of getting picked: this gives irregularly sampled times. See Section \ref{subsec:training_details} and Algorithm \ref{alg:train}.

\paragraph{Inference Settings}
Towards \rqref{asarm_continuous} (continuous-time any-subset modeling), we test video infilling at multiple subsampled time resolutions: 32 and 16 evenly spaced frames per SDE time unit (validating that continuous-time formulation aids in variable resolution modeling). Towards non-causal conditioning, we prompt every $K$th frame, where $K \in \{4, 8, 16\}$. As the frames get filled in, the model is conditioned on its own generations (except that we teacher-force the conditioning frames), also adding a causal aspect. We conduct with $\{250, 500\}$ discretization steps of the SDE over $t \in [0, 1]$ (where $t = 1$ has the last frame). Section \ref{sec:sevir_additional_details} has weather experiment details.
See Algorithm \ref{alg:inference}.

\paragraph{Metrics} For video modeling, we use FVD \cite{fvd_unterthiner2018towards, ge2024content} and FID \cite{fid_heusel2017gans} metrics. FVD measures the temporal consistency of videos, and FID measures the per-frame quality of videos. See Section \ref{subsec:evaluation_details}. 
For SEVIR, we report pixelwise MAE, MSE, and RMSE, together with thresholded CSI (critical success index), POD (probability of detection), false alarm rate, and bias at the standard VIL thresholds $\{16,74,133,160\}$ \cite{veillette2020sevir}. 
See Section \ref{sec:sevir_additional_details}.

\paragraph{Model Architecture}
Our architecture extends DiT \cite{peebles2023scalable}. See Section \ref{sec:architecture} for more details.

\paragraph{Baselines}
The point of the evaluation is to answer \rqref{asarm_continuous}, \rqref{time_scaling}, and \rqref{data_to_data}. 
We construct carefully controlled comparison against Section \ref{sec:ablation}'s baselines, which also correspond to competitive probabilistic modeling methods from recent literature. 
Comparison to autoregressively conditioned diffusion bridges answers \rqref{time_scaling}; and is basically DDBM \cite{zhou2023denoising, albergo2023stochastic, peluchetti2023diffusion, peluchetti2023non, liu2022let} and PFI \cite{chen2024probabilistic} (see Sec. \ref{sec:comparison}, \ref{sec:pfi_comparison}).
Comparison to noise-to-data diffusion answers \rqref{data_to_data}; and is basically DDPM/SMLD \cite{song2020score, song2019generative, ho2020denoising}. 
All methods take the same conditioning and use the same architecture, so the computational complexity is the same; they also have the same training setup.
See Section \ref{subsec:sde_parameter_selection}. 

\section{Results}\label{sec:results}

\begin{wrapfigure}{R}{0.51\textwidth}
    \vspace{-\intextsep}
    \centering
    \includegraphics[width=0.49\textwidth]{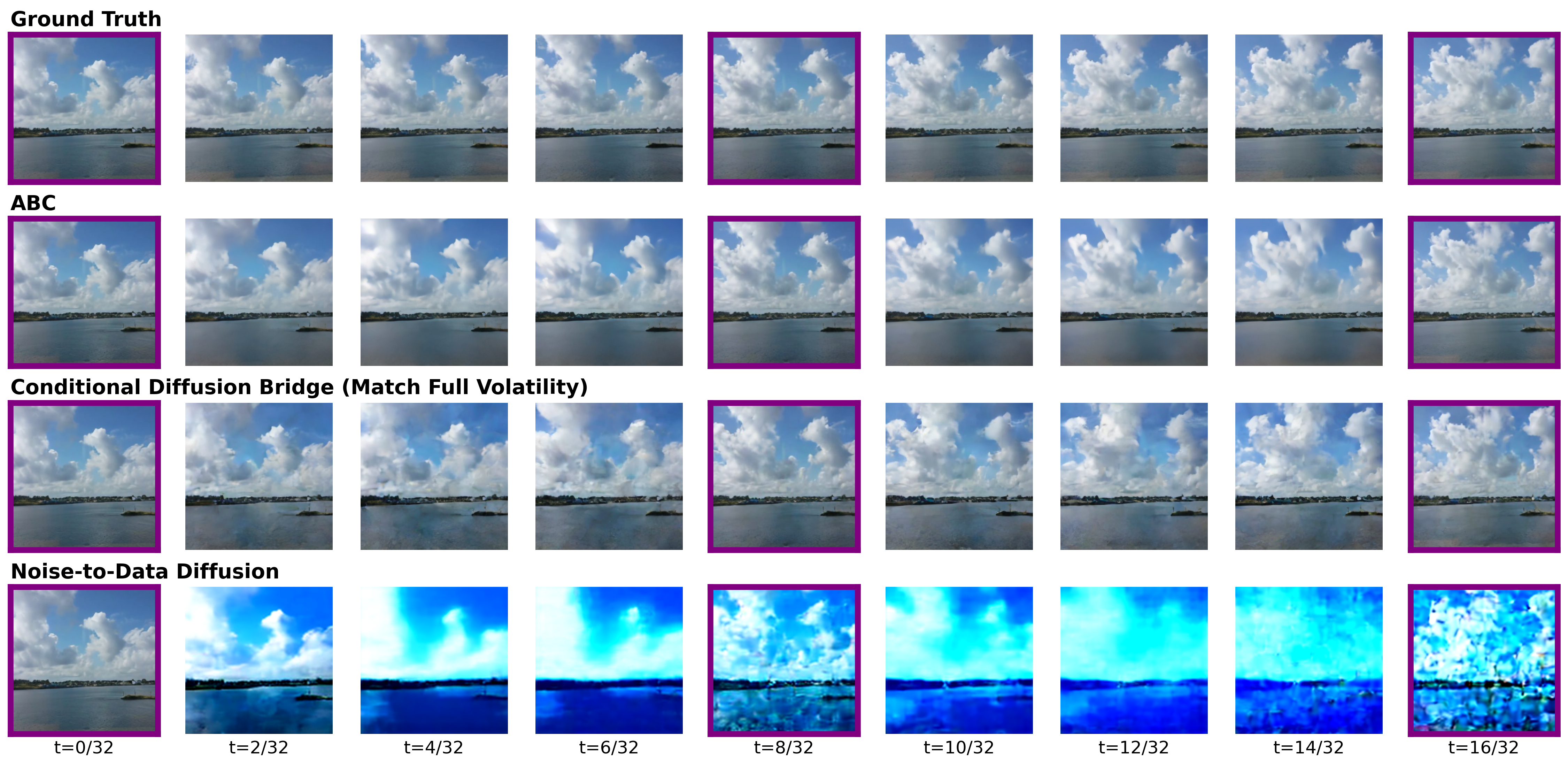}
    \caption{\textbf{Comparison:} ABC beats conditional diffusion bridges (with equal $\sigma$) and noise-to-data diffusion models. Generated with 250 steps, and conditioning every eight (plus final) frames, although the model can overwrite prompts (to assess adherence). \textbf{Visit \texttt{\href{https://abc-diffusion.github.io/}{https://abc-diffusion.github.io/}} or supplementary material.}}\label{fig:sky_timelapse_comparison}
    \vspace{-\intextsep}
\end{wrapfigure}

\textit{Our experiments show that our ABC models outperform diffusion models at generating time series, due to our inductive biases of data-to-data transitions and time-adaptive computation.}

\begin{figure}
    \centering
    \includegraphics[width=0.45\linewidth]{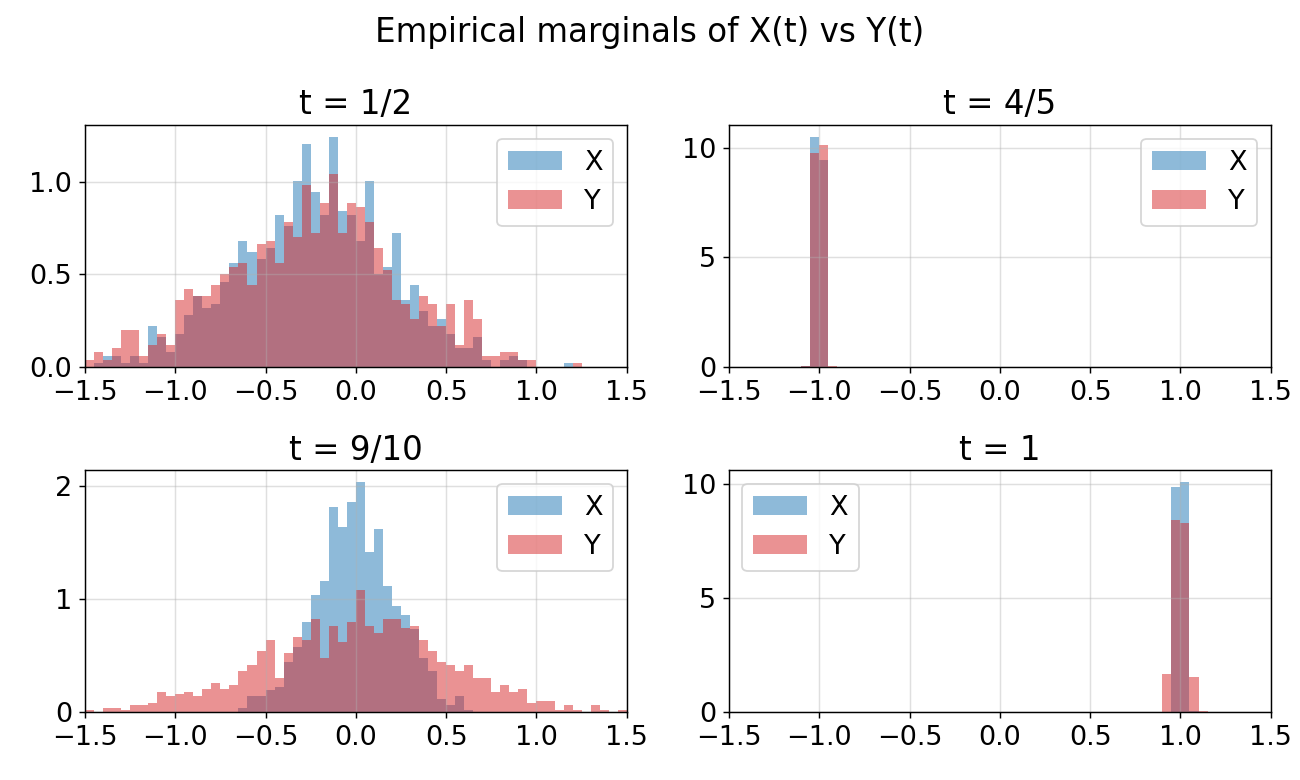}
    \includegraphics[width=0.45\linewidth]{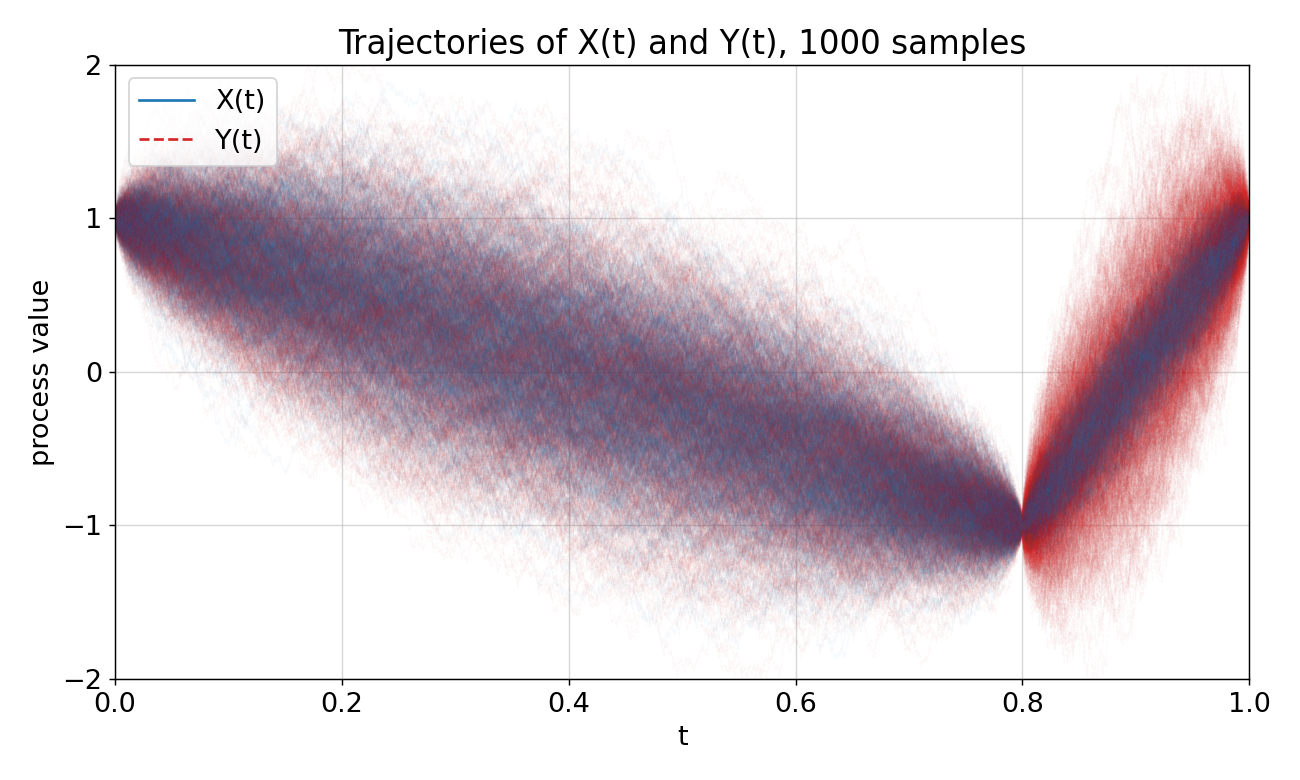}
    \caption{\textbf{Insufficiency of Autoregressive Conditional Diffusion Bridges:} See Section \ref{sec:stitched-vs-conditioned}.}\label{fig:insufficient_stiched_bridge}
\end{figure}

\paragraph{Toy Experiment}

Fig. \ref{fig:insufficient_stiched_bridge} illustrates Thm. \ref{thm:non_bijective_path_measure} on Brownian motion pinned at $B(4/5) = -1, B(1) = B(0) = +1$. Chained diffusion bridges $Y(t)$, while capturing finite-dimensional marginals, have incorrect quadratic variation. Modeling with one continual SDE correctly captures the dynamics.


\paragraph{Video Modeling}

\textbf{\textit{View generated videos at \textup{\texttt{\href{https://abc-diffusion.github.io/}{https://abc-diffusion.github.io/}}}.}} See Tables \ref{tab:sky_NON_causal_fvd} and \ref{tab:celebvhq_NON_causal_fvd} for FVD scores. 
Judging by FVD on non-causal generation, ABC is generally superior to the baselines across these video modeling datasets. The only baselines that are sometimes competitive are certain settings of the conditional diffusion bridge, depending on the hyperparameters. The noise-to-data diffusion model is not competitive at all (as expected from its numerical stiffness and implausible dynamics), \textit{validating our hypothesis from \rqref{data_to_data}}.

However, FVD mainly assesses temporal consistency, but underemphasizes per-frame quality. FID scores in the appendix (Tables \ref{tab:celebvhq_NON_causal_fid}, \ref{tab:sky_NON_causal_fid}), indicate that conditional diffusion bridge models (\textit{e.g.}, Sky-Timelapse $\sigma=0.3$) which beat ABC on FVD are often not competitive on FID (per-frame quality); creating flickering \href{https://abc-diffusion.github.io/}{videos}. In the overall ranking across FVD and FID (Tables \ref{tab:celebvhq_non_causal_overall_avg_rank}, \ref{tab:sky_non_causal_overall_avg_rank}), \textit{ABC clearly prevails, answering \rqref{time_scaling} in the affirmative: time-adaptive volatility is important}.

We also conduct experiments on strictly causal (forward in time) prompting (Section \ref{sec:causal_sky_timelapse}, \ref{sec:celebv_hq_causal}); this is not our training's focus, but ABC nonetheless achieves Pareto optimality. This \textit{flexibility across causal and non-causal prompting further confirms \rqref{asarm_continuous}}. 
Section \ref{sec:more_results} has additional results: ablation on full-path conditioning, sensitivity to hyperparameter choice, FID scores, causal evaluations, scaling, generalization to time grids unseen during training, constraint adherence, irregular time grid sampling, diversity analysis, runtime performance profiling.

\begin{table}[htbp]
\centering\tiny
\caption{\textbf{Sky-timelapse Non-Causal FVD ($\downarrow$):} On 1360 videos. BB stands for Brownian Bridge (Eq. \ref{eqn:residual_bb}). $\sigma$ is the coefficient of $dB(t)$ in the SDE. Cos: ($\alpha, \epsilon$) and Exp: ($B, K$) are the parameters to the volatility schedules for noise-to-data diffusion (Sections \ref{subsec:exponential_decay_volatility}, \ref{subsec:cosine_decay_volatility}). We report mean and std results across three random seeds.}
\label{tab:sky_NON_causal_fvd}
\resizebox{\textwidth}{!}{\begin{tabular}{lrrrrrrrrrrrr}
\toprule
 & \multicolumn{6}{c}{Num Frames = 32} & \multicolumn{6}{c}{Num Frames = 16} \\
\cmidrule(lr){2-7} \cmidrule(lr){8-13}
 & \multicolumn{2}{c}{pin\_every = 4} & \multicolumn{2}{c}{pin\_every =8} & \multicolumn{2}{c}{pin\_every = 16} & \multicolumn{2}{c}{pin\_every = 4} & \multicolumn{2}{c}{pin\_every = 8} & \multicolumn{2}{c}{pin\_every = 16} \\
\cmidrule(lr){2-3} \cmidrule(lr){4-5} \cmidrule(lr){6-7} \cmidrule(lr){8-9} \cmidrule(lr){10-11} \cmidrule(lr){12-13}
\textbf{Method} \ \ \ \ \ \ Steps & \multicolumn{1}{c}{250} & \multicolumn{1}{c}{500} & \multicolumn{1}{c}{250} & \multicolumn{1}{c}{500} & \multicolumn{1}{c}{250} & \multicolumn{1}{c}{500} & \multicolumn{1}{c}{250}& \multicolumn{1}{c}{500} & \multicolumn{1}{c}{250} & \multicolumn{1}{c}{500} & \multicolumn{1}{c}{250} & \multicolumn{1}{c}{500} \\
\midrule
\rowcolor{gray!20} \multicolumn{13}{l}{\textit{ABC}} \\
\quad No BB: $\sigma$=0.4 & \textbf{61.8} $\pm$ 0.2 & \textbf{\textit{58.1}} $\pm$ 0.1 & \textbf{162.5} $\pm$ 0.9 & \textbf{137.9} $\pm$ 0.4 & \textbf{\textit{346.5}} $\pm$ 0.2 & \textbf{\textit{283.2}} $\pm$ 0.7 & \textbf{\textit{123.2}} $\pm$ 0.5 & 115.1 $\pm$ 0.5 & \textbf{\textit{234.1}} $\pm$ 0.2 & \textbf{207.2} $\pm$ 0.3 & 313.4 $\pm$ 1.3 & 278.1 $\pm$ 1.1 \\
\quad W/ BB: $\sigma$=0.4 & \textbf{\textit{63.3}} $\pm$ 0.1 & \textbf{57.8} $\pm$ 0.2 & \textbf{\textit{166.6}} $\pm$ 0.5 & \textbf{\textit{140.9}} $\pm$ 0.4 & 377.7 $\pm$ 0.1 & 308.5 $\pm$ 3.0 & 131.9 $\pm$ 0.7 & 123.7 $\pm$ 1.0 & 235.7 $\pm$ 1.2 & 210.8 $\pm$ 0.5 & \textbf{\textit{308.6}} $\pm$ 2.2 & 276.8 $\pm$ 1.3 \\
\rowcolor{gray!20} \multicolumn{13}{l}{\textit{Conditional Diffusion Bridge}} \\
\quad $\sigma=0.071$ & 81.9 $\pm$ 0.2 & 72.7 $\pm$ 0.5 & 236.4 $\pm$ 0.7 & 208.2 $\pm$ 0.6 & 469.2 $\pm$ 1.1 & 437.4 $\pm$ 2.0 & 218.7 $\pm$ 0.2 & 193.3 $\pm$ 0.8 & 436.2 $\pm$ 1.9 & 386.4 $\pm$ 0.8 & 541.0 $\pm$ 1.3 & 482.3 $\pm$ 1.9 \\
\quad $\sigma=0.1$ & 72.3 $\pm$ 0.2 & 65.9 $\pm$ 0.7 & 222.2 $\pm$ 0.5 & 199.9 $\pm$ 0.1 & 453.5 $\pm$ 1.5 & 434.2 $\pm$ 0.3 & 164.9 $\pm$ 0.3 & 147.0 $\pm$ 0.4 & 373.7 $\pm$ 1.5 & 326.2 $\pm$ 1.2 & 450.9 $\pm$ 0.8 & 404.9 $\pm$ 2.0 \\
\quad $\sigma=0.3$ & 102.2 $\pm$ 0.2 & 72.7 $\pm$ 0.3 & 193.5 $\pm$ 0.3& 159.0 $\pm$ 0.3 & \textbf{303.4} $\pm$ 1.3 & \textbf{277.7} $\pm$ 0.5& \textbf{121.4} $\pm$ 0.8 & \textbf{111.8} $\pm$ 1.0 & \textbf{217.6} $\pm$ 1.2 & \textbf{\textit{208.9}} $\pm$ 1.4 & \textbf{267.1} $\pm$ 1.6& \textbf{249.0} $\pm$ 0.6 \\
\quad $\sigma=0.4$ & 142.8 $\pm$ 0.6 & 93.9 $\pm$ 0.7 & 252.2 $\pm$ 1.4& 190.5 $\pm$ 0.6 & 375.3 $\pm$ 2.4 & 308.5 $\pm$ 4.2 & 136.4 $\pm$ 2.0& \textbf{\textit{114.1}} $\pm$ 0.5 & 250.9 $\pm$ 1.4 & 226.8 $\pm$ 2.4& 316.0 $\pm$ 3.7 & \textbf{\textit{267.2}} $\pm$ 1.5 \\
\quad $\sigma=0.6$ & 223.3 $\pm$ 0.7 & 149.0 $\pm$ 2.1 & 409.6 $\pm$ 2.7 & 298.5 $\pm$ 1.3 & 592.1 $\pm$ 3.5 & 469.8 $\pm$ 2.3 & 206.6 $\pm$ 0.6 & 155.0 $\pm$ 1.5 & 356.5 $\pm$ 2.2 & 285.5 $\pm$ 1.3 & 506.8 $\pm$ 2.0 & 380.2 $\pm$ 2.5 \\
\rowcolor{gray!20} \multicolumn{13}{l}{\textit{Noise-to-Data Diffusion}} \\
\quad Cos: (3.0, 0.04) & 509.2 $\pm$ 0.9 & 216.0 $\pm$ 1.0 & 822.0 $\pm$ 2.4 & 449.5 $\pm$ 1.3 & 761.4 $\pm$ 2.5 & 559.3 $\pm$ 3.3 & 315.5 $\pm$ 1.7 & 160.3 $\pm$ 0.3 & 508.1 $\pm$ 0.9 & 270.2 $\pm$ 0.7 & 542.1 $\pm$ 1.9 & 297.6 $\pm$ 1.0 \\
\quad Cos: (5.0, 0.05) & 688.3 $\pm$ 0.5 & 345.5 $\pm$ 1.0 & 917.8 $\pm$ 1.3 & 676.0 $\pm$ 2.8 & 842.0 $\pm$ 3.1 & 783.0 $\pm$ 8.3 & 482.8 $\pm$ 1.9 & 208.7 $\pm$ 2.0 & 775.4 $\pm$ 0.7 & 364.4 $\pm$ 1.1 & 835.4 $\pm$ 3.4 & 423.0 $\pm$ 1.3 \\
\quad Exp: (4.0, 2.5) & 574.8 $\pm$ 0.5 & 211.8 $\pm$ 0.9 & 716.8 $\pm$0.8 & 442.7 $\pm$ 1.7 & 716.8 $\pm$ 0.7 & 637.9 $\pm$ 0.2 & 312.7 $\pm$0.2 & 153.9 $\pm$ 0.4 & 534.8 $\pm$ 1.9 & 290.1 $\pm$ 1.6 & 611.2 $\pm$2.4 & 361.9 $\pm$ 0.9 \\
\quad Exp: (5.0, 5.0) & 799.4 $\pm$ 0.9 & 582.5 $\pm$ 1.9 & 783.9 $\pm$1.5 & 707.8 $\pm$ 0.8 & 776.5 $\pm$ 0.5 & 678.2 $\pm$ 1.8 & 760.8 $\pm$1.2 & 312.5 $\pm$ 3.9 & 772.0 $\pm$ 3.5 & 539.2 $\pm$ 3.0 & 696.2 $\pm$0.9 & 452.3 $\pm$ 3.5 \\
\bottomrule
\end{tabular}}
\end{table}

\begin{table}[htbp]
\centering\tiny
\caption{\textbf{CelebV-HQ Non-Causal FVD ($\downarrow$):} On 2048 videos. Mean and std over three random seeds.}
\label{tab:celebvhq_NON_causal_fvd}
\resizebox{\textwidth}{!}{\begin{tabular}{lrrrrrrrrrrrr}
\toprule
 & \multicolumn{6}{c}{Num Frames = 32} & \multicolumn{6}{c}{Num Frames = 16} \\
\cmidrule(lr){2-7} \cmidrule(lr){8-13}
 & \multicolumn{2}{c}{pin\_every = 4} & \multicolumn{2}{c}{pin\_every = 8} & \multicolumn{2}{c}{pin\_every = 16} & \multicolumn{2}{c}{pin\_every = 4} & \multicolumn{2}{c}{pin\_every = 8} & \multicolumn{2}{c}{pin\_every = 16} \\
\cmidrule(lr){2-3} \cmidrule(lr){4-5} \cmidrule(lr){6-7} \cmidrule(lr){8-9} \cmidrule(lr){10-11} \cmidrule(lr){12-13}
\textbf{Method} \ \ \ \ \ \ Steps & \multicolumn{1}{c}{250} & \multicolumn{1}{c}{500} & \multicolumn{1}{c}{250} & \multicolumn{1}{c}{500} & \multicolumn{1}{c}{250} & \multicolumn{1}{c}{500} & \multicolumn{1}{c}{250} & \multicolumn{1}{c}{500} & \multicolumn{1}{c}{250} &\multicolumn{1}{c}{500} & \multicolumn{1}{c}{250} & \multicolumn{1}{c}{500} \\
\midrule
\rowcolor{gray!20} \multicolumn{13}{l}{\textit{ABC}} \\
\quad W/ BB: $\sigma$=0.5 & \textbf{\textit{43.7}} $\pm$ 0.2 & \textbf{\textit{35.3}} $\pm$ 0.1 & \textbf{107.5} $\pm$ 0.5 & \textbf{88.5} $\pm$ 0.2 & 285.8 $\pm$ 1.9 & 248.3 $\pm$ 0.4 & \textbf{\textit{106.6}} $\pm$ 0.5 & \textbf{\textit{96.4}} $\pm$ 0.3 & 201.6 $\pm$1.6 & \textbf{\textit{169.3}} $\pm$ 0.4 & 333.9 $\pm$ 0.3 & 278.4 $\pm$ 0.9 \\
\quad No BB: $\sigma$=0.5 & \textbf{40.4} $\pm$ 0.2 & \textbf{32.6} $\pm$ 0.1 & 112.2 $\pm$ 0.5 & \textbf{\textit{93.8}} $\pm$ 0.4 & 311.7 $\pm$ 1.2 & 264.4 $\pm$ 1.1 & \textbf{102.9} $\pm$ 0.7 & \textbf{92.1} $\pm$ 0.2 & \textbf{183.0} $\pm$ 0.6 & \textbf{158.5} $\pm$0.6 & \textbf{\textit{301.6}} $\pm$ 1.4 & \textbf{264.4} $\pm$ 1.0 \\
\rowcolor{gray!20} \multicolumn{13}{l}{\textit{Conditional Diffusion Bridge}} \\
\quad $\sigma=0.5$ & 262.3 $\pm$ 1.0 & 181.7 $\pm$ 0.2 & 389.7 $\pm$ 0.9 & 265.1 $\pm$ 0.8 & 840.8 $\pm$ 3.0 & 421.0 $\pm$ 1.7 & 230.2 $\pm$ 0.7 & 168.9 $\pm$ 0.6 & 311.7 $\pm$ 0.6 & 246.7 $\pm$ 1.1 & 433.8 $\pm$ 0.7 & 319.4 $\pm$ 1.0 \\
\quad $\sigma=0.125$ & 56.8 $\pm$ 0.2 & 41.6 $\pm$ 0.2 & 119.8 $\pm$ 0.3 & 97.3 $\pm$ 0.2& \textbf{233.8} $\pm$ 0.8 & \textbf{\textit{217.7}} $\pm$ 0.3 & 113.1 $\pm$ 0.3 & 102.1 $\pm$ 0.3 & \textbf{\textit{186.1}} $\pm$ 1.0 & 171.7 $\pm$ 1.0 & \textbf{289.6} $\pm$ 0.6& 299.4 $\pm$ 1.3 \\
\quad $\sigma=0.09$ & 49.0 $\pm$ 0.0 & 38.8 $\pm$ 0.2 & \textbf{\textit{111.5}} $\pm$ 0.5& 93.9 $\pm$ 0.3 & \textbf{\textit{237.1}} $\pm$ 1.0 & \textbf{203.5} $\pm$ 1.3 & 120.4 $\pm$ 0.1 & 109.0 $\pm$ 0.9 & 205.1 $\pm$ 0.4 & 192.3 $\pm$ 0.8 & 339.6 $\pm$ 0.3 & 368.6 $\pm$ 0.7 \\
\rowcolor{gray!20} \multicolumn{13}{l}{\textit{Noise-to-Data Diffusion}} \\
\quad Cos: (3.0, 0.04) & 706.0 $\pm$ 0.4 & 434.4 $\pm$ 0.3 & 908.6 $\pm$ 0.0 & 590.6 $\pm$ 2.2 & 970.5 $\pm$ 0.5 & 760.5 $\pm$ 3.9 & 565.2 $\pm$ 2.1 & 213.7 $\pm$ 1.6 & 672.8 $\pm$ 1.2 & 346.3 $\pm$ 1.7 & 825.8 $\pm$ 1.0 & 492.1 $\pm$ 3.3 \\
\quad Cos: (5.0, 0.05) & 906.0 $\pm$ 0.8 & 572.5 $\pm$ 1.5 & 996.1 $\pm$ 0.8 & 756.0 $\pm$ 2.9 & 992.0 $\pm$ 2.3 & 907.3 $\pm$ 0.6 & 741.9 $\pm$ 1.7 & 295.8 $\pm$ 0.5 & 857.7 $\pm$ 1.5 & 484.0 $\pm$ 1.7 & 1023.2 $\pm$ 1.0 & 692.2 $\pm$ 2.8 \\
\quad Exp: (4.0, 2.5) & 810.8 $\pm$ 1.4 & 337.5 $\pm$ 0.7 & 1019.7 $\pm$ 0.9 & 497.2 $\pm$ 0.8 & 1021.9 $\pm$ 0.6 & 670.3 $\pm$ 0.8 & 424.8 $\pm$ 0.7 & 138.9 $\pm$ 0.8 & 588.6 $\pm$ 0.9 & 224.0 $\pm$ 1.3 & 649.8 $\pm$ 1.7 & \textbf{\textit{277.4}} $\pm$ 0.6 \\
\quad Exp: (5.0, 5.0) & 1123.8 $\pm$ 0.6 & 748.6 $\pm$ 0.4 & 1206.2 $\pm$ 0.9 & 956.4 $\pm$ 1.5 & 1071.6 $\pm$ 1.3 & 956.7 $\pm$ 1.8 & 804.2 $\pm$ 2.4 & 420.8 $\pm$ 1.5 & 878.9 $\pm$ 0.6 & 591.1 $\pm$ 1.8 & 967.4 $\pm$ 0.7 & 763.5 $\pm$ 1.3 \\
\bottomrule
\end{tabular}
}
\end{table}

\paragraph{Weather Modeling}
Table \ref{tab:sevir_main_metrics} 
shows that ABC also performs strongly on SEVIR. In the non-causal setting, ABC Non-Causal achieves the best overall performance, outperforming both conditional diffusion bridges and noise-to-data diffusion; this again supports \rqref{time_scaling} and \rqref{data_to_data}. ABC also performs well in the strictly causal setting, showing that the same framework can handle both sparse non-causal conditioning and standard forecasting. See Sections \ref{sec:sevir_additional_details} and \ref{sec:sevir_additional_results} for more details.

\begin{table}[htbp]
\tiny
\centering
\caption{\textbf{SEVIR VIL}: Full Test Set (847 clips) Native Metrics. Three random inference seeds.}
\label{tab:sevir_main_metrics}
\resizebox{\linewidth}{!}{%
\begin{tabular}{lrrrrrrr}
\toprule
Method & MAE ($\downarrow$) & RMSE ($\downarrow$) & CSI$_{16}$ ($\uparrow$) & CSI$_{74}$ ($\uparrow$) & CSI$_{133}$ ($\uparrow$) & CSI$_{160}$ ($\uparrow$) & \shortstack{Avg. Rank\\($\downarrow$)} \\
\midrule
\textit{ABC Non-causal}
& $\mathbf{18.228 \pm 0.003}$
& $\mathbf{34.849 \pm 0.004}$
& $\mathbf{0.562 \pm 0.00007}$
& $\mathbf{0.699 \pm 0.00003}$
& $\mathbf{0.644 \pm 0.00007}$
& $\mathbf{0.664 \pm 0.00004}$
& \textbf{1.0} \\
\textit{ABC Causal}
& $\mathit{24.382 \pm 0.010}$
& $\mathit{54.793 \pm 0.022}$
& $\mathit{0.554 \pm 0.00040}$
& $\mathit{0.570 \pm 0.00023}$
& $\mathit{0.476 \pm 0.00027}$
& $\mathit{0.494 \pm 0.00016}$
& \textit{2.0} \\
\textit{Conditional Diffusion Bridge}
& $55.196 \pm 0.098$
& $88.038 \pm 0.148$
& $0.280 \pm 0.00067$
& $0.206 \pm 0.00092$
& $0.144 \pm 0.00069$
& $0.129 \pm 0.00061$
& 4.0 \\
\textit{Noise-to-Data Diffusion}
& $35.698 \pm 0.121$
& $62.457 \pm 0.240$
& $0.373 \pm 0.00010$
& $0.437 \pm 0.00074$
& $0.338 \pm 0.00067$
& $0.328 \pm 0.00059$
& 3.0 \\
\bottomrule
\end{tabular}%
}
\end{table}

\begin{table}[h]
\centering
\newcommand{\sevirmeanstd}[2]{\shortstack{$#1$\\[-1pt]{\scriptsize$\pm\,#2$}}}
\caption{\textbf{SEVIR VIL}: Full Test Set (847 clips) Detection/Calibration Metrics.}
\label{tab:sevir_detection_metrics}
\resizebox{\linewidth}{!}{%
\begin{tabular}{lcccccccccccc}
\toprule
Method & \shortstack{POD$_{16}$\\($\uparrow$)} & \shortstack{POD$_{74}$\\($\uparrow$)} & \shortstack{POD$_{133}$\\($\uparrow$)} & \shortstack{POD$_{160}$\\($\uparrow$)} & \shortstack{FAR$_{16}$\\($\downarrow$)} & \shortstack{FAR$_{74}$\\($\downarrow$)} & \shortstack{FAR$_{133}$\\($\downarrow$)} & \shortstack{FAR$_{160}$\\($\downarrow$)} & Bias$_{16}$ & Bias$_{74}$ & Bias$_{133}$ & Bias$_{160}$ \\
\midrule
\rowcolor{gray!20} \multicolumn{13}{l}{\textit{ABC}} \\
\quad Non-causal
& \sevirmeanstd{\mathbf{0.926}}{\mathbf{0.00007}}
& \sevirmeanstd{\mathbf{0.843}}{\mathbf{0.00002}}
& \sevirmeanstd{\mathbf{0.774}}{\mathbf{0.00001}}
& \sevirmeanstd{\mathbf{0.762}}{\mathbf{0.00001}}
& \sevirmeanstd{\mathbf{0.412}}{\mathbf{0.00005}}
& \sevirmeanstd{\mathbf{0.197}}{\mathbf{0.00003}}
& \sevirmeanstd{\mathbf{0.206}}{\mathbf{0.00011}}
& \sevirmeanstd{\mathbf{0.162}}{\mathbf{0.00008}}
& \sevirmeanstd{1.576}{0.00005}
& \sevirmeanstd{1.049}{0.00004}
& \sevirmeanstd{0.975}{0.00014}
& \sevirmeanstd{0.909}{0.00011} \\
\quad Causal
& \sevirmeanstd{\mathit{0.814}}{\mathit{0.00039}}
& \sevirmeanstd{\mathit{0.712}}{\mathit{0.00020}}
& \sevirmeanstd{\mathit{0.621}}{\mathit{0.00030}}
& \sevirmeanstd{\mathit{0.638}}{\mathit{0.00016}}
& \sevirmeanstd{\mathit{0.366}}{\mathit{0.00033}}
& \sevirmeanstd{\mathit{0.258}}{\mathit{0.00022}}
& \sevirmeanstd{\mathit{0.329}}{\mathit{0.00019}}
& \sevirmeanstd{\mathit{0.314}}{\mathit{0.00018}}
& \sevirmeanstd{\mathit{1.285}}{\mathit{0.00053}}
& \sevirmeanstd{\mathit{0.959}}{\mathit{0.00027}}
& \sevirmeanstd{\mathit{0.924}}{\mathit{0.00019}}
& \sevirmeanstd{\mathit{0.930}}{\mathit{0.00027}} \\
\rowcolor{gray!20} \multicolumn{13}{l}{\textit{Conditional Bridge}} \\
\quad Diffusion Bridge
& \sevirmeanstd{0.630}{0.00122}
& \sevirmeanstd{0.374}{0.00145}
& \sevirmeanstd{0.272}{0.00089}
& \sevirmeanstd{0.232}{0.00067}
& \sevirmeanstd{0.665}{0.00062}
& \sevirmeanstd{0.686}{0.00118}
& \sevirmeanstd{0.765}{0.00127}
& \sevirmeanstd{0.775}{0.00137}
& \sevirmeanstd{1.880}{0.00070}
& \sevirmeanstd{1.190}{0.00206}
& \sevirmeanstd{1.158}{0.00404}
& \sevirmeanstd{1.034}{0.00487} \\
\rowcolor{gray!20} \multicolumn{13}{l}{\textit{Noise-to-Data Diffusion}} \\
\quad Noise-to-Data
& \sevirmeanstd{0.821}{0.00016}
& \sevirmeanstd{0.596}{0.00092}
& \sevirmeanstd{0.451}{0.00205}
& \sevirmeanstd{0.421}{0.00264}
& \sevirmeanstd{0.594}{0.00015}
& \sevirmeanstd{0.378}{0.00173}
& \sevirmeanstd{0.426}{0.00413}
& \sevirmeanstd{0.402}{0.00637}
& \sevirmeanstd{2.023}{0.00108}
& \sevirmeanstd{0.959}{0.00364}
& \sevirmeanstd{0.785}{0.00898}
& \sevirmeanstd{0.704}{0.01183} \\
\bottomrule
\end{tabular}%
}
\end{table}

\section{Discussion, Limitations, Conclusion}
Future work should investigate sampling schemes, \textit{e.g.}, hybrid ODE-SDE; coarse-to-fine by causally generating every $K$ frames at low sampling rate before infilling the $K-1$ frames in between. 
For a mathematically faithful implementation, we made the transformer network attend to the entire history, but efficient architectures like SSMs \cite{gu2023mamba} selectively compress history.
Future work could also consider text-to-image generation \cite{rombach2022high}.
Finally, future work should investigate generating chunks of states at a time, \textit{e.g.}, augment SDE state to four frames; this fits into our framework.

In summary, we introduce \textbf{ABC}: an SDE-based generalization of any-subset autoregressive generative models to continuous time and space. We rigorously derived a training objective generalizing denoising score matching to path- and time-dependent cases. Finally, we theoretically and empirically showed that its inductive biases of data-to-data generative trajectories and time-adaptive volatility lead to faithful modeling of stochastic processes, such as videos and weather.

\clearpage

\section*{Acknowledgements}

\textcolor{blue}{
This material is based upon work supported by the U.S. Department of Energy, Office of Science, Office of Advanced Scientific Computing Research, Department of Energy Computational Science Graduate Fellowship under Award Number(s) DE-SC0025528 to G. Guo.
This research used resources of the National Energy Research Scientific Computing Center, a DOE Office of Science User Facility supported by the Office of Science of the U.S. Department of Energy under Contract No. DE-AC02-05CH11231 using NERSC award ASCR-ERCAP0038735 to G. Guo.
Research supported by the NVIDIA Academic Grant Program using A100 GPU-Hours on Brev, awarded to S. Ermon.
Research supported by the Stanford Institute for Human-Centered Artificial Intelligence (HAI)'s Google Cloud Credits Program, awarded to S. Ermon.
J. Blanchet gratefully acknowledges support from DoD through the grants Air Force Office of Scientific Research under award number FA9550-20-1-0397 and ONR N00014-24-1-2655, also support from NSF via grants 2229012, 2312204, 2403007 is gratefully acknowledged.
}

\textcolor{blue}{We thank Jerry W Liu, Peter Pao-Huang, John Yao, Tristan Saidi for helpful feedback.}

\printbibliography

\newpage

\appendix



\section{Data-Generating SDE via Doob
h-Transform}
\label{appendix:B}
In this section we prove Theorem \ref{thm:new_sde} and establish the regularity conditions. 


Let $0 = t_0 < t_1 < \cdots < t_L$ be fixed physical times and $x_0 \in \RR^d$ an
initial condition.  The \emph{base (prior) process} $(X_t)_{t\in[0,t_L]}$ is the
Ornstein--Uhlenbeck (OU) SDE
\begin{equation}\label{eqn:ou}
  dX_t = -a(t)X_t\,dt + \sigma(t)\,dW_t,\qquad X_0=x_0,
\end{equation}
where $W$ is a $d$-dimensional $\PP$-Brownian motion (the base-process driver),
$a:[0,t_L]\to\RR$, and $\sigma:[0,t_L]\to(0,\infty)$.  We reserve the symbol $B$
for the $\QQ$-Brownian motion constructed in Step~5.  Denote by $\PP$ the induced
law on path space $C([0,t_L];\RR^d)$, and by $\pbase(x_{t_1},\ldots,x_{t_L}\mid x_0)$
the finite dimensional density under $\PP$.

Recall that at time $t \in (t_{i^*}, t_{i^*+1})$, we write the \emph{observed past} as
\begin{equation}\label{eqn:xhist}
  \mathbf{x}_{0:i^*} := [(0,x_0)]\oplus[(t_j,x_{t_j})\mid j\leq i^*],
\end{equation}
where $i^*$ is the index of the most recent observation,
$i^* := \max\{i: t_i < t\}$.
The current state $x_t$ is \emph{not} included in $\mathbf{x}_{0:i^*}$; it appears
as a separate explicit argument in all functions below.

\subsection{Regularity Conditions}\label{sec:regularity}

The following conditions are assumed throughout the proof.  Each plays a distinct role.

\begin{enumerate}[label=\textbf{(R\arabic*)}]

\item \label{reg:sde}
  \textbf{Well-posedness of the base SDE.}
  $a$ and $\sigma$ are bounded and measurable on $[0,t_L]$, and $\sigma(t)\geq\sigma_{\min}>0$
  for all $t$.
  Under these conditions, the linear SDE \eqref{eqn:ou} has a unique strong solution, and
  the finite dimensional distribution $\pbase(x_{t_1},\ldots,x_{t_L}\mid x_0)$ is a
  non-degenerate Gaussian for every $x_0$.
  The lower bound on $\sigma$ is essential: it prevents the base process from
  degenerating into a deterministic ODE, which would violate absolute continuity.

\item \label{reg:abscon}
  \textbf{Absolute continuity.}
  For $\PP$-a.e. $x_0$,
  $\pdata(x_{t_1},\ldots,x_{t_L}\mid x_0)\ll\pbase(x_{t_1},\ldots,x_{t_L}\mid x_0)$,
  i.e., the data distribution assigns no mass to sets that the base process
  cannot reach.  This guarantees the Radon--Nikodym derivative
  $\frac{d\QQ}{d\PP}=\frac{\pdata}{\pbase}$ exists as an $L^1(\PP)$ density.
  It is implied by \ref{reg:sde} whenever $\pdata$ is a continuous distribution
  without atoms, since the OU process has full Gaussian support.

\item \label{reg:fullsupport}
  \textbf{Full support / strict positivity.}
  $\pdata(x_{t_1},\ldots,x_{t_L}\mid x_0)>0$ $\pbase$-almost everywhere.
  This ensures the change-of-measure martingale $Z_t:=\EE_\PP[d\QQ/d\PP\mid\FF_t]$
  satisfies $Z_t>0$ $\PP$-almost surely, so that $\log Z_t$ and the score
  $\nabla_{x_t}\log u$ are well-defined.

\item \label{reg:squareint}
  \textbf{Square-integrability of the Radon--Nikodym derivative.}
  $\EE_\PP[(d\QQ/d\PP)^2]<\infty$, equivalently $Z_T\in L^2(\PP)$.
  This implies $Z_t\in L^2(\PP)$ for all $t$ (by the conditional Jensen inequality),
  which is the integrability condition required for $Z_t$ to be a bona fide
  (not merely local) $\PP$-martingale.  It holds whenever $\pdata/\pbase$ is
  square-integrable under $\pbase$, a mild condition on the tails of $\pdata$.

\item \label{reg:smooth}
  \textbf{Smoothness of $u$.}
  The function $u(t,x_t;\mathbf{x}_{0:i^*}):=Z_t$ satisfies $u>0$,
  $u\in C^{1,2}$ jointly in $(t,x_t)$ on each open interval $(t_{i^*},t_{i^*+1})$,
  and the gradient $\nabla_{x_t}\log u$ is locally square-integrable in $t$.
  The $C^{1,2}$ condition is needed to apply It\^o's formula to $u$ and to $\log u$.
  It follows from standard parabolic PDE regularity when $\pdata$ has sufficiently
  smooth conditional densities; in our setting it holds on the open intervals because
  the OU Green's function (the base transition kernel) is smooth.

\item \label{reg:novikov}
  \textbf{Novikov's condition.}
  \[
    \EE_\PP\left[\exp\left(
      \tfrac{1}{2}\int_0^{t_L}
      \|\sigma(s)\,\nabla_{x_s}\log u(s,x_s;\mathbf{x}_{0:i^*_s})\|^2\,ds
    \right)\right]<\infty.
  \]
  This is the classical sufficient condition \cite{novikov1972identity}
  that upgrades a local $\PP$-martingale to a true $\PP$-martingale.  It is invoked
  in Step~4 of the proof below to guarantee that the Dol\'eans--Dade exponential
  $Z_t$ is a true martingale and that Girsanov's theorem \cite{girsanov1960transforming} applies without qualification.
  In practice it is satisfied because the score $\nabla_{x_t}\log u$ coincides with
  the neural-network output $\mathbf{f}_\theta$, which is bounded on compact time
  intervals under any finite-parameter model.

\item \label{reg:markov}
  \textbf{Markov property of the base process.}
  For $s<t<\tau$ all in $[0,t_L]$,
  $\pbase(x_\tau\mid x_{[0,t]})=\pbase(x_\tau\mid x_t)$ under $\PP$.
  The OU process is a Markov process driven by Brownian motion, so this holds
  automatically.  It is used in Step~6 to factor the conditional base density
  and reduce the integral over future states to a one-step expectation over
  $x_{t_{i^*+1}}$ alone.

\end{enumerate}

\subsection{Gaussian Base Process: Score Target and Noising Kernel}

The following lemma establishes closed-form expressions for the two quantities
that enter both the SDE drift and the training loss: the \emph{score target}
$\nabla_{x_t}\log\pbase(x_{t_{i^*+1}}\mid x_t)$ and the \emph{noising kernel}
$\pbase(x_t\mid x_{t_{i^*}},x_{t_{i^*+1}})$.  Both follow from the Gaussian
process structure of the OU base process.

Define the resolvent and covariance kernel of the OU base process \eqref{eqn:ou}
conditioned on $x_{t_{i^*}}$:
\begin{align}
  \Phi(s,t) &:= \exp\left(\int_s^t -a(u)\,du\right), \label{eqn:phi}\\
  C_{t_{i^*}}(\tau_a,\tau_b)
  &:= \int_{t_{i^*}}^{\min(\tau_a,\tau_b)}
      \Phi(s,\tau_a)\,\Phi(s,\tau_b)\,\sigma(s)^2\,ds. \label{eqn:cov_kernel}
\end{align}

\begin{lemma}[Closed-form score target and noising kernel]\label{lem:gaussian_kernels}
Let $t_{i^*}\leq t < t_{i^*+1}$.  Under the OU base process \eqref{eqn:ou} with
regularity condition \ref{reg:sde}:
\begin{enumerate}[label=(\roman*)]

\item \textbf{Score target.}
  The conditional distribution $x_{t_{i^*+1}}\mid x_t$ under $\pbase$ is
  Gaussian with mean $\Phi(t,t_{i^*+1})x_t$ and variance $C_t(t_{i^*+1},t_{i^*+1})I$.
  Consequently,
  \begin{equation}\label{eqn:score_target}
    \nabla_{x_t}\log\pbase(x_{t_{i^*+1}}\mid x_t)
    = \frac{\Phi(t,t_{i^*+1})}{C_t(t_{i^*+1},t_{i^*+1})}
      \bigl(x_{t_{i^*+1}}-\Phi(t,t_{i^*+1})x_t\bigr).
  \end{equation}
  This quantity is non-singular and linear in $(x_t, x_{t_{i^*+1}})$ for all
  $t\in[t_{i^*},t_{i^*+1})$.

\item \textbf{Noising kernel.}
  The conditional distribution $x_t\mid (x_{t_{i^*}},x_{t_{i^*+1}})$ under
  $\pbase$ is Gaussian with
  \begin{align}
    \mu_{t\mid t_{i^*}, t_{i^*+1}}(x_{t_{i^*}}, x_{t_{i^*+1}})
    &= \Phi(t_{i^*},t)x_{t_{i^*}}
       + \frac{C_{t_{i^*}}(t,t_{i^*+1})}{C_{t_{i^*}}(t_{i^*+1},t_{i^*+1})}
         \bigl(x_{t_{i^*+1}}-\Phi(t_{i^*},t_{i^*+1})x_{t_{i^*}}\bigr),
       \label{eqn:noising_mean}\\
    \Sigma_{t\mid t_{i^*}, t_{i^*+1}}
    &= \left(C_{t_{i^*}}(t,t)
       - \frac{C_{t_{i^*}}(t,t_{i^*+1})^2}{C_{t_{i^*}}(t_{i^*+1},t_{i^*+1})}\right)I.
       \label{eqn:noising_var}
  \end{align}
  The variance $\Sigma_{t\mid t_{i^*+1}}$ is strictly positive for
  $t\in(t_{i^*},t_{i^*+1})$ and vanishes at both endpoints.
\end{enumerate}
\end{lemma}
\begin{proof}
Part~(i) follows from the explicit variation-of-constants solution of the linear SDE
\eqref{eqn:ou} and Itô's isometry; see \cite{oksendal2003stochastic}, Chapter~5.
The score formula is then the log-gradient of the resulting Gaussian density.
Part~(ii) is Gaussian process conditioning applied to the jointly Gaussian pair
$(x_t, x_{t_{i^*+1}})$ given $x_{t_{i^*}}$; the mean and variance
\eqref{eqn:noising_mean}--\eqref{eqn:noising_var} follow from the standard
formulas of \cite{bishop2006pattern} (equations 2.81--2.82), with covariances
computed via Itô's isometry and the independence of non-overlapping Brownian
increments \cite{oksendal2003stochastic}.

The OU process \eqref{eqn:ou} has the explicit variation-of-constants solution,
verified by differentiating with the Leibniz integral rule:
\begin{equation}\label{eqn:ou_solution}
  x_t = \Phi(t_{i^*},t)\,x_{t_{i^*}}
        + \int_{t_{i^*}}^{t}\Phi(s,t)\,\sigma(s)\,dW_s.
\end{equation}

\medskip\noindent\underline{Part~(i): Score target.}
Setting $t_{i^*}\leftarrow t$ in \eqref{eqn:ou_solution} and letting the process
run to $t_{i^*+1}$:
\begin{equation}\label{eqn:x_next}
  x_{t_{i^*+1}} = \Phi(t,t_{i^*+1})\,x_t
                  + \int_t^{t_{i^*+1}}\Phi(s,t_{i^*+1})\,\sigma(s)\,dW_s.
\end{equation}
Since $W$ has independent increments, the stochastic integral has zero mean.
By Itô's isometry:
\[
  \mathrm{Var}(x_{t_{i^*+1}}\mid x_t)
  = \int_t^{t_{i^*+1}}\Phi(s,t_{i^*+1})^2\,\sigma(s)^2\,ds\cdot I
  = C_t(t_{i^*+1},t_{i^*+1})\,I.
\]
Hence $x_{t_{i^*+1}}\mid x_t\sim\mathcal{N}\left(\Phi(t,t_{i^*+1})x_t,\;
C_t(t_{i^*+1},t_{i^*+1})I\right)$.
Taking the log-gradient of this Gaussian density with respect to $x_t$:
\begin{align*}
  \nabla_{x_t}\log\pbase(x_{t_{i^*+1}}\mid x_t)
  &= \nabla_{x_t}\left[
       -\frac{\|x_{t_{i^*+1}}-\Phi(t,t_{i^*+1})x_t\|^2}
             {2\,C_t(t_{i^*+1},t_{i^*+1})}
     \right] \\
  &= \frac{\Phi(t,t_{i^*+1})}{C_t(t_{i^*+1},t_{i^*+1})}
     \bigl(x_{t_{i^*+1}}-\Phi(t,t_{i^*+1})x_t\bigr),
\end{align*}
which is \eqref{eqn:score_target}.  Non-singularity holds because
$C_t(t_{i^*+1},t_{i^*+1})=\int_t^{t_{i^*+1}}\Phi(s,t_{i^*+1})^2\sigma(s)^2\,ds>0$
for all $t<t_{i^*+1}$, by \ref{reg:sde} ($\sigma\geq\sigma_{\min}>0$).

\medskip\noindent\underline{Part~(ii): Noising kernel.}
Conditioned on $x_{t_{i^*}}$, both $x_t$ and $x_{t_{i^*+1}}$ are affine functions
of $W$ (by \eqref{eqn:ou_solution}), so the pair is jointly Gaussian with means
\[
  \mu_t (x_{t_{i^*}})= \Phi(t_{i^*},t)\,x_{t_{i^*}},
  \qquad
  \mu_{t_{i^*+1}}(x_{t_{i^*}}) = \Phi(t_{i^*},t_{i^*+1})\,x_{t_{i^*}}.
\]
We compute the covariances using Itô's isometry.
Off-diagonal entries of the covariance matrix are zero because $W$ has
independent components.  For the diagonal entries, the cross-covariance
involves the inner product of the two stochastic integrals:
\begin{align*}
  \Sigma_{t,t_{i^*+1}}[j,j]
  &= \EE_\PP\left[
       \int_{t_{i^*}}^t \Phi(s,t)\sigma(s)\,dW_s[j]
       \cdot
       \int_{t_{i^*}}^{t_{i^*+1}}\Phi(s,t_{i^*+1})\sigma(s)\,dW_s[j]
     \right].
\end{align*}
Split the second integral at $t$; the piece over $(t,t_{i^*+1})$ is independent of
the first integral (non-overlapping increments), contributing zero.  Itô's isometry
on the overlapping part $[t_{i^*},t]$ gives
\[
  \Sigma_{t,t_{i^*+1}}[j,j]
  = \int_{t_{i^*}}^{t}\Phi(s,t)\,\Phi(s,t_{i^*+1})\,\sigma(s)^2\,ds
  = C_{t_{i^*}}(t,t_{i^*+1}).
\]
By the same argument,
\[
  \Sigma_{t,t} = C_{t_{i^*}}(t,t)\,I,
  \qquad
  \Sigma_{t_{i^*+1},t_{i^*+1}} = C_{t_{i^*}}(t_{i^*+1},t_{i^*+1})\,I,
  \qquad
  \Sigma_{t,t_{i^*+1}} = C_{t_{i^*}}(t,t_{i^*+1})\,I.
\]
Applying the standard Gaussian conditional formulas
(\cite{bishop2006pattern}, equations 2.81--2.82):
\begin{align*}
  \mu_{t\mid t_{i^*},t_{i^*+1}}(x_{t_{i^*}}, x_{t_{i^*+1}})
  &= \mu_t (x_{t_{i^*}}) + \Sigma_{t,t_{i^*+1}}\Sigma_{t_{i^*+1},t_{i^*+1}}^{-1}
     (x_{t_{i^*+1}}-\mu_{t_{i^*+1}}(x_{t_{i^*}})), \\
  \Sigma_{t\mid t_{i^*+1}}
  &= \Sigma_{t,t} - \Sigma_{t,t_{i^*+1}}\Sigma_{t_{i^*+1},t_{i^*+1}}^{-1}
     \Sigma_{t_{i^*+1},t},
\end{align*}
and substituting the scalar covariances yields
\eqref{eqn:noising_mean}--\eqref{eqn:noising_var}.

Strict positivity of $\Sigma_{t\mid t_{i^*+1}}$ on $(t_{i^*},t_{i^*+1})$ follows
from the Cauchy--Schwarz inequality applied to the covariance kernel: by applying the
 Cauchy-Schwarz inequality and the non-negative integration from time $t=t$ to $t_{i^*+1}$, we have that
\[
  C_{t_{i^*}}(t,t_{i^*+1})^2
  \;\leq\;
  C_{t_{i^*}}(t,t)\cdot C_{t_{i^*}}(t_{i^*+1},t_{i^*+1}),
\]
with strict inequality when $t\in(t_{i^*},t_{i^*+1})$, since the integrand
$\Phi(s,t)\Phi(s,t_{i^*+1})\sigma(s)^2$ is not a scalar multiple of
$\Phi(s,t)^2\sigma(s)^2$ (the resolvent $\Phi(s,\cdot)$ is strictly monotone in
its second argument for $a(t)\neq 0$, and bounded away from a constant ratio by
\ref{reg:sde}).
\end{proof}

\subsection{Proof of Theorem \ref{thm:new_sde}} \label{sec:proof_thm_1}
\begin{remark}
    $s(t,x_t;\mathbf{x}_{0:i^*})$ in the theorem statement is simply another notation for the oracle score
$\nabla_{x_t}\log u(t,x_t;\mathbf{x}_{0:i^*})$, where $u$ is the Doob
$h$-function (Step~2 in proof).  They are the same analytical object. The proof uses $W_t$ for the
$\PP$-Brownian motion and $B_t$ for the $\QQ$-Brownian motion (Step~5 in proof).

Learning $s$ from data is addressed separately in
Appendix~\ref{appendix:C}. The neural network $\mathbf{f}_\theta$ is a parametric approximation introduced that learns $s$; the two
coincide, $\mathbf{f}_{\theta^*}=s$, only at the minimizer of the score matching
loss.  
\end{remark}

\begin{proof}[Proof of Theorem \ref{thm:new_sde}] The proof has seven steps. Steps 1--2 establish the target measure $\QQ$; Steps 3--6 extract the modified SDE dynamics via
Girsanov's theorem; Step 7 derives the score.

\noindent\textbf{Step 1: Base process and target change-of-measure.}
We take the OU process \eqref{eqn:ou} as the base process under $\PP$.  Its
transition densities are Gaussian by linearity, and by \ref{reg:sde} the finite dimensional
distribution $\pbase(x_{t_1},\ldots,x_{t_L}\mid x_0)$ is a non-degenerate Gaussian
with full support on $(\RR^d)^L$.

We wish to define a new measure $\QQ$ on path space such that $(X_0,X_{t_1},\ldots,X_{t_L})$
under $\QQ$ has joint density $\pdata(x_0,x_{t_1},\ldots,x_{t_L})$.
Equivalently, conditioned on $X_0=x_0$, the grid $(X_{t_1},\ldots,X_{t_L})$
should have density $\pdata(\cdot\mid x_0)$.  By
\ref{reg:abscon}, the Radon--Nikodym derivative is well-defined on the
$\sigma$-algebra generated by the grid:
\begin{equation}\label{eqn:rnd}
  \frac{d\QQ}{d\PP}\bigg|_{\sigma(X_{t_1},\ldots,X_{t_L})}
  = \frac{\pdata(x_{t_1},\ldots,x_{t_L}\mid x_0)}
         {\pbase(x_{t_1},\ldots,x_{t_L}\mid x_0)}.
\end{equation}
For any event $A$ in the path $\sigma$-algebra, $\QQ(A)=\EE_\PP[\mathbbm{1}_A\,d\QQ/d\PP]$.
By construction, for every Borel set $S\subseteq(\RR^d)^L$,
\[
  \QQ\left((X_{t_1},\ldots,X_{t_L})\in S\right)
  = \int_S \pdata(x_{t_1},\ldots,x_{t_L}\mid x_0)\,dx_{t_1}\cdots dx_{t_L},
\]
so the grid under $\QQ$ has exactly the desired density.

\noindent\textbf{Step 2: The Doob $h$-transform is a $\PP$-martingale.}
In order to extract the SDE dynamics under $\QQ$, we need the
Radon--Nikodym derivative \cite{radon1913theorie, nikodym1930generalisation} as a process indexed by $t$.  Define
\begin{equation}\label{eqn:Zt}
  Z_t
  := h(t,x_t;\mathbf{x}_{0:i^*})
  := \EE_\PP\left[
       \frac{d\QQ}{d\PP}\,\bigg|\,\FF_t
     \right]
   = \EE_\PP\left[
       \frac{\pdata(X_{t_1},\ldots,X_{t_L}\mid X_0)}
            {\pbase(X_{t_1},\ldots,X_{t_L}\mid X_0)}
       \,\bigg|\,\FF_t
     \right].
\end{equation}
This is the \emph{Doob $h$-function} \cite{doob1984classical}: a positive
$(\FF_t,\PP)$-adapted process obtained by conditioning the terminal
Radon--Nikodym derivative on the current information.

\medskip\noindent\underline{$Z_t$ is a $\PP$-martingale.}
For any $\tau>t$, the tower property of conditional expectation gives
\[
  \EE_\PP[Z_\tau\mid\FF_t]
  = \EE_\PP\left[\EE_\PP\left[\frac{d\QQ}{d\PP}\,\bigg|\,\FF_\tau\right]\,\bigg|\,\FF_t\right]
  = \EE_\PP\left[\frac{d\QQ}{d\PP}\,\bigg|\,\FF_t\right]
  = Z_t.
\]
By \ref{reg:squareint}, $Z_t\in L^2(\PP)$ for all $t$, so this is a true (square-integrable)
martingale, not merely a local one.

\medskip\noindent\underline{Normalisation $Z_0=1$.}
At $t=0$, $\FF_0=\sigma(X_0)$ and conditioning on $X_0=x_0$ gives
\[
  Z_0
  = \EE_\PP\left[\frac{d\QQ}{d\PP}\right]
  = \int \frac{\pdata(x_{t_1},\ldots,x_{t_L}\mid x_0)}
              {\pbase(x_{t_1},\ldots,x_{t_L}\mid x_0)}
    \pbase(x_{t_1},\ldots,x_{t_L}\mid x_0)\,dx_{t_1}\cdots dx_{t_L}
  = 1.
\]

\noindent\textbf{Step 3: It\^o's lemma on $Z_t$.}
To obtain the dynamics of $Z_t$ along the OU paths, write $Z_t=u(t,x_t;\mathbf{x}_{0:i^*})$,
making explicit the dependence on $(t,x_t)$ and on the already-realized path
history $(x_0,x_{t_1},\ldots,x_{t_{i^*}})$ from the filtration.  The realized history
is $\FF_t$-measurable and plays the role of a parameter;
only the current state $x_t$ is ``live'' in the It\^o sense.

By \ref{reg:smooth}, $u\in C^{1,2}$ in $(t,x_t)$.
Applying It\^o's formula \cite{ito1944109, ito1951stochastic} to $u(t,x_t;\mathbf{x}_{0:i^*})$ along \eqref{eqn:ou}:
\begin{equation}\label{eqn:ito_u}
  dZ_t
  = \left(\frac{\partial u}{\partial t}
    + \tfrac{1}{2}\mathrm{Tr}\bigl(\sigma(t)^2\nabla_{x_t}^2 u\bigr)
    - a(t)x_t\cdot\nabla_{x_t}u\right)dt
  + \nabla_{x_t}u\cdot\sigma(t)\,dW_t.
\end{equation}
Since $Z_t$ is a $\PP$-martingale (Step~2), its $dt$-drift must vanish
$\PP$-almost surely.  The $dW_t$ term already has zero mean, so the bracketed
expression in \eqref{eqn:ito_u} equals zero, leaving
\begin{equation}\label{eqn:dZt}
  dZ_t = \nabla_{x_t}u(t,x_t;\mathbf{x}_{0:i^*})\cdot\sigma(t)\,dW_t.
\end{equation}
This identifies $Z_t$ as a \emph{stochastic integral}: it has no drift and
its quadratic variation is $d\langle Z\rangle_t = \|\nabla_{x_t}u\|^2\sigma(t)^2\,dt$.

\noindent\textbf{Step 4: Dol\'eans--Dade exponential form.}
To apply Girsanov's theorem \cite{girsanov1960transforming} we must write $Z_t$ in the specific exponential
martingale form of Dol\'eans--Dade \cite{doleans1970quelques}.  Apply It\^o's formula to $\log Z_t$
(valid since $Z_t>0$ by \ref{reg:fullsupport}):
\begin{equation}\label{eqn:ito_logZ}
  d\log Z_t
  = \frac{1}{Z_t}\,dZ_t - \frac{1}{2Z_t^2}(dZ_t)^2
  = \frac{\nabla_{x_t}u}{u}\cdot\sigma(t)\,dW_t
    - \frac{1}{2}\left\|\frac{\nabla_{x_t}u}{u}\,\sigma(t)\right\|^2 dt.
\end{equation}
Using the chain rule $\nabla_{x_t}u/u = \nabla_{x_t}\log u$ and integrating
\eqref{eqn:ito_logZ} from $0$ to $t$:
\begin{equation}\label{eqn:logZ}
  \log Z_t
  = \int_0^t \sigma(s)\,\nabla_{x_s}\log u(s,x_s;\mathbf{x}_{0:i^*_s})\cdot dW_s
  - \frac{1}{2}\int_0^t \|\sigma(s)\,\nabla_{x_s}\log u(s,x_s;\mathbf{x}_{0:i^*_s})\|^2\,ds.
\end{equation}
Exponentiating:
\begin{equation}\label{eqn:doleans}
  Z_t = \exp\left(
    \int_0^t Y_s\cdot dW_s
    - \frac{1}{2}\int_0^t\|Y_s\|^2\,ds
  \right),
  \qquad Y_s := \sigma(s)\,\nabla_{x_s}\log u(s,x_s;\mathbf{x}_{0:i^*_s}).
\end{equation}
This is the Dol\'eans--Dade exponential (Theorem 3.5.1 of Karatzas--Shreve \cite{karatzas2014brownian})
with drift process $Y_s$.  By \ref{reg:novikov} (Novikov's condition), $Z_t$
is a true (not just local) $\PP$-martingale, confirming that the exponential
in \eqref{eqn:doleans} has mean $1$ and that all subsequent applications of
Girsanov's theorem are valid.

\noindent\textbf{Step 5: Girsanov's theorem — the $\QQ$-Brownian motion.}
Having written $Z_t$ in the required exponential form \eqref{eqn:doleans}
with $Y_s = \sigma(s)\nabla_{x_s}\log u$, Girsanov's theorem (\cite{girsanov1960transforming},
Theorem 3.5.1 of \cite{karatzas2014brownian}) states that
\begin{equation}\label{eqn:BQ}
  B_t := W_t - \int_0^t \sigma(s)\,\nabla_{x_s}\log u(s,x_s;\mathbf{x}_{0:i^*_s})\,ds
\end{equation}
is a standard Brownian motion under $\QQ$.  Equivalently:
\begin{equation}\label{eqn:dBrel}
  dW_t = dB_t + \sigma(t)\,\nabla_{x_t}\log u(t,x_t;\mathbf{x}_{0:i^*})\,dt.
\end{equation}
Under the new measure $\QQ$, the $\PP$-driver $W_t$ is no longer a martingale; $B_t$ is the
process that plays the role of standard Brownian motion, acquiring the additional drift
$\sigma(t)\nabla_{x_t}\log u$ relative to $W_t$.

\noindent\textbf{Step 6: Substitution into the OU SDE.}
Substitute \eqref{eqn:dBrel} into the base OU SDE \eqref{eqn:ou}:
\begin{align}
  dX_t
  &= -a(t)X_t\,dt + \sigma(t)\,dW_t \nonumber\\
  &= -a(t)X_t\,dt
     + \sigma(t)\left[dB_t + \sigma(t)\,\nabla_{x_t}\log u(t,x_t;\mathbf{x}_{0:i^*})\,dt\right]
     \nonumber\\
  &= \Bigl(-a(t)X_t + \sigma(t)^2\,\nabla_{x_t}\log u(t,x_t;\mathbf{x}_{0:i^*})\Bigr)dt
     + \sigma(t)\,dB_t. \label{eqn:new_sde_derived}
\end{align}
This is the claimed SDE \eqref{eqn:new_sde}.  Since $B^{\QQ}$ is a $\QQ$-Brownian
motion, the SDE \eqref{eqn:new_sde_derived} is driven by Gaussian increments under
$\QQ$, and by construction its finite dimensional law at times $t_1,\ldots,t_L$ is
$\pdata(\cdot\mid x_0)$.

\noindent\textbf{Step 7: Analytic form of $\nabla_{x_t}\log u$.}
It remains to derive the explicit form \eqref{eqn:expanded_score_main} of the score.
Fix $t\in(t_{i^*},t_{i^*+1})$ and expand $Z_t=u(t,x_t;\mathbf{x}_{0:i^*})$ from its definition
\eqref{eqn:Zt} by conditioning on the already-realized states
$\mathcal{F}_{t_{i^*}}=(x_0,x_{t_1},\ldots,x_{t_{i^*}})$:

\begin{align}
  Z_t
  &= \EE_\PP\left[
       \frac{\pdata(X_{t_1},\ldots,X_{t_L}\mid X_0)}
            {\pbase(X_{t_1},\ldots,X_{t_L}\mid X_0)}
       \,\bigg|\,\FF_t
     \right] \nonumber\\
  &= \underbrace{\frac{\pdata(x_{t_1},\ldots,x_{t_{i^*}}\mid x_0)}
                      {\pbase(x_{t_1},\ldots,x_{t_{i^*}}\mid x_0)}}_{=:\,k,\;\FF_{t_{i^*}}\text{-measurable}} \nonumber \\
     &*\int
       \frac{\pdata(x_{t_{i^*+1}},\ldots,x_{t_L}\mid x_0,\ldots,x_{t_{i^*}})}
            {\pbase(x_{t_{i^*+1}},\ldots,x_{t_L}\mid x_0,\ldots,x_{t_{i^*}})}
       \,\pbase(x_{t_{i^*+1}},\ldots,x_{t_L}\mid\FF_t)\,dx_{t_{i^*+1}}\cdots dx_{t_L}.
       \label{eqn:Zt_expand1}
\end{align}
By the Markov property \ref{reg:markov}, conditioned on $\FF_t$ the future increments
of the base process depend only on the current state $x_t$:
\[
  \pbase(x_{t_{i^*+1}},\ldots,x_{t_L}\mid\FF_t)
  = \pbase(x_{t_{i^*+1}}\mid x_t)\prod_{j=i^*+1}^{L-1}\pbase(x_{t_{j+1}}\mid x_{t_j}).
\]
Substituting into \eqref{eqn:Zt_expand1} and cancelling the telescoping Markov
product against the denominator of the data/base ratio (noting the base Markov
factorisation)
\begin{equation}\label{eqn:markov_factor}
  \pbase(x_{t_{i^*+1}},\ldots,x_{t_L}\mid x_{t_{i^*}})
  = \pbase(x_{t_{i^*+1}}\mid x_{t_{i^*}})\prod_{j=i^*+1}^{L-1}\pbase(x_{t_{j+1}}\mid x_{t_j}),
\end{equation}
the product of future base transitions in numerator and denominator telescopes, and
the integral over $x_{t_{i^*+2}},\ldots,x_{t_L}$ of the normalised data density equals one:
\begin{equation}\label{eqn:Zt_expand2}
  Z_t
  = k\int \pdata(x_{t_{i^*+1}}\mid x_0,\ldots,x_{t_{i^*}})
          \cdot\frac{\pbase(x_{t_{i^*+1}}\mid x_t)}{\pbase(x_{t_{i^*+1}}\mid x_{t_{i^*}})}
          \,dx_{t_{i^*+1}}.
\end{equation}
Applying Bayes' rule to the base process (we can apply Bayes' rule freely here because the Gaussian structure of the base process makes all conditional densities, whether forward or backward, analytically tractable),
\begin{equation}\label{eqn:bayes_base}
  \frac{\pbase(x_{t_{i^*+1}}\mid x_t)}{\pbase(x_{t_{i^*+1}}\mid x_{t_{i^*}})}
  = \frac{\pbase(x_t\mid x_{t_{i^*}},x_{t_{i^*+1}})}{\pbase(x_t\mid x_{t_{i^*}})},
\end{equation}
and collecting:
\begin{equation}\label{eqn:Zt_final}
  Z_t
  = \frac{k}{\pbase(x_t\mid x_{t_{i^*}})}
    \int \pdata(x_{t_{i^*+1}}\mid x_0,\ldots,x_{t_{i^*}})\,
         \pbase(x_t\mid x_{t_{i^*}},x_{t_{i^*+1}})\,dx_{t_{i^*+1}}.
\end{equation}
Taking $\nabla_{x_t}\log(\cdot)$ of \eqref{eqn:Zt_final}, and noting that $k$
is constant in $x_t$:
\begin{align}
  \nabla_{x_t}\log u(t,x_t;\mathbf{x}_{0:i^*})
  &= -\nabla_{x_t}\log\pbase(x_t\mid x_{t_{i^*}}) \nonumber\\
  &\quad+ \nabla_{x_t}\log\int
      \pdata(x_{t_{i^*+1}}\mid x_0,\ldots,x_{t_{i^*}})\,
      \pbase(x_t\mid x_{t_{i^*}},x_{t_{i^*+1}})\,
      dx_{t_{i^*+1}}.\label{eqn:score_derived}
\end{align}
This establishes \eqref{eqn:expanded_score_main} in Theorem \ref{thm:new_sde}.

\medskip
Combining Steps~1--7: the SDE \eqref{eqn:new_sde_derived} is driven by the
$\QQ$-Brownian motion $B_t$ and has oracle drift $s=\nabla_{x_t}\log u$ as given
by \eqref{eqn:score_derived}.  Conditionally on $X_0=x_0$ the finite dimensional law
at $(t_1,\ldots,t_L)$ is $\pdata(\cdot\mid x_0)$ by construction of $\QQ$ in
Step~1.  Since $X_0\sim\pdata(x_0)$ by assumption,
\[
  \mathrm{Law}(X_0,X_{t_1},\ldots,X_{t_L})
  = \pdata(x_0)\cdot\pdata(x_{t_1},\ldots,x_{t_L}\mid x_0)
  = \pdata(x_0,x_{t_1},\ldots,x_{t_L}).
\]
\end{proof}

\section{Learning the Score Function}
\label{appendix:C}

The oracle score $s(t,x_t;\mathbf{x}_{0:i^*}) = \nabla_{x_t}\log u(t,x_t;\mathbf{x}_{0:i^*})$ from Theorem~\ref{thm:new_sde} is defined implicitly through the Doob $h$-function.  This section makes its structure explicit, shows how it can be learned from data, and proves Theorem~\ref{thm:path_dependent_score}.

\subsection{Analytic Form of the Score}
\label{subsec:analytic_score_function}

Taking the log-gradient of the expression for $u$ derived in Step~7 of Appendix~\ref{appendix:B} (equation~\eqref{eqn:Zt_final}), the oracle score decomposes as

\begin{align}
    s(t, x_t; \mathbf{x}_{0:i^*})
    \;=\; &\nabla_{x_t}\log u(t, x_t; \mathbf{x}_{0:i^*})\\
    \;=\;
    &-\underbrace{\nabla_{x_t}\log \pbase(x_t \mid x_{t_{i^*}})}_{\text{analytic}}
    \nonumber \\
    &+\; \underbrace{\nabla_{x_t}\log \int
      \pdata(x_{t_{i^*+1}} \mid x_0, \ldots, x_{t_{i^*}})\,
      \pbase(x_t \mid x_{t_{i^*}}, x_{t_{i^*+1}})\,
      dx_{t_{i^*+1}}}_{\text{to be learned}}.
    \label{eqn:expanded_score}
\end{align}

The first term is the score of the OU marginal $\pbase(x_t \mid x_{t_{i^*}})$, which is Gaussian and therefore available in closed form at any $t \in (t_{i^*}, t_{i^*+1})$.  The second term is a log-normalizer over the next observation $x_{t_{i^*+1}}$ weighted by the data distribution, and is the term that must be approximated by a neural network.

The difficulty is that the first term diverges as $t \to t_{i^*}^+$: the OU marginal concentrates into a Dirac delta, making the gradient unbounded.  A network that directly targets the sum would need to handle this singularity.  We show below that the singularity cancels algebraically, leaving a well-posed training objective.

\bigskip

\begin{remark}[Boundary interpretation at $t = t_{i^*}$]
One might worry that when $t = t_{i^*}$ exactly, the gradient $\nabla_{x_t}$ in~\eqref{eqn:expanded_score} accidentally differentiates through $x_{t_{i^*}}$, since the two coincide numerically.  This is not the case for two reasons.  First, $t_{i^*}$ and $t$ are distinct variables that happen to be equal at the boundary; $\nabla_{x_t}$ and $\nabla_{x_{t_{i^*}}}$ remain separate operators (analogous to pass-by-value versus pass-by-reference in programming).  Second, It\^o calculus is non-anticipative: in the Riemann approximation of the It\^o integral, the integrand is evaluated at the left endpoint of each discretisation interval \cite{steele2001stochastic}, so $x_t$ is treated as a fixed value when taking $\nabla_{x_t}$, independently of $x_{t_{i^*}}$, which was already realised in $\mathcal{F}_t$.
\end{remark}

\subsection{Bridge Score Matching}
\label{sec:learn_score}

A natural first learning objective uses the bridge kernel $\pbase(x_t \mid x_{t_{i^*}}, x_{t_{i^*+1}})$ as the noising distribution.  By the denoising score matching identity \cite{vincent2011connection}, the minimizer of the squared-error loss against the bridge score $\nabla_{x_t}\log\pbase(x_t \mid x_{t_{i^*}}, x_{t_{i^*+1}})$ is the conditional expectation of that score given $(x_t, \mathbf{x}_{0:i^*})$, which equals the log-gradient of the mixture $\mathcal{M}(x_t) = \int \pdata(x_{t_{i^*+1}} \mid \mathbf{x}_{0:i^*})\,\pbase(x_t \mid x_{t_{i^*}}, x_{t_{i^*+1}})\,dx_{t_{i^*+1}}$.  Comparing with~\eqref{eqn:expanded_score}, this is $s + \nabla_{x_t}\log\pbase(x_t \mid x_{t_{i^*}})$, i.e.\ the oracle plus the singular term.

For a fixed trajectory segment and a score network $\mathbf{r}_\theta$, define the bridge loss

\begin{align}
  \mathcal{L}_\text{bridge}&(\theta;\, t,\, \mathbf{x}_{0:t_{i^*}},\, x_{t_{i^*+1}})
  \;:=\; \nonumber \\
  &\EE_{\pbase(x_t \mid x_{t_{i^*}}, x_{t_{i^*+1}})}\left[
    \bigl\|\mathbf{r}_\theta(t, t_{i^*+1}, x_t;\, \mathbf{x}_{0:i^*})
    - \nabla_{x_t}\log \pbase(x_t \mid x_{t_{i^*}}, x_{t_{i^*+1}})\bigr\|_2^2
  \right].\label{eqn:dsm_fixed_traj}
\end{align}

Averaging over all trajectory segments and times gives the joint objective:

\begin{equation}\label{eqn:dsm_joint_conditioned_objective}
  \mathcal{L}_\text{bridge}^\QQ(\theta)
  \;:=\;
  \EE_{t \sim \mathcal{U}(0,\,t_{L-1})}\;
  \EE_{\pdata(x_0, \ldots, x_{t_{i^*}})}\;
  \EE_{\pdata(x_{t_{i^*+1}} \mid \mathbf{x}_{0:i^*})}\;
  \mathcal{L}_\text{bridge}(\theta;\, t,\, \mathbf{x}_{0:t_{i^*}},\, x_{t_{i^*+1}}).
\end{equation}

This is a valid objective in the sense that its minimizer recovers the oracle (up to the singular additive term).  However, because the target $\nabla_{x_t}\log\pbase(x_t \mid x_{t_{i^*}}, x_{t_{i^*+1}})$ diverges as $t \to t_{i^*}^+$, it is more natural to reparametrize the network so that the diverging term is absorbed analytically.  This is carried out in the proof of Theorem~\ref{thm:path_dependent_score} below.

\subsection{Proof of Theorem~2}
\label{sec:proof_thm_2}
\begin{proof}[Proof of Theorem~\ref{thm:path_dependent_score}]

The proof has two steps.  Step~1 cancels the singularity via a reparametrization that converts the bridge loss into the simplified loss $\mathcal{L}_\text{DSM}$ of the theorem statement.  Step~2 identifies the minimizer of this simplified loss with the oracle score $s$.

\bigskip
\noindent \textbf{Step 1: Singularity cancellation.}

By Bayes' rule applied to the Markov base process:
\begin{equation}\label{eqn:bayes_grad}
  \nabla_{x_t}\log \pbase(x_t \mid x_{t_{i^*}}, x_{t_{i^*+1}})
  \;=\;
  \nabla_{x_t}\log \pbase(x_{t_{i^*+1}} \mid x_t)
  \;+\;
  \nabla_{x_t}\log \pbase(x_t \mid x_{t_{i^*}}).
\end{equation}

Define a reparametrized network $\mathbf{f}_\theta$ by
\begin{equation}\label{eqn:reparam}
  \mathbf{r}_\theta(t, t_{i^*+1}, x_t;\, \mathbf{x}_{0:i^*})
  \;:=\;
  \mathbf{f}_\theta(t, t_{i^*+1}, x_t;\, \mathbf{x}_{0:i^*})
  \;+\;
  \nabla_{x_t}\log \pbase(x_t \mid x_{t_{i^*}}),
\end{equation}
so that $\mathbf{f}_\theta$ carries no singular term (it contributes the singular term analytically through~\eqref{eqn:reparam}).  Substituting~\eqref{eqn:reparam} into the bridge loss~\eqref{eqn:dsm_fixed_traj} and applying~\eqref{eqn:bayes_grad}:
\begin{align}
  &\mathcal{L}_\text{bridge}(\theta) \nonumber \\
  &= \EE\left[
      \bigl\|
        \mathbf{f}_\theta
        + \nabla_{x_t}\log \pbase(x_t \mid x_{t_{i^*}})
        - \nabla_{x_t}\log \pbase(x_{t_{i^*+1}} \mid x_t)
        - \nabla_{x_t}\log \pbase(x_t \mid x_{t_{i^*}})
      \bigr\|^2
    \right]
  \nonumber \\
  &= \EE\left[
      \bigl\|
        \mathbf{f}_\theta(t, t_{i^*+1}, x_t;\, \mathbf{x}_{0:i^*})
        - \nabla_{x_t}\log \pbase(x_{t_{i^*+1}} \mid x_t)
      \bigr\|^2
    \right].
  \label{eqn:simplified_loss}
\end{align}

The two $\nabla_{x_t}\log \pbase(x_t \mid x_{t_{i^*}})$ terms cancel exactly.  The residual target $\nabla_{x_t}\log \pbase(x_{t_{i^*+1}} \mid x_t)$ is non-singular on $(t_{i^*}, t_{i^*+1})$ since $x_{t_{i^*+1}} \ne x_t$ throughout the open interval.  Averaging over all segments gives the full objective
\begin{align}
  &\mathcal{L}_\mathrm{DSM}^{\QQ}(\theta) \nonumber \\
  &\;=\;
  \EE_{t \sim \mathcal{U}(0,\,t_{L-1})}\;
  \EE_{\substack{\pdata(x_0, \ldots, x_{t_{i^*+1}}) \\[3pt]
                 \pbase(x_t \mid x_{t_{i^*}}, x_{t_{i^*+1}})}}
  \left[
    \bigl\|
      \mathbf{f}_\theta(t, t_{i^*+1}, x_t;\, \mathbf{x}_{0:i^*})
      - \nabla_{x_t}\log \pbase(x_{t_{i^*+1}} \mid x_t)
    \bigr\|^2
  \right],\label{eqn:loss_final}
\end{align}
which is $\mathcal{L}_\text{DSM}(\hat{s})$ from Theorem~\ref{thm:path_dependent_score} with $\hat{s} = \mathbf{f}_\theta$.  Both quantities inside the expectation are available in closed form from Lemma~\ref{lem:gaussian_kernels}: the noising kernel provides $x_t$ given $(x_{t_{i^*}}, x_{t_{i^*+1}})$, and the score target is given by~\eqref{eqn:score_target}.

\bigskip
\noindent \textbf{Step 2: The minimizer of $\mathcal{L}_\mathrm{DSM}^\QQ$ is the oracle score $s$.}

Fix $t \in (t_{i^*}, t_{i^*+1})$ and $\mathbf{x}_{0:i^*}$.  Let
$p(x_{t_{i^*+1}}, x_t \mid \mathbf{x}_{0:i^*}) := \pdata(x_{t_{i^*+1}} \mid \mathbf{x}_{0:i^*})\,\pbase(x_t \mid x_{t_{i^*}}, x_{t_{i^*+1}})$
denote the joint sampling distribution.  By the bias-variance decomposition:
\begin{align}
  &\EE\left[
    \bigl\|\mathbf{f}(t, x_t;\, \mathbf{x}_{0:i^*})
    - \nabla_{x_t}\log \pbase(x_{t_{i^*+1}} \mid x_t)\bigr\|^2
  \right]
  \nonumber \\[4pt]
  &\quad =\;
  \EE_{x_t}\left[
    \bigl\|\mathbf{f} - \EE_{x_{t_{i^*+1}} \mid x_t}\bigl[\nabla_{x_t}\log \pbase(x_{t_{i^*+1}} \mid x_t)\bigr]\bigr\|^2
  \right]
  +\; \text{const},
  \label{eqn:bias_var}
\end{align}

where the constant is independent of $\mathbf{f}$.  The unique minimizer is therefore the conditional mean
\begin{equation}\label{eqn:minimiser}
  \mathbf{f}^*(t, x_t;\, \mathbf{x}_{0:i^*})
  \;=\;
  \EE_{x_{t_{i^*+1}} \mid x_t,\, \mathbf{x}_{0:i^*}}\bigl[\nabla_{x_t}\log \pbase(x_{t_{i^*+1}} \mid x_t)\bigr],
\end{equation}

where $p(x_{t_{i^*+1}} \mid x_t, \mathbf{x}_{0:i^*}) \propto \pdata(x_{t_{i^*+1}} \mid \mathbf{x}_{0:i^*})\,\pbase(x_t \mid x_{t_{i^*}}, x_{t_{i^*+1}})$.

It remains to show $\mathbf{f}^* = s$.  Define $$\mathcal{M}(x_t) := \int \pdata(x_{t_{i^*+1}} \mid \mathbf{x}_{0:i^*})\,\pbase(x_t \mid x_{t_{i^*}}, x_{t_{i^*+1}})\,dx_{t_{i^*+1}},$$ the denominator of the conditional.  Writing out~\eqref{eqn:minimiser} explicitly:
\[
  \mathbf{f}^*
  \;=\;
  \frac{1}{\mathcal{M}(x_t)}
  \int
    \pdata(x_{t_{i^*+1}} \mid \mathbf{x}_{0:i^*})\;
    \pbase(x_t \mid x_{t_{i^*}}, x_{t_{i^*+1}})\;
    \nabla_{x_t}\log \pbase(x_{t_{i^*+1}} \mid x_t)\;
  dx_{t_{i^*+1}}.
\]

Apply~\eqref{eqn:bayes_grad} to substitute $$\nabla_{x_t}\log\pbase(x_{t_{i^*+1}}\mid x_t) = \nabla_{x_t}\log\pbase(x_t\mid x_{t_{i^*}},x_{t_{i^*+1}}) - \nabla_{x_t}\log\pbase(x_t\mid x_{t_{i^*}}),$$ we get
\begin{align*}
  \mathbf{f}^*
  \;=\;
  \frac{1}{\mathcal{M}(x_t)}
  &\int
    \pdata(x_{t_{i^*+1}} \mid \mathbf{x}_{0:i^*})\;
    \nabla_{x_t}\pbase(x_t \mid x_{t_{i^*}}, x_{t_{i^*+1}})\;
  dx_{t_{i^*+1}} \\
  &\;-\;
  \nabla_{x_t}\log \pbase(x_t \mid x_{t_{i^*}}).
\end{align*}

Interchanging $\nabla_{x_t}$ with the integral (justified since $\pbase(\cdot \mid x_{t_{i^*}}, x_{t_{i^*+1}})$ is Gaussian in $x_t$, so dominated convergence applies):
\begin{equation}\label{eqn:score_equals_f}
  \mathbf{f}^*
  \;=\;
  \frac{\nabla_{x_t}\mathcal{M}(x_t)}{\mathcal{M}(x_t)}
  \;-\;
  \nabla_{x_t}\log \pbase(x_t \mid x_{t_{i^*}})
  \;=\;
  \nabla_{x_t}\log \mathcal{M}(x_t)
  \;-\;
  \nabla_{x_t}\log \pbase(x_t \mid x_{t_{i^*}}).
\end{equation}

From~\eqref{eqn:Zt_final} in Appendix~\ref{appendix:B},
$u(t, x_t;\, \mathbf{x}_{0:i^*}) = k\,\mathcal{M}(x_t)/\pbase(x_t \mid x_{t_{i^*}})$ with $k > 0$ independent of $x_t$.  Therefore
\[
  s(t, x_t;\, \mathbf{x}_{0:i^*})
  \;=\;
  \nabla_{x_t}\log u
  \;=\;
  \nabla_{x_t}\log \mathcal{M}(x_t) - \nabla_{x_t}\log \pbase(x_t \mid x_{t_{i^*}})
  \;=\;
  \mathbf{f}^*.
\]

The oracle score $s$ is the unique minimizer of $\mathcal{L}_\mathrm{DSM}^\QQ$ (equivalently, $\mathcal{L}_\text{DSM}(\hat{s})$ with $\hat{s} = \mathbf{f}_\theta$).
\end{proof}

\subsection{Non-Singularity and Novikov's Condition}
\label{sec:tractable_score}

The simplified loss~\eqref{eqn:loss_final} eliminates both practical and theoretical concerns associated with the original two-term score~\eqref{eqn:expanded_score}.

On the practical side, the target $\nabla_{x_t}\log\pbase(x_{t_{i^*+1}} \mid x_t)$ is bounded and smooth on the interval $[t_{i^*}, t_{i^*+1})$ because $x_{t_{i^*+1}} \ne x_t$ by construction, and the singular term with $x_{t_{i^*}}$ got canceled out in \eqref{eqn:loss_final}.  The network $\mathbf{f}_\theta$ therefore has a bounded, non-singular regression target throughout training.

On the theoretical side, the SDE~\eqref{eqn:new_sde} in Theorem~\ref{thm:new_sde} requires Novikov's condition (Regularity~R6 in Appendix~\ref{appendix:B}):
\[
  \EE_\PP\left[
    \exp\left(
      \tfrac{1}{2}\int_0^{t_L}
      \|\sigma(s)\,\nabla_{x_s}\log u(s,x_s;\mathbf{x}_{0:i^*_s})\|^2\,ds
    \right)
  \right] < \infty.
\]
From~\eqref{eqn:reparam} and~\eqref{eqn:bayes_grad}, the learnable part of $\nabla_{x_t}\log u$ is $\mathbf{f}_{\theta^*} = s$, which depends on $x_{t_{i^*+1}} \ne x_t$ and is therefore non-singular on the open interval.  The singular term $\nabla_{x_t}\log\pbase(x_t \mid x_{t_{i^*}})$ is present in the oracle score $s$ analytically (through the reparametrization~\eqref{eqn:reparam}), but its contribution to the It\^o integral is square-integrable on $[t_{i^*}, t_{i^*+1})$ since it got canceled out in \eqref{eqn:loss_final}. 
The integrand is therefore square-integrable on every compact sub-interval, and Novikov's condition is satisfied \cite{novikov1972identity, girsanov1960transforming, steele2001stochastic}.

\section{Generic Conditional Diffusion Bridges Cannot Capture the Path Measure}\label{sec:proof_hacky_bridge_bad}
We prove Theorem \ref{thm:non_bijective_path_measure} here.

\begin{proof}
Marginal agreement at $(0, t_1, \ldots, t_L)$ holds by hypothesis. To see that the full path laws differ, compare quadratic variations, which are pathwise functionals and hence law-determined.

For $X$, quadratic variation is $\langle X \rangle_t = \int_0^t \sigma(s)^2\,ds$ (drift does not contribute) \cite{ito1944109, ito1951stochastic, karatzas2014brownian}. For $Y$, the deterministic time change on $[t_{i-1}, t_i]$ has rate $1/(t_i - t_{i-1})$, so on this subinterval $Y$ has quadratic variation density
\[
  s \;\mapsto\; \tilde{\sigma}\big((s - t_{i-1})/(t_i - t_{i-1})\big)^2 \,/\, (t_i - t_{i-1}).
\]
Equality of laws would force this to equal $\sigma(s)^2$ pointwise a.e.\ on every $(t_{i-1}, t_i)$ --- a rigid constraint independent of the marginal-matching hypothesis on $\tilde{\mathbf{f}}_\theta$. Generic choices violate it (\textit{e.g.}, $\tilde{\sigma}$ that does not depend on $i$ when $\exists i:t_i - t_{i-1} \neq t_L/L$, any constant $\tilde{\sigma}$ against a non-constant $\sigma$). So, $\mathrm{Law}(Y) \neq \mathrm{Law}(X)$.
\end{proof}

This shows that endpoint-matching alone does not identify the correct path law. In particular, unless the bridge volatility schedules are chosen to satisfy a specific time-rescaling compatibility condition, the chained-bridge construction leads to different path laws.

\section{Need for Path-Dependent Drift}

\begin{figure}[p]
    \centering

    \begin{subfigure}{0.47\textwidth}
        \centering
        \includegraphics[width=\linewidth]{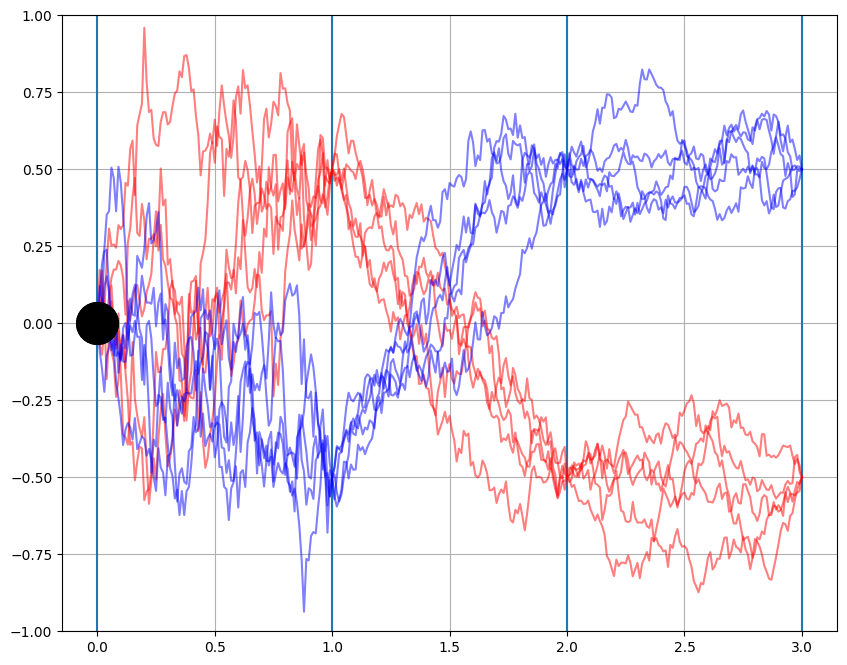}
        \caption{Data}
        \label{fig:sub1}
    \end{subfigure}
    \hfill
    \begin{subfigure}{0.47\textwidth}
        \centering
        \includegraphics[width=\linewidth]{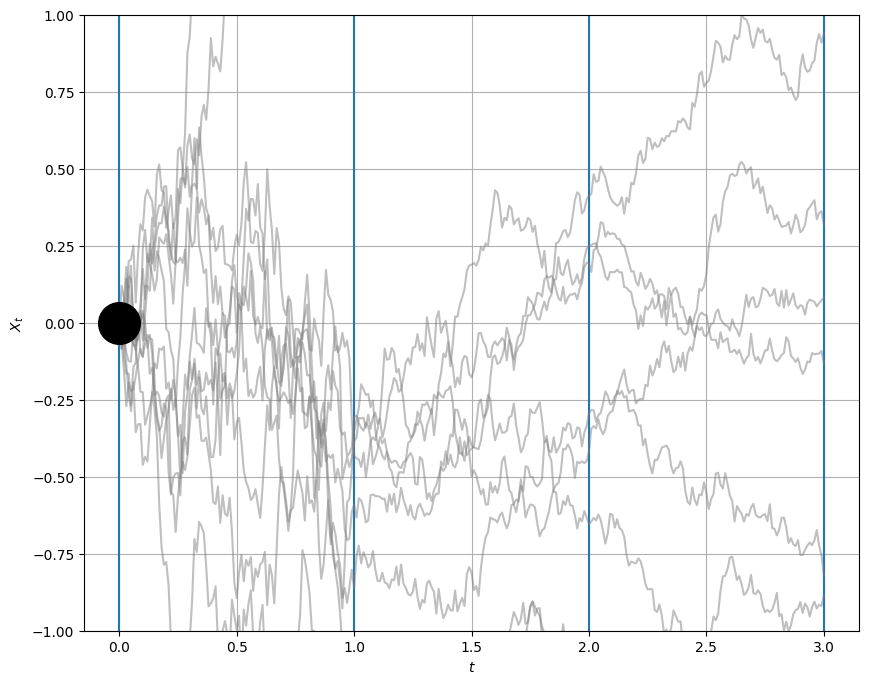}
        \caption{Base Process}
        \label{fig:sub2}
    \end{subfigure}
    \vspace{0.02\textheight}
    \begin{subfigure}{0.47\textwidth}
        \centering
        \includegraphics[width=\linewidth]{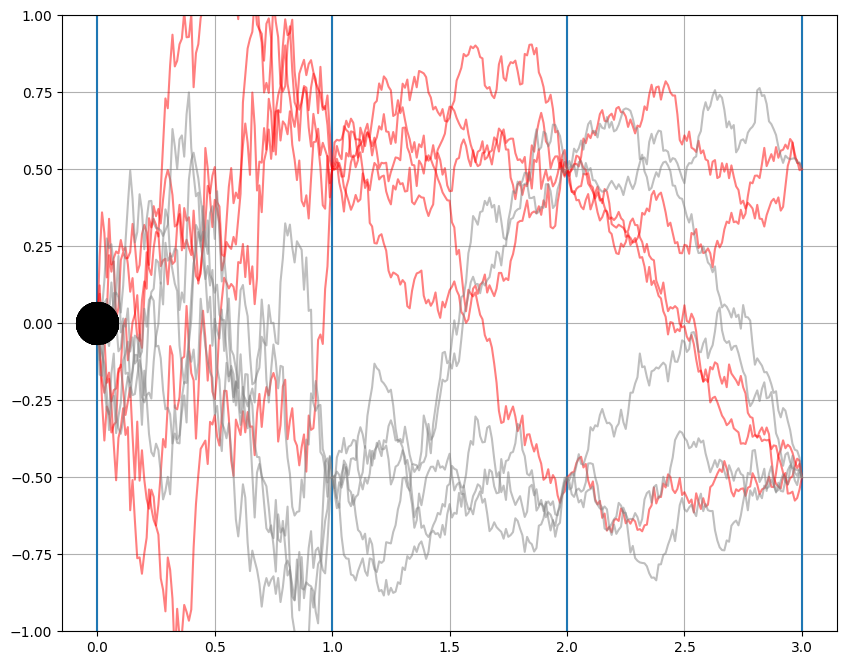}
        \caption{Markov}
        \label{fig:sub3}
    \end{subfigure}
    \hfill
    \begin{subfigure}{0.47\textwidth}
        \centering
        \includegraphics[width=\linewidth]{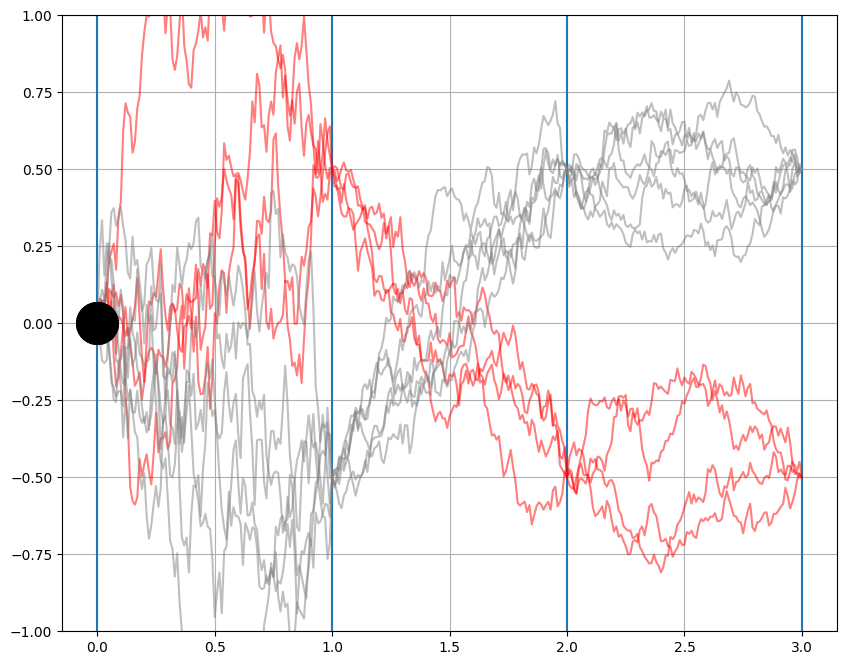}
        \caption{ABC}
        \label{fig:sub4}
    \end{subfigure}

    \caption{We can consider the case where the data process is not Markov. This can happen for several reasons. For example, in a physical system, a snapshot will only reveal the position but not the velocity of the item. Even when the process restricted to the finite observation times is a discrete-time Markov process, we still cannot ensure that the continuous time extension is Markov. Indeed, we can find one that is not possible. This happens since continuous time extension while keeping the Markov structure will enforce some non-flexible temporal correlation. We can see such occurence in the example provided in (a). This is a mixture of two Markov processes: the red process and the blue process. The mixture itself is non-Markov. In the region around time $t=1.5$, we will see that the drift for the red process tends to be negative, while that of the blue process tends to be positive. Still, we will see that the process restricted to the finite observations (which are at times $0,1,2,$ and $3$) is actually Markov: If $X_1 = 0.5$, then $X_2=X_3=-0.5$, and, if $X_1 = -0.5$, then $X_2=X_3=0.5$. However, we have that $X_1$ and $X_2$ have a correlation of $-1$, while $X_2$ and $X_3$ have a correlation of $1$. We then cannot find \textit{any} continuous-time Markov SDE process such that the correlation structures can be satisfied (even when we allow the volatility schedule to be different from the data process). By trying to match the marginal of the observed data at each observation time, we will introduce positive correlation. We can see that the majority of the red process, which are the processes with $X_1=0.5$, will also have $X_2=0.5$ instead of $-0.5$. 
    However, we will see that a similar problem also occurs from time $t=2$ to $3$. The value $X_2$ is supposed to be the same as $X_3$ under the data process. However, the volatility we add creates some mixing, thereby violating the joint finite-dimensional distribution matching desiderata. For our ABC, we start from the base process, which is data-independent, but with matching volatility schedule. The path-dependent (on finite time observation) Doob-h transform helps change the base process (b) into our ABC process (d).}
    \label{fig:nonmarkov}

\end{figure}

\newpage
\clearpage

\section{Algorithms}

\subsection{Training}
See Algorithm \ref{alg:train}. For simplicity, we denote it in the non-batched setting, but it is trivial to batch it, and our implementation is batched. We normalize the time horizon such that $t_L = 1$.

\begin{algorithm}
\caption{Training Procedure}\label{alg:train}
\begin{algorithmic}[1]
\Require Time series distribution $p_\text{data}$, SDE parameters $\sigma(t), a(t)$, Waypoint time sampling distribution $p_\text{waypoint\_times}$, Conditioning set size sampling distribution $p_\text{waypoint\_set\_size}$
\State Initialize $\mathbf{f}_\theta(x_t; t, t_{i^*+1}, \mathcal{K})$
\For{$\texttt{iter} = 1 \rightarrow \texttt{iter}_\text{max}$}
    \State Sample $|\mathcal{T_W}| \sim p_\text{waypoint\_set\_size}$ \texttt{\ \ //} $|\mathcal{T_W}| \in \mathbb{Z}^{+} \oplus \{0\}$
    \State $\mathcal{T_W} \gets \{\tau_1, \ldots, \tau_{|\mathcal{T_W}|}\} \sim p_\text{waypoint\_times}$ \texttt{\ \ // ascending order,} $\tau_i \in (0, 1)$
    \State $\mathcal{T_W} \gets \mathcal{T_W} \oplus \{0, 1\}$ \texttt{\ \ // include initial condition and terminal state}
    \State $t \sim \mathcal{U}[0, 1)$ \texttt{\ \ // current time (interpolates waypoints)}
    \State $\mathcal{T_P} \gets \{\tau_i | \tau_i \leq t \cap \tau_i \in \mathcal{T_W}\}$ \texttt{\ \ // observe all past times (causal)}
    \State $\mathcal{T_K} \gets \mathcal{T_P} \cup \text{RandomSubset}(\mathcal{T_W} \setminus \mathcal{T_P})$ \texttt{\ \ // conditioning can include future}
    \State $t_{i^{*}} \gets \text{max}\{\tau_i \in \mathcal{T_K} | \tau_i \leq t\}$ \texttt{\ \ // most recent waypoint time}
    \State $t_{i^{*} + 1} \gets \text{min}\{\tau_i \in \mathcal{T_W} | \tau_i > t\}$ \texttt{\ \ // next waypoint time}
    \State $x_0, x_{\tau_1}, \ldots, x_{\tau_{\left|\mathcal{T_W}\right|}}, x_1 \sim p_\text{data}\left(x_0, x_{\tau_1}, \ldots, x_{\tau_{\left|\mathcal{T_W}\right|}}, x_1\right)$ \texttt{\ \ // sample from joint data distribution in continuous time}
    \State $\mathcal{K} \gets \{(\tau, x_{\tau}) | \tau \in \mathcal{T_K}\}$ \texttt{\ \ // waypoints to condition on (past and future)}
    \State $x_t | x_{t_{i^*}}, x_{t_{i^*+1}} \overset{\mathbb{P}}{\sim} \mathcal{N}(\mu_{t | t_{i^*}, t_{i^*+1}}(x_{t_{i^*}}, x_{t_{i^*+1}}), \Sigma_{t | t_{i^*}, t_{i^*+1}})$ \texttt{\ \ // sample target, Eq \ref{eqn:sampling_dist_noising_kernel_main}}
    \State $\nabla_{x_t} \text{log } p_\text{base}(x_{t_{i^*+1}} | x_t) = \frac{\Phi(t, t_{i^*+1})}{C_t(t_{i^*+1}, t_{i^*+1})}(x_{t_{i^*+1}} - \Phi(t, t_{i^*+1})x_{t})$ \texttt{\ \ // score target, Eq \ref{eqn:score_target}}
    \State $\mathcal{L}(\theta) \gets w(t, t_{i^*}, t_{i^*+1})\|\mathbf{f}_\theta(x_t; t, t_{i^*+1}, \mathcal{K}) - \nabla_{x_t} \text{log } p_\text{base}(x_{t_{i^*+1}} | x_t)\|^2$ \texttt{\ \ // DSM, Eq \ref{eqn:simplified_dsm_main} (can be batched)}
    \State $\theta \gets \theta - \nabla_{\theta}\mathcal{L}(\theta)$ \texttt{\ \ // update (can be batched)}
\EndFor
\\
\Return $\mathbf{f}_\theta(x_t; t, t_{i^*+1}, \mathcal{K})$ \texttt{\ \ // trained score network}
\end{algorithmic}
\end{algorithm}

\subsection{Inference}
See Algorithm \ref{alg:inference}. 
Normally, we would use a constant $\Delta t$ discretization, but this may result in the discretization scheme overshooting the waypoints. As such, when we are near a waypoint, we decrease the timestep to prevent overshooting (Line \ref{alg:line:adjust_timestep}; for numerical stability, we set $\epsilon_t = 10^{-6}$). To ensure that we still maintain the same total number of steps, we add back that difference to the next timestep after crossing the waypoint. (This strategy relies on the distance between waypoints being at least twice as large as the discretization interval, which is generally true in most reasonable SDE simulations.)

This inference procedure applies to both our ABC model (Equation \ref{eqn:new_sde}), and the autoregressively chained diffusion bridge construction from Theorem \ref{thm:non_bijective_path_measure} (Equations \ref{eqn:cond_diffusion_bridge}, \ref{eqn:time_shifted_bridge}).
For noise-to-data generation, reset $x_t \sim \mathcal{N}(0, I)$ after Line \ref{line:observe_waypoint}, and let $x_0 \sim \mathcal{N}(0, I)$.

\begin{algorithm}
\caption{Inference Procedure}\label{alg:inference}
\begin{algorithmic}[1]
\Require Trained score network $\mathbf{f}_\theta(x_t; t, t_{i^*+1}, \mathcal{K})$, Waypoint times $\mathcal{T_W}$, Observed waypoints $\mathcal{K}$ ($\mathcal{T_K} \subseteq \mathcal{T_W}$, includes $(0, x_0)$ and possibly future conditioning), Number of steps $N$, Waypoint time exceedance threshold $\epsilon_t$, SDE parameters $\sigma(t), a(t)$
\State $t \leftarrow 0$
\State $\Delta t_\text{leftover} \leftarrow 0$
\While{$t < 1$}
    \State $t_{i^{*}} \gets \text{max}\{\tau_i \in \mathcal{T_K} | \tau_i \leq t\}$ \texttt{\ \ // most recent waypoint time}
    \State $t_{i^{*} + 1} \gets \text{min}\{\tau_i \in \mathcal{T_W} | \tau_i > t\}$ \texttt{\ \ // next waypoint time}
    \If{$t + \frac{1}{N} > t_{i^*+1} + \epsilon_t$}\label{alg:line:adjust_timestep}
        \State $\Delta t \leftarrow t_{i^*+1} - t + \epsilon_t$ \texttt{\ \ // Don't overshoot the waypoint}
        \State $\Delta t_\text{leftover} \leftarrow 1 / N - \Delta t$ \texttt{\ \ // Give this back later}
    \Else
        \State $\Delta t \leftarrow 1 / N + \Delta t_\text{leftover}$ \texttt{\ \ // Recover the leftover simulation time}
        \State $\Delta t_\text{leftover} \leftarrow 0$
    \EndIf
    \State $dB^\mathbb{Q}_t \sim \mathcal{N}(\mu=\mathbf{0}, \Sigma=(\Delta t)I)$ 
    \State $dB^\mathbb{P}_t \gets dB^\mathbb{Q}_t + \sigma(t)\mathbf{f}_\theta(x_t; t, t_{i^*+1}, \mathcal{K})\Delta t$ \texttt{\ \ // Eq \ref{eqn:dBrel}}
    \State $dx_t \gets -a(t)x_t\Delta t + \sigma(t) dB^\mathbb{P}_t$ \texttt{\ \ // Eq \ref{eqn:ou}}
    \State $x_{t + \Delta t} \gets x_t + dx_t$
    \For{$\tau_i \in \mathcal{T_W} \setminus \mathcal{T_K}$}
        \If{$\tau_i \leq t < \tau_i + \frac{1}{N}$}
            \State $\mathcal{K} \gets \mathcal{K} \oplus (t, x_t)$ \texttt{\ \ // we've observed another waypoint, $\tau_i \approx t$}\label{line:observe_waypoint}
        \EndIf
    \EndFor
\EndWhile
\\
\Return $\mathcal{K}$ \texttt{\ \ // trajectory at waypoints}
\end{algorithmic}
\end{algorithm}

\newpage
\clearpage

\section{Comparison to Denoising Diffusion Bridge Models}\label{sec:comparison}

Since our loss function and generative SDE bear some resemblance to the popular denoising diffusion bridge \cite{zhou2023denoising} and denoising diffusion \cite{song2020score} models, we show the similarities and differences.
We conduct the analysis on denoising diffusion bridge models (DDBM) \cite{zhou2023denoising}, as denoising diffusion models \cite{song2020score} are just a special case of DDBMs with $x_0 \sim \mathcal{N}(0, I)$.

\subsection{Similarity: Two-Point Case}

Let us consider the case where we only wish to model the marginal: $p_\text{data}(x_1 | x_0)$; that is, a two-point time series where $t_L = 1$. 
Then (omitting the $t_{i^*+1}$ argument, since it is always $1$), our ABC loss function (Equation \ref{eqn:simplified_dsm_main}) and dynamics (Equation \ref{eqn:new_sde}) become:
\begin{equation}\label{eqn:two_point_loss}
    \theta^* = \underset{\theta}{\text{argmin }} \displaystyle \mathbb{E}_{t \sim \mathcal{U}[0, 1)} \mathbb{E}_{\substack{p_\text{data}(x_0, \ldots, x_1) \\ p_\text{base}(x_t | x_0, x_1)}} \left[\left\|\mathbf{f}_\theta(t, x_0, x_t) - \nabla_{x_t} \text{log }p_\text{base}(x_1 | x_t)\right\|_2^2\right],
\end{equation} 
\begin{equation}\label{eqn:two_point_dynamics}
    dx_t = \left(-a(t)x_t + \sigma(t)^2 \mathbf{f}_\theta(t, x_0, x_t)\right)dt + \sigma(t)dB_t
\end{equation}

Let us also recall the DDBM equations (Equations \ref{eqn:trad_diffusion_bridge}, \ref{eqn:trad_dsm}):
\begin{align}
    dx_t &= \left(-a(t)x_t + \sigma(t)^2 \left[\mathbf{s_\theta}(t, x_t, x_0) - \nabla_{x_t}\text{log } p(x_0 | x_t)\right]\right)dt + \sigma(t)dB(t), x_0 \sim p_\text{data}(x_0) \nonumber \\
    \mathbf{\theta^*} &= \underset{\mathbf{\theta}}{\text{argmin }}\mathbb{E}_{t \sim \mathcal{U}[0, 1)}\underset{\substack{x_0, x_1 \sim p_\text{data} \\ x_t \sim p_\text{base}(x_t | x_0, x_1)}}{\mathbb{E}}\left[w(t)\|\mathbf{s_\theta}(t, x_t, x_0) - \nabla_{x_t}\text{log } p_\text{base}(x_t | x_0, x_1)\|^2\right] \nonumber
\end{align}

In that case, the loss functions and dynamics are very similar between ABC and diffusion bridge models. The losses both have the Gaussian noising kernel $p_\text{base}(x_t | x_0, x_1)$.
The only difference is that diffusion bridge models have (1) doubly-conditioned score target $\nabla_{x_t}\text{log } p_\text{base}(x_t | x_0, x_1)$ and (2) subtracted drift adjustment $\nabla_{x_t}\text{log }p_\text{base}(x_0 | x_t)$ from the score.
But, if we make the reparameterization to the diffusion bridge score network:
\begin{align}
    \mathbf{s}_\theta(t, x_t, x_0) &= \mathbf{b}_\theta(t, x_t, x_0) + \nabla_{x_t} \text{log }p_\text{base}(x_0 | x_t)
\end{align}
and apply Bayes' rule (and the Markov property):
\begin{align}
    \nabla_{x_t} \text{log }p_\text{base}(x_t | x_0, x_1) &= \nabla_{x_t} \left[\text{log }p_\text{base}(x_t | x_1) + \text{log }p_\text{base}(x_0 | x_t, x_1) - \text{log }p_\text{base}(x_0 | x_1)\right] \\
    &= \nabla_{x_t} \text{log }p_\text{base}(x_t | x_1) + \nabla_{x_t} \text{log }p_\text{base}(x_0 | x_t)
\end{align}
the loss (Eq. \ref{eqn:trad_dsm}, $w(t)$ omitted) and dynamics of diffusion bridge models (Eq. \ref{eqn:trad_diffusion_bridge}) reduce to:
\begin{equation}\label{eqn:simplified_loss_diffusion_bridge}
    \mathbf{\theta^*} = \underset{\mathbf{\theta}}{\text{argmin }}\mathbb{E}_{t \sim \mathcal{U}[0, 1)}\mathbb{E}_{x_0, x_1 \sim p_\text{data}}\mathbb{E}_{x_t \sim p_\text{base}(x_t | x_0, x_1)}[\|\mathbf{b_\theta}(t, x_t, x_0) - \nabla_{x_t}\text{log } p_\text{base}(x_t | x_1)\|^2],
\end{equation}
\begin{equation}\label{eqn:simplified_dynamics_diffusion_bridge}
    dx_t = \left(-a(t)x_t + \sigma(t)^2 \mathbf{b_\theta}(t, x_t, x_0)\right)dt + \sigma(t)dB_t
\end{equation}

The dynamics match (Equations \ref{eqn:two_point_dynamics}, \ref{eqn:simplified_dynamics_diffusion_bridge}), with appropriate choice of $\mathbf{b_\theta}, \mathbf{f_\theta}$. 
Now, the only difference is that the loss target is $\nabla_{x_t}\text{log }p_\text{base}(x_1 | x_t)$ in ABC (Equation \ref{eqn:two_point_loss}), and $\nabla_{x_t}\text{log }p_\text{base}(x_t | x_1)$ in diffusion bridge models (Equation \ref{eqn:simplified_loss_diffusion_bridge}). Intuitively, both of these targets pull the intermediate $x_t$ to be more "aligned" with the data $x_1$, although they represent slightly different things. ABC's $\nabla_{x_t}\text{log }p_\text{base}(x_1 | x_t)$ can be interpreted as pulling an initially unconditioned diffusion towards certain waypoints, while diffusion bridges' (and denoising diffusion models') $\nabla_{x_t}\text{log }p_\text{base}(x_t | x_1)$ can be interpreted as denoising the forward data-to-noise corruption.
This comes from the fact that we derive ABC with Girsanov's Theorem and change-of-measure \cite{girsanov1960transforming, doob1984classical} in only one time direction (Section \ref{sec:proof_thm_1}), while the other methods derived their method via Anderson's time-reversal \cite{anderson1982reverse, zhou2023denoising, song2020score}. 

\subsection{Difference: Path-Dependent Conditioning}

Another key difference is that, in general, our ABC has a path-dependent score network $\mathbf{f}_\theta(t,t_{i^*+1},x_t;\mathbf{x}_{0:i^*})$ (where $\mathbf{x}_{0:i^*}$ is a concatenation of states as defined in Equation \ref{eqn:all_states}), as opposed to the diffusion bridge score network that depends only on the current state and the start point: $\mathbf{s}_\theta(t, x_t, x_0)$. Conceptually, this extends the framework from Markovian to \textit{non-Markovian} autoregressive modeling in continuous time, which does not exist in prior work to our knowledge. 

This suggests that DDBM will fail to capture the correlation structure, for example between $x_{t_1}$ and $x_{t_2}$, while our approach can do so.







\subsection{Difference: Time-Adaptive Volatility and Noising Kernels}

Furthermore, even if we were to condition the diffusion bridge model's score network on the physical time, it still would not have the physical time-adaptive volatility that ABC has. The implication of this is that the dynamics of the generation process are fundamentally inconsistent with the underlying real-world dynamics. We showed this in Theorem \ref{thm:non_bijective_path_measure} and Section \ref{sec:proof_hacky_bridge_bad}.

\subsection{Difference: Any-Subset Autoregression}

Finally, the prior work \cite{zhou2023denoising, song2020score} does not consider any-subset autoregression. They only have the ability to bridge between two endpoints. In contrast, our framework allows us to model arbitrarily many finite-dimensional marginals of the SDE's path measure.
In practical terms, what this means is that our method is naturally suited for handling time-series, especially those that are irregularly sampled or have future conditioning, \text{e.g.}, infilling corrupted video files.
\newpage
\clearpage

\section{Comparison to Probabilistic Forecasting with Interpolants}\label{sec:pfi_comparison}

The stochastic interpolant framework is closely related to diffusion bridges and denoising diffusions \cite{albergo2023stochastic, chen2024probabilistic}. 
PFI (probabilistic forecasting with interpolants) uses state-to-state stochastic interpolants to generate time series data. We review some similarities and key differences.

\subsection{Difference: Discrete-Time Modeling}

PFI's stated aim is to model the discrete-time, continuous-space process $\{x_{kT}\}_{k \in \mathbb{Z}}$, where $T$ is the sampling interval and $k$ is the discrete time index. This already differs from our setting of \textit{continuous-time} autoregression.

\subsection{Connection: Autoregressive Diffusion Bridge Construction}

PFI generates samples with the auxiliary non-physical SDE (using similar notation to Equation \ref{eqn:cond_diffusion_bridge}):
\begin{equation}\label{eqn:pfi_sde}
    d\tilde{X}^{(k)}_{s} = \mathbf{b}^g_\theta\left(s, x_{kT}, \tilde{X}^{(k)}_{s}\right)ds + g(s)dB_s, \tilde{X}^{(k)}_{0} = x_{kT},
\end{equation}
where $k$ is the discrete index of the state, $s \in [0, 1]$ is a fake time variable they construct to interpolate between consecutive points in the time series, $\tilde{X}^{(k)}_{s}$ is the value of the auxiliary "bridge" SDE, and $\tilde{X}^{(k)}_{1} \sim p\left(x_{(k+1)T} | x_{kT}\right)$. The superscript $^g$ on $\mathbf{b}^g_\theta\left(s, x_{kT}, \tilde{X}^{(k)}_{s}\right)$ indicates that the drift is dependent on the volatility schedule.

Indeed, this is highly similar to the autoregressive diffusion bridge construction in Equation \ref{eqn:cond_diffusion_bridge}, if we set $\tilde{A}=0$ and assume regular time sampling intervals:
\begin{equation}\label{eqn:autoregressive_bridge_a_0}
    d\tilde{X}^{(i)}_{\tau} = \tilde{\sigma}(\tau)^2\tilde{\mathbf{f}}_\theta\left(\tau, \mathbf{x}_{0:i-1}, \tilde{X}^{(i)}_{\tau}\right)d\tau + \tilde{\sigma}(\tau)dB^{(i)}_{\tau}, \tilde{X}^{(i)}_{0}=x_{t_{i-1}}
\end{equation}

Note the correspondence of variables: $i \leftrightarrow k, \tau \leftrightarrow s, x_{t_{i-1}} \leftrightarrow x_{kT}$.
As such, to equate Equations \ref{eqn:autoregressive_bridge_a_0} and \ref{eqn:pfi_sde}, we just need:
\begin{align}
    \mathbf{x}_{0:i-1} &= x_{t_{i-1}} \texttt{ \ //In Equation \ref{eqn:autoregressive_bridge_a_0}}\\
    g(s) &= \tilde{\sigma}(\tau) \texttt{ \ //match Eq \ref{eqn:pfi_sde}'s \& \ref{eqn:autoregressive_bridge_a_0}'s volatilities}\\
    \mathbf{b}^g_\theta\left(s, x_{kT}, \tilde{X}^{(k)}_{s}\right) &= \tilde{\sigma}(\tau)^2\tilde{\mathbf{f}}_\theta\left(\tau, \mathbf{x}_{0:i-1}, \tilde{X}^{(i)}_{\tau}\right) \texttt{ \ //match Eq \ref{eqn:pfi_sde}'s \& \ref{eqn:autoregressive_bridge_a_0}'s drifts} \label{eqn:match_drifts}
\end{align}

Whether the equality in Equation \ref{eqn:match_drifts} is achieved depends on what PFI's trained network learned (which depends on the choice of interpolant schedule). In principle, however, both learned drifts have the same high-level intuition: they "pull towards" the data at the next waypoint time.

\subsection{Difference to Connection: Constructing ABC-Style Time-Adaptive Volatility}

Considering the correspondence between Equations \ref{eqn:pfi_sde} and \ref{eqn:autoregressive_bridge_a_0}, PFI again runs into the same problem that autoregressively chained diffusion bridges do: the volatility doesn't naturally scale with physical time elapsed (Theorem \ref{thm:non_bijective_path_measure}). So, in the naive application, this leads to incorrect dynamics.

It's worth noting that via the Fokker-Planck equation \cite{risken1989fokker}, they do provide a mechanism for tuning the diffusion coefficient $g(s)$ and adjusting the drift accordingly, such that the single-time marginal laws of the resultant SDE are still the same as the single-time marginal laws of the data-generating SDE they had originally trained. This "trades diffusion for transport" \cite{chen2024probabilistic}.

In principle, we could tune this diffusion coefficient to match ABC's inductive bias that volatility should scale with physical time: \textit{e.g.}, $g(s) = \sqrt{T} * g_\text{const}$ if the model was originally trained with constant volatility. They do not consider this in their work (likely because they used regularly sampled time intervals, which is another limitation of their work).

\subsection{Difference: Markovianization}

Their experiments make the assumption that the transition kernel of the time series is physical time-invariant and Markov, which is not true in many important cases, \textit{e.g.}, stock data driven by fractional Brownian motion. As a consequence, they do not condition their models on the physical time variable, nor comprehensive past history (see Equation \ref{eqn:pfi_sde}). The lack of conditioning on the physical time variable becomes problematic when we have a non-Markov and/or time-varying process, where we need to know the time relations between observed keypoints. For instance, in their video generation experiments, they randomly conditioned their model on a time-stamped previous frame as well as the most recent frame but \textit{without} the timestamp, making it very difficult to ascertain where the previous history frame was in relation to the most recent frame, thereby hindering the estimation of the dynamics.

\subsection{Difference: Any-Subset Modeling}

They do not consider the any-subset inference case. In particular, they only deal with \textit{causal} conditioning on a maximum of three states (previous waypoint, current auxiliary generative state, and sometimes one waypoint from the history, making for $O(N)$ conditions), on a regularly spaced physical time grid. They can also only infer the next waypoint in the regular discretized time grid. In contrast, our method is explicitly designed as an any-subset autoregressive model that can condition on any subset of observations in the time series, even those in the \textit{future}. In the case of regularly spaced grids, we would have $O(2^N)$ capacity, versus their $O(N)$ capacity, but we also allow for irregularly spaced physical time grids. Our inference is also not limited to the immediate next waypoint on a fixed time grid, but is flexible to infer future states at arbitrary timestamps.

\newpage
\clearpage

\section{Additional Results}\label{sec:more_results}

\subsection{Sensitivity of ABC to Volatility Hyperparameter}

Table \ref{tab:sensitivity_sigma} shows what happens when we change ABC's volatility hyperparameter $\sigma$. Although we did not have the computational resources to retrain ABC with a large hyperparameter sweep (each run takes 36 hours on 4 A100 80GB GPUs), we found that there seems to be a tolerable range of hyperparameters, as the alternate choices we tested did not significantly degrade the results.

Additionally, on the strictly causal training runs for CelebV-HQ (Table \ref{tab:celebvhq_causal_CAUSAL_EVAL_fvd}, \ref{tab:celebvhq_causal_CAUSAL_EVAL_fid}), we had the chance to test a few more hyperparameter choices of ABC $\sigma \in \{0.25, 0.5, 0.75\}$. We found that overall, \textit{any} of these choices for ABC beat the best of the conditional diffusion bridge methods (Table \ref{tab:celebvhq_causal_overall_avg_rank}).

\begin{table}[h]
    \centering
    \caption{\textbf{Sensitivity of ABC to Volatility:} Infill videos of 32 frames, prompt every 8 frames, sample with 250 SDE steps. No Brownian Bridge.}\label{tab:sensitivity_sigma}
    \begin{tabular}{l r r}
        \toprule
        \textit{Setting} & \textit{FVD} $(\downarrow)$ & \textit{FID} $(\downarrow)$ \\
        \hline
        \rowcolor{gray!20}\multicolumn{3}{l}{\textbf{CelebV-HQ}} \\
        \quad $\sigma=0.4$ & 99.7 & 2.6 \\
        \quad $\sigma=0.5$ & 111.5 & 2.9 \\
        \hline
        \rowcolor{gray!20}\multicolumn{3}{l}{\textbf{Sky-Timelapse}} \\
        \quad $\sigma=0.3$ & 183.5 & 4.8 \\
        \quad $\sigma=0.4$ & 162.3 & 4.4 \\
        \bottomrule
    \end{tabular}
\end{table}

\newpage
\clearpage

\subsection{Ablation on Conditioning}\label{sec:cond_ablation}

Towards answering \rqref{asarm_continuous}, we conduct ablations on how much conditioning is necessary for the model to achieve high-quality generations, \textit{i.e.}, how much we get from a path-dependent formulation. For simplicity, we test on causal rollout. 
We test with a few conditioning options: 
\begin{enumerate}
    \item \textit{Full autoregressive conditioning} on all the waypoints elapsed in the path.
    \item \textit{Monte Carlo estimate} (single waypoint, uniformly sampled) of the intermediate path history, along with times $0$ and $t_{i^*}$ (that is, initial condition and most recent waypoint observation), making for a total of three conditioning points. This is \textit{similar to PFI's \cite{chen2024probabilistic} conditioning strategy}, except that they omitted the initial condition.
    \item \textit{Near-Markov} conditioning, where we only use the initial condition from time $0$ and the most recent waypoint at time $t_{i^*}$ as conditioning. This is \textit{similar to LFM's \cite{islam2025longitudinal} Markovian conditioning strategy}, except that they omitted the initial condition.
    \item \textit{Prefix only} conditioning, where we only condition on the prompt frames from the prefix that are before time $t$.
\end{enumerate}

See Tables \ref{tab:cond_ablation_fvd_score_celebvhq}, \ref{tab:cond_ablation_fid_score_celebvhq}, \ref{tab:cond_ablation_fvd_score_sky_timelapse}, \ref{tab:cond_ablation_fid_score_sky_timelapse}.
We can draw a few conclusions from this experiment:
\begin{itemize}
    \item Our \textit{path-dependent} conditioning provides the best results on a majority of scenarios: 15 out of 24 tests on CelebV-HQ, and on 18 out of 24 tests on Sky-Timelapse. This confirms that \textit{non-Markovian dependence in the drift is beneficial for generative modeling of time series, underscoring the benefit of our modeling approach over that of previous works (PFI \cite{chen2024probabilistic}, LFM \cite{islam2025longitudinal})}.
    \item To give the other approaches credit, having less conditioning can still be competitive (except for prefix-only). This actually supports \rqref{asarm_continuous}, as it means that our model effectively learned to extract information across a variety of subsets of the path history. This also means it could be possible to mostly maintain quality while reducing computational resources used; the optimal settings for this can be a subject of exploration in future work.
\end{itemize}

\begin{table}[h]
\centering
\caption{Conditioning ablation, FVD (lower is better). Best per column in \textbf{bold}. CelebV-HQ dataset. Model is causally trained ABC with $\sigma=0.5$ and no Brownian Bridge (see Section \ref{sec:celebv_hq_causal}).}
\label{tab:cond_ablation_fvd_score_celebvhq}
\resizebox{\linewidth}{!}{%
\begin{tabular}{lcccccccccccc}
\toprule
 & \multicolumn{6}{c}{32 frames} & \multicolumn{6}{c}{16 frames} \\
\cmidrule(lr){2-7}\cmidrule(lr){8-13}
 & \multicolumn{2}{c}{prefix\_length\_1} & \multicolumn{2}{c}{prefix\_length\_4} & \multicolumn{2}{c}{prefix\_length\_8} & \multicolumn{2}{c}{prefix\_length\_1} & \multicolumn{2}{c}{prefix\_length\_4} & \multicolumn{2}{c}{prefix\_length\_8} \\
\cmidrule(lr){2-3}\cmidrule(lr){4-5}\cmidrule(lr){6-7}\cmidrule(lr){8-9}\cmidrule(lr){10-11}\cmidrule(lr){12-13}
Conditioning & step\_250 & step\_500 & step\_250 & step\_500 & step\_250 & step\_500 & step\_250 & step\_500 & step\_250 & step\_500 & step\_250 & step\_500 \\
\midrule
Full AR Cond $(0,\ \ldots,\ t_{i^{*}})$ & \textbf{303.93} & \textbf{317.08} & 350.54 & 248.25 & 281.24 & 199.37 & \textbf{281.84} & 278.45 & \textbf{209.82} & \textbf{174.93} & \textbf{117.75} & \textbf{96.14} \\
Monte Carlo $(0,\ t_{\mathrm{rand}},\ t_{i^{*}})$ & 332.86 & 328.25 & \textbf{310.63} & \textbf{217.15} & \textbf{224.36} & \textbf{160.10} & 297.76 & 275.99 & 241.55 & 201.88 & 142.10 & 117.08 \\
Near-Markov $(0,\ t_{i^{*}})$ & 364.95 & 329.51 & 344.95 & 242.50 & 260.68 & 185.04 & 311.54 & \textbf{270.01} & 289.44 & 226.39 & 163.98 & 128.18 \\
Prefix Only $(0,\ \ldots,\ t_{\text{pre len}})$ & 448.23 & 472.08 & 426.77 & 447.01 & 343.07 & 374.03 & 456.68 & 461.18 & 353.44 & 374.38 & 244.85 & 262.89 \\
\bottomrule
\end{tabular}%
}
\end{table}

\begin{table}[h]
\centering
\caption{Conditioning ablation, FID (lower is better). Best per column in \textbf{bold}. CelebV-HQ dataset. Model is causally trained ABC with $\sigma=0.5$ and no Brownian Bridge (see Section \ref{sec:celebv_hq_causal}).}
\label{tab:cond_ablation_fid_score_celebvhq}
\resizebox{\linewidth}{!}{%
\begin{tabular}{lcccccccccccc}
\toprule
 & \multicolumn{6}{c}{32 frames} & \multicolumn{6}{c}{16 frames} \\
\cmidrule(lr){2-7}\cmidrule(lr){8-13}
 & \multicolumn{2}{c}{prefix\_length\_1} & \multicolumn{2}{c}{prefix\_length\_4} & \multicolumn{2}{c}{prefix\_length\_8} & \multicolumn{2}{c}{prefix\_length\_1} & \multicolumn{2}{c}{prefix\_length\_4} & \multicolumn{2}{c}{prefix\_length\_8} \\
\cmidrule(lr){2-3}\cmidrule(lr){4-5}\cmidrule(lr){6-7}\cmidrule(lr){8-9}\cmidrule(lr){10-11}\cmidrule(lr){12-13}
Conditioning & step\_250 & step\_500 & step\_250 & step\_500 & step\_250 & step\_500 & step\_250 & step\_500 & step\_250 & step\_500 & step\_250 & step\_500 \\
\midrule
Full AR Cond $(0,\ \ldots,\ t_{i^{*}})$ & \textbf{15.30} & \textbf{7.05} & 22.65 & 13.74 & 14.74 & 9.34 & \textbf{9.41} & \textbf{6.27} & \textbf{10.93} & \textbf{7.59} & \textbf{3.82} & \textbf{2.73} \\
Monte Carlo $(0,\ t_{\mathrm{rand}},\ t_{i^{*}})$ & 16.52 & 7.94 & \textbf{22.60} & 13.25 & \textbf{12.92} & \textbf{7.51} & 10.29 & 6.72 & 12.77 & 8.65 & 4.70 & 3.01 \\
Near-Markov $(0,\ t_{i^{*}})$ & 16.52 & 8.31 & 22.86 & \textbf{13.23} & 14.17 & 7.94 & 10.94 & 7.17 & 14.62 & 9.72 & 5.14 & 3.16 \\
Prefix Only $(0,\ \ldots,\ t_{\text{pre len}})$ & 33.05 & 43.07 & 55.24 & 61.59 & 35.89 & 41.35 & 36.04 & 41.57 & 33.64 & 37.94 & 12.47 & 13.82 \\
\bottomrule
\end{tabular}%
}
\end{table}

\begin{table}[h]
\centering
\caption{Conditioning ablation, FVD (lower is better). Best per column in \textbf{bold}. Sky Timelapse dataset. ABC $\sigma=0.4$ without Brownian Bridge.}
\label{tab:cond_ablation_fvd_score_sky_timelapse}
\resizebox{\linewidth}{!}{%
\begin{tabular}{lcccccccccccc}
\toprule
 & \multicolumn{6}{c}{32 frames} & \multicolumn{6}{c}{16 frames} \\
\cmidrule(lr){2-7}\cmidrule(lr){8-13}
 & \multicolumn{2}{c}{prefix\_length\_1} & \multicolumn{2}{c}{prefix\_length\_4} & \multicolumn{2}{c}{prefix\_length\_8} & \multicolumn{2}{c}{prefix\_length\_1} & \multicolumn{2}{c}{prefix\_length\_4} & \multicolumn{2}{c}{prefix\_length\_8} \\
\cmidrule(lr){2-3}\cmidrule(lr){4-5}\cmidrule(lr){6-7}\cmidrule(lr){8-9}\cmidrule(lr){10-11}\cmidrule(lr){12-13}
Conditioning & step\_250 & step\_500 & step\_250 & step\_500 & step\_250 & step\_500 & step\_250 & step\_500 & step\_250 & step\_500 & step\_250 & step\_500 \\
\midrule
Full AR Cond $(0,\ \ldots,\ t_{i^{*}})$ & \textbf{369.06} & \textbf{285.94} & \textbf{357.69} & \textbf{276.78} & \textbf{297.03} & \textbf{229.90} & \textbf{242.40} & \textbf{207.81} & \textbf{186.79} & \textbf{164.61} & \textbf{119.81} & \textbf{106.15} \\
Monte Carlo $(0,\ t_{\mathrm{rand}},\ t_{i^{*}})$ & 377.28 & 302.48 & 391.86 & 323.46 & 333.51 & 275.35 & 275.88 & 228.71 & 254.09 & 213.07 & 181.17 & 150.86 \\
Near-Markov $(0,\ t_{i^{*}})$ & 405.34 & 332.55 & 421.13 & 357.90 & 368.54 & 313.05 & 302.90 & 258.34 & 300.64 & 268.23 & 219.64 & 191.57 \\
Prefix Only $(0,\ \ldots,\ t_{\text{pre len}})$ & 928.86 & 923.94 & 871.58 & 874.81 & 730.85 & 738.43 & 994.55 & 991.96 & 733.17 & 734.32 & 467.81 & 468.97 \\
\bottomrule
\end{tabular}%
}
\end{table}

\begin{table}[h]
\centering
\caption{Conditioning ablation, FID (lower is better). Best per column in \textbf{bold}. Sky Timelapse dataset. ABC $\sigma=0.4$ without Brownian Bridge.}
\label{tab:cond_ablation_fid_score_sky_timelapse}
\resizebox{\linewidth}{!}{%
\begin{tabular}{lcccccccccccc}
\toprule
 & \multicolumn{6}{c}{32 frames} & \multicolumn{6}{c}{16 frames} \\
\cmidrule(lr){2-7}\cmidrule(lr){8-13}
 & \multicolumn{2}{c}{prefix\_length\_1} & \multicolumn{2}{c}{prefix\_length\_4} & \multicolumn{2}{c}{prefix\_length\_8} & \multicolumn{2}{c}{prefix\_length\_1} & \multicolumn{2}{c}{prefix\_length\_4} & \multicolumn{2}{c}{prefix\_length\_8} \\
\cmidrule(lr){2-3}\cmidrule(lr){4-5}\cmidrule(lr){6-7}\cmidrule(lr){8-9}\cmidrule(lr){10-11}\cmidrule(lr){12-13}
Conditioning & step\_250 & step\_500 & step\_250 & step\_500 & step\_250 & step\_500 & step\_250 & step\_500 & step\_250 & step\_500 & step\_250 & step\_500 \\
\midrule
Full AR Cond $(0,\ \ldots,\ t_{i^{*}})$ & 19.74 & 14.22 & \textbf{16.84} & 12.29 & \textbf{11.62} & 8.13 & 10.10 & 8.60 & \textbf{6.65} & \textbf{5.45} & \textbf{2.55} & \textbf{2.17} \\
Monte Carlo $(0,\ t_{\mathrm{rand}},\ t_{i^{*}})$ & \textbf{18.31} & \textbf{12.80} & 17.27 & \textbf{11.55} & 12.17 & \textbf{7.95} & \textbf{10.04} & \textbf{8.45} & 7.02 & 5.66 & 2.84 & 2.35 \\
Near-Markov $(0,\ t_{i^{*}})$ & 19.08 & 13.31 & 18.92 & 13.14 & 13.87 & 9.19 & 10.13 & 8.49 & 7.41 & 5.86 & 3.11 & 2.50 \\
Prefix Only $(0,\ \ldots,\ t_{\text{pre len}})$ & 253.48 & 256.25 & 225.76 & 227.64 & 177.73 & 180.75 & 215.40 & 216.65 & 148.42 & 150.49 & 71.14 & 73.63 \\
\bottomrule
\end{tabular}%
}
\end{table}

\newpage
\clearpage

\subsection{Exploring Impact of Scaling}

We also investigate whether increasing computational resources can improve model performance. 
We launched another training run on CelebV-HQ, this time with a 594,067,472 parameter model (depth=20, hidden\_size=1024, num\_heads=16). We increased the number of scheduled cosine decay steps to 380K, with 20K warmup steps; and increased the total number of training steps to 320K. We used 8 A100 80GB GPUs for this.

See Tables \ref{tab:scaling_ablation_fid_score}, \ref{tab:scaling_ablation_fvd_score}.
The experiments are done with the infilling setting. We see that on FVD (Table \ref{tab:scaling_ablation_fvd_score}), increasing parameter count and training time boosts the performance, especially as the prompting gets sparser. With respect to FID (Table \ref{tab:scaling_ablation_fid_score}), while the model with the most resources is not always the best, we note that all the FIDs are low to begin with, and the differences between FIDs are much smaller than those between FVDs.

\begin{table}[h]
\centering
\caption{Model-scaling ablation, FVD (lower is better) on CelebV-HQ. We fix the video length to 32 frames and use 250 diffusion steps. Best per column in \textbf{bold}.}
\label{tab:scaling_ablation_fvd_score}
\begin{tabular}{lccc}
\toprule
Model & pin\_every\_4 & pin\_every\_8 & pin\_every\_16 \\
\midrule
Base (202M params, 240K steps) & 37.75 & 99.69 & 315.49 \\
Medium (594M params, 240K steps) & 36.40 & 94.17 & 268.67 \\
Medium (594M params, 320K steps) & \textbf{35.02} & \textbf{88.47} & \textbf{238.77} \\
\bottomrule
\end{tabular}
\end{table}

\begin{table}[h]
\centering
\caption{Model-scaling ablation, FID (lower is better) on CelebV-HQ. We fix the video length to 32 frames and use 250 diffusion steps. Best per column in \textbf{bold}.}
\label{tab:scaling_ablation_fid_score}
\begin{tabular}{lccc}
\toprule
Model & pin\_every\_4 & pin\_every\_8 & pin\_every\_16 \\
\midrule
Base (202M params, 240K steps) & 1.29 & \textbf{2.62} & \textbf{5.91} \\
Medium (594M params, 240K steps) & 1.26 & 2.73 & 6.12 \\
Medium (594M params, 320K steps) & \textbf{1.21} & 2.67 & 6.51 \\
\bottomrule
\end{tabular}
\end{table}

\newpage
\clearpage

\subsection{Time Super-Resolution}

Given the continuous-time formulation, we are also interested to see whether ABC can generalize to times not seen during training. 
As such, we trained ABC on Sky Timelapse with the same settings ($\sigma=0.4$) as before, except that we eliminate roughly half of the time grid from the training data. Specifically, we excluded every odd-indexed frame (under 0-based indexing) except the final frame from the possible predicted or conditioning times during training. (So, out of 64 frames, the model was only allowed to see some subset of frames $0, 2, 4, \ldots, 62, 63$.) 
Then, we evaluated it to generate the full time-resolution data, including those parts of the time grid that were not seen during training (that is, all 64 frames).

We then compare against an ABC model that was trained on the full-grid data. The experiment shows that if we include the Brownian bridge inductive bias from \eqref{eqn:residual_bb} to the score network, it can generalize fairly well to times that were not seen during training. Out of all the compared methods, it has the best frame-wise FID. On FVD (Pin Every 16), it even beat the network that was trained on the whole time grid albeit without the Brownian bridge inductive bias. However, a network trained on the partial time grid without having a Brownian bridge inductive bias generally fails when predicting times not seen during training. The best-performing network in terms of FVD is that which has the full-grid training, along with the Brownian bridge inductive bias. See Table \ref{tab:time_superres}.

\begin{table}[h]
    \centering
    \begin{tabular}{l r r r r}
    \toprule
        & \multicolumn{2}{c}{Pin Every 8 } & \multicolumn{2}{c}{Pin Every 16 } \\
        Model & \multicolumn{1}{c}{FVD} & \multicolumn{1}{c}{FID} & \multicolumn{1}{c}{FVD} & \multicolumn{1}{c}{FID} \\
        \midrule
        Full-Grid Train, No BB & 107.35 & 2.94 & 323.65 & 7.68 \\
        Full-Grid Train, W/ BB & 77.17 & 2.45 & 244.83 & 5.20 \\
        Half-Grid Train, No BB & 875.72 & 97.06 & 1003.98 & 121.58 \\
        Half-Grid Train, W/ BB & 155.27 & 1.96 & 300.27 & 3.33 \\
        \bottomrule
    \end{tabular}
    \caption{Time super-resolution. Videos of 64 frames total, 500 inference steps. All models have $\sigma=0.4$. Sky-Timelapse dataset.}
    \label{tab:time_superres}
\end{table}

\newpage
\clearpage

\subsection{Satisfaction of Future Constraints}

We conduct an experiment on Sky-Timelapse where we generate 16-frame videos, having future knowledge of the 4th, 8th, 12th, and 16th frames. Over 1360 generated videos, we measure the adherence of our model to these ground truth frames. Just via training the model with our loss function and without any inductive architectural bias, we achieve near-perfect adherence to the conditioning, as measured by SSIM (scale of 0-1, 1 is exact reconstruction) and PSNR (values above 40 are considered to be near-perfect reconstructions). See Table~\ref{tab:constraint_satisfcation}.

\begin{table}[h]
\centering
\begin{tabular}{c c c}
\toprule
Metric & SSIM ($\uparrow$) & PSNR ($\uparrow$) \\
\midrule
Value & 0.995 & 45.6 \\
\bottomrule
\end{tabular}
\caption{\textbf{Satisfaction of Future Constraints.} We test how well the model lands on the ground truth future conditioning when following its learned score function.}\label{tab:constraint_satisfcation}
\end{table}

\newpage
\clearpage

\subsection{Irregular Time Sampling}

We want to assess how ABC compares against diffusion bridge baselines in sampling with irregular physical time intervals. We expect ABC to prevail, as the diffusion bridge baselines do not adapt their volatility to the time elapsed. 

We set $\sigma = 0.4$ for both ABC (no Brownian bridge) and the AR conditional diffusion bridge baseline. We tested causal generation starting from a single frame on 1360 videos from the Sky-Timelapse dataset, and use 250 SDE solver steps. Out of 64 ground-truth frames provided in the Sky-Timelapse dataset samples, we had the models generate the irregularly-spaced frames (corresponding to the repeating spacing pattern of 1-1-2-4-8): $1, 2, 4, 8, 16, 17, 18, 20, 24, 32, 33, 34, 36, 40, 48, 49, 50, 52, 56$. ABC is clearly the best, as seen in Table~\ref{tab:irregular_time_sampling}.

\subsubsection{Adaptive Inference-Time Volatility Scaling in Diffusion Bridge Baseline}

Theorem~\ref{thm:non_bijective_path_measure} also suggests that we could select the volatility of an autoregressive conditional diffusion bridge baseline to match ABC's, via the scaling 
\begin{equation}\label{eqn:volatility_scale_inference_time}
    \tilde{\sigma}_{\text{new}}(\tau) = \sqrt{t_{i^*+1} - t_{i^*}} * \sigma_\text{ABC}(t_{i^*} + (t_{i^*+1} - t_{i^*})\tau)
\end{equation}
If we adopted this construction at \textit{train-time} and inference-time, then it would indeed be equivalent to ABC. However, adopting it \textit{only at inference time would give incorrect results}. This is because the training target scores (Lemma~\ref{lem:gaussian_kernels}) are carefully constructed to fulfill the prescribed volatility schedule. As such, changing the volatility schedule at inference time would lead to evaluation on a completely different implied change-of-measure than it was trained for. Intuitively, if the score network was trained on interpolating paths with very low volatility, it should not be expected to give the correct score on interpolating paths with very high volatility that it had never seen. 

To empirically illustrate the folly of inference-time volatility schedule changes, we add an additional baseline to this experiment, where the volatility is scaled \textit{at inference-time only} according to Equation~\ref{eqn:volatility_scale_inference_time}. See Table~\ref{tab:irregular_time_sampling}.

\begin{table}[h]
    \centering
    \begin{tabular}{l r r r}
    \toprule
        Metric & ABC & AR Cond: $\sigma$ & AR Cond: $\sqrt{\Delta t} * \sigma$\\
        \midrule
        FVD ($\downarrow$) & \textbf{301.3} & 465.0 & 658.1 \\
        FID ($\downarrow$) & \textbf{12.2} & 75.8 & 46.1 \\
    \bottomrule
    \end{tabular}
    \caption{\textbf{Irregular Time Sampling.} "AR Cond: $\sigma$" refers to the autoregressive conditional diffusion bridge baseline from Theorem~\ref{thm:non_bijective_path_measure} \textit{without} time-adaptive volatility. "AR Cond: $\sqrt{\Delta t} * \sigma$" refers to a version of that baseline which, \textit{at inference time only}, has its volatility coefficient scaled as $\sqrt{t_{i^*+1} - t_{i^*}}$ according to the gap in physical time over an interval, such that its volatility coefficient matches ABC's.}\label{tab:irregular_time_sampling}
\end{table}

\newpage
\clearpage

\subsection{Diversity Analysis}

We want to assess how diverse ABC's completions for a given video prompt are.

Over 1360 Sky-Timelapse videos, we gave ABC 4 frames from each video, then had it complete the next 12 frames. To measure diversity, we had the model generate 4 different completions per video, and calculated the frame-wise LPIPS~\cite{zhang2018unreasonable} between each of the 6 pairs of different video completions. LPIPS of 0 indicates deterministic completions.

Quantitatively, the completions are diverse: as time went further from the conditioning frames (0-3), diversity increased, as shown in Table~\ref{tab:lpips}. Qualitatively, we confirmed by manual inspection that the model produces diverse plausible completions (\textit{e.g.}, cloud paths vary).

\begin{table}[h]
    \centering
    \begin{tabular}{l r r r}
        \toprule
        Frame Index & 4 & 8 & 12 \\
        \midrule
        LPIPS & $0.056 \pm 0.026$ & $0.096 \pm 0.041$ & $0.132 \pm 0.064$ \\
        \bottomrule
    \end{tabular}
    \caption{\textbf{LPIPS Scores by Frame:} Out of 16 total frames, frames 0-3 are the prompts. Reporting mean and std LPIPS.}\label{tab:lpips}
\end{table}

\newpage
\clearpage

\subsection{CelebV-HQ: Additional Non-Causal Results}\label{subsec:celebvhq_additional_non_causal}

See Table \ref{tab:celebvhq_NON_causal_fid} for FID metrics and Table \ref{tab:celebvhq_non_causal_overall_avg_rank} for overall rankings. 
Table \ref{tab:celebvhq_non_causal_overall_avg_rank} is calculated by first calculating the average rank in each of Tables \ref{tab:celebvhq_NON_causal_fid} and \ref{tab:celebvhq_NON_causal_fvd} (that is, getting a summary rank for FID and FVD), then averaging (and getting the standard deviation of) the two summary metric ranks. The rank is 1-based, where lower is better.

\begin{table}[htbp]
\centering\tiny
\caption{CelebV-HQ --- Non-Causal --- FID ($\downarrow$)}
\label{tab:celebvhq_NON_causal_fid}
\resizebox{\textwidth}{!}{\begin{tabular}{lrrrrrrrrrrrr}
\toprule
 & \multicolumn{6}{c}{Num Frames = 32} & \multicolumn{6}{c}{Num Frames = 16} \\
\cmidrule(lr){2-7} \cmidrule(lr){8-13}
 & \multicolumn{2}{c}{pin\_every = 4} & \multicolumn{2}{c}{pin\_every = 8} & \multicolumn{2}{c}{pin\_every = 16} & \multicolumn{2}{c}{pin\_every = 4} & \multicolumn{2}{c}{pin\_every = 8} & \multicolumn{2}{c}{pin\_every = 16} \\
\cmidrule(lr){2-3} \cmidrule(lr){4-5} \cmidrule(lr){6-7} \cmidrule(lr){8-9} \cmidrule(lr){10-11} \cmidrule(lr){12-13}
\textbf{Method} \ \ \ \ \ \ Steps & \multicolumn{1}{c}{250} & \multicolumn{1}{c}{500} & \multicolumn{1}{c}{250} & \multicolumn{1}{c}{500} & \multicolumn{1}{c}{250} & \multicolumn{1}{c}{500} & \multicolumn{1}{c}{250} & \multicolumn{1}{c}{500} & \multicolumn{1}{c}{250} &\multicolumn{1}{c}{500} & \multicolumn{1}{c}{250} & \multicolumn{1}{c}{500} \\
\midrule
\rowcolor{gray!20} \multicolumn{13}{l}{\textit{ABC}} \\
\quad W/ BB: $\sigma$=0.5 & \textbf{\textit{1.3}} $\pm$ 0.0 & \textbf{\textit{1.0}} $\pm$0.0 & \textbf{\textit{3.4}} $\pm$ 0.0 & \textbf{\textit{2.4}} $\pm$ 0.0 & 12.5 $\pm$ 0.0 & 8.3 $\pm$ 0.0 & \textbf{\textit{2.7}} $\pm$ 0.0 & \textbf{\textit{2.2}} $\pm$ 0.0 & 6.7 $\pm$ 0.0 & 4.8 $\pm$ 0.0 & 15.9 $\pm$ 0.0 & 11.3 $\pm$ 0.1 \\
\quad No BB: $\sigma$=0.5 & \textbf{1.3} $\pm$ 0.0 & \textbf{0.9} $\pm$ 0.0 & \textbf{2.9} $\pm$ 0.0 & \textbf{2.2} $\pm$ 0.0 & \textbf{\textit{8.8}} $\pm$ 0.1 & 5.8 $\pm$ 0.0 & \textbf{2.6} $\pm$ 0.0 & \textbf{2.1} $\pm$ 0.0 & 5.5 $\pm$ 0.0 & 4.1 $\pm$ 0.1 & 9.5 $\pm$0.0 & 6.8 $\pm$ 0.1 \\
\rowcolor{gray!20} \multicolumn{13}{l}{\textit{Conditional Diffusion Bridge}} \\
\quad $\sigma=0.5$ & 24.2 $\pm$ 0.0 & 10.0 $\pm$ 0.0 & 55.0 $\pm$ 0.0 & 24.2 $\pm$ 0.1 & 148.8 $\pm$ 0.2 & 70.9 $\pm$ 0.2 & 12.5 $\pm$ 0.0 & 5.5 $\pm$ 0.0 & 28.0 $\pm$ 0.0 & 13.7 $\pm$ 0.1 & 68.3 $\pm$ 0.1 & 33.6 $\pm$ 0.1 \\
\quad $\sigma=0.125$ & 1.7 $\pm$ 0.0 & 1.2 $\pm$ 0.0 & 4.2 $\pm$ 0.0 & 2.6 $\pm$ 0.0 & 9.4 $\pm$ 0.1 & \textbf{\textit{5.3}} $\pm$ 0.1 & 2.9 $\pm$ 0.0 & 2.3 $\pm$ 0.0 & \textbf{\textit{5.2}} $\pm$ 0.1 & \textbf{\textit{3.5}} $\pm$ 0.0 & \textbf{\textit{6.9}} $\pm$ 0.1 & \textbf{\textit{5.0}} $\pm$ 0.0 \\
\quad $\sigma=0.09$ & 1.5 $\pm$ 0.0 & 1.1 $\pm$ 0.0 & 3.5 $\pm$ 0.0 & 2.6 $\pm$ 0.0 & \textbf{7.4} $\pm$ 0.0 & \textbf{4.1} $\pm$ 0.0 & 3.0 $\pm$ 0.0 & 2.4 $\pm$ 0.0 & \textbf{4.9} $\pm$ 0.0 & \textbf{3.4} $\pm$ 0.0 & \textbf{4.9} $\pm$ 0.0 & \textbf{4.0} $\pm$ 0.0 \\
\rowcolor{gray!20} \multicolumn{13}{l}{\textit{Noise-to-Data Diffusion}} \\
\quad Cos: (3.0, 0.04) & 110.9 $\pm$ 0.1 & 18.1 $\pm$ 0.1 & 185.8 $\pm$ 0.0 & 65.1 $\pm$ 0.2 & 240.1 $\pm$ 0.2 & 135.8 $\pm$ 0.4 & 25.6 $\pm$ 0.1 & 3.7 $\pm$ 0.0 & 82.1 $\pm$ 0.2 & 16.7 $\pm$ 0.0 & 132.7 $\pm$ 0.3 & 42.8 $\pm$ 0.3 \\
\quad Cos: (5.0, 0.05) & 154.3 $\pm$ 0.1 & 42.2 $\pm$ 0.1 & 236.4 $\pm$ 0.3 & 114.2 $\pm$0.2 & 285.9 $\pm$ 0.1 & 196.5 $\pm$ 0.2 & 51.2 $\pm$ 0.0 & 6.1 $\pm$ 0.0 & 126.8 $\pm$ 0.1 & 29.4 $\pm$ 0.1 & 190.6 $\pm$ 0.1 & 75.7 $\pm$ 0.4 \\
\quad Exp: (4.0, 2.5) & 125.5 $\pm$ 0.1 & 33.1 $\pm$ 0.0 & 188.8 $\pm$ 0.1 & 67.9 $\pm$ 0.1 & 238.4 $\pm$ 0.2 & 117.2 $\pm$ 0.1 & 32.6 $\pm$ 0.1 & 3.0 $\pm$ 0.0 & 74.3 $\pm$ 0.1 &9.4 $\pm$ 0.0 & 101.3 $\pm$ 0.2 & 15.9 $\pm$ 0.0 \\
\quad Exp: (5.0, 5.0) & 190.3 $\pm$ 0.0 & 116.1 $\pm$ 0.1 & 253.1 $\pm$ 0.1 & 190.0 $\pm$0.0 & 301.8 $\pm$ 0.0 & 226.1 $\pm$ 0.1 & 116.3 $\pm$ 0.1 & 9.1 $\pm$ 0.0 & 183.7 $\pm$ 0.2 & 33.1 $\pm$ 0.1 & 215.6 $\pm$ 0.1 & 63.8 $\pm$ 0.4 \\
\bottomrule
\end{tabular}
}
\end{table}

\begin{table}[htbp]
\centering
\caption{CelebV-HQ (Non-Causal) --- Overall Average Rank ($\downarrow$)}
\label{tab:celebvhq_non_causal_overall_avg_rank}
\begin{tabular}{lcc}
\toprule
Method & Avg Rank & Rank \\
\midrule
\rowcolor{gray!20} \multicolumn{3}{l}{\textit{ABC}} \\
\quad W/ BB: $\sigma$=0.5 & \textbf{\textit{2.6}} $\pm$ 0.4 & 2 \\
\quad No BB: $\sigma$=0.5 & \textbf{1.9} $\pm$ 0.0 & 1 \\
\rowcolor{gray!20} \multicolumn{3}{l}{\textit{Conditional Diffusion Bridge}} \\
\quad $\sigma=0.5$ & 5.2 $\pm$ 0.1 & 5 \\
\quad $\sigma=0.125$ & 2.9 $\pm$ 0.0 & 4 \\
\quad $\sigma=0.09$ & 2.8 $\pm$ 0.6 & 3 \\
\rowcolor{gray!20} \multicolumn{3}{l}{\textit{Noise-to-Data Diffusion}} \\
\quad Cos: (3.0, 0.04) & 6.6 $\pm$ 0.1 & 7 \\
\quad Cos: (5.0, 0.05) & 8.0 $\pm$ 0.1 & 8 \\
\quad Exp: (4.0, 2.5) & 6.0 $\pm$ 0.1 & 6 \\
\quad Exp: (5.0, 5.0) & 8.9 $\pm$ 0.0 & 9 \\
\bottomrule
\end{tabular}
\end{table}

\newpage
\clearpage

\subsection{Sky-Timelapse: Additional Non-Causal Results}

See Table \ref{tab:sky_NON_causal_fid} for FID. See Table \ref{tab:sky_non_causal_overall_avg_rank} for overall ranking across FVD and FID in the non-causal setting for Sky-Timelapse. Table \ref{tab:sky_non_causal_overall_avg_rank} was calculated analogously to Table \ref{tab:celebvhq_non_causal_overall_avg_rank}; summarizing the winner across both FID and FVD.

\begin{table}[htbp]
\centering\tiny
\caption{Sky-timelapse --- Non-Causal --- FID ($\downarrow$)}
\label{tab:sky_NON_causal_fid}
\resizebox{\textwidth}{!}{\begin{tabular}{lrrrrrrrrrrrr}
\toprule
 & \multicolumn{6}{c}{Num Frames = 32} & \multicolumn{6}{c}{Num Frames = 16} \\
\cmidrule(lr){2-7} \cmidrule(lr){8-13}
 & \multicolumn{2}{c}{pin\_every = 4} & \multicolumn{2}{c}{pin\_every =8} & \multicolumn{2}{c}{pin\_every = 16} & \multicolumn{2}{c}{pin\_every = 4} & \multicolumn{2}{c}{pin\_every = 8} & \multicolumn{2}{c}{pin\_every = 16} \\
\cmidrule(lr){2-3} \cmidrule(lr){4-5} \cmidrule(lr){6-7} \cmidrule(lr){8-9} \cmidrule(lr){10-11} \cmidrule(lr){12-13}
\textbf{Method} \ \ \ \ \ \ Steps & \multicolumn{1}{c}{250} & \multicolumn{1}{c}{500} & \multicolumn{1}{c}{250} & \multicolumn{1}{c}{500} & \multicolumn{1}{c}{250} & \multicolumn{1}{c}{500} & \multicolumn{1}{c}{250}& \multicolumn{1}{c}{500} & \multicolumn{1}{c}{250} & \multicolumn{1}{c}{500} & \multicolumn{1}{c}{250} & \multicolumn{1}{c}{500} \\
\midrule
\rowcolor{gray!20} \multicolumn{13}{l}{\textit{ABC}} \\
\quad No BB: $\sigma$=0.4 & 1.8 $\pm$ 0.0 & 1.1 $\pm$ 0.0 & 4.3 $\pm$ 0.0 & 2.9 $\pm$ 0.0 & 8.8 $\pm$ 0.0 & 6.6 $\pm$ 0.0 & \textbf{\textit{2.1}} $\pm$ 0.0 & 1.7 $\pm$ 0.0 & 4.6 $\pm$ 0.0 & 3.9 $\pm$ 0.0 & 8.8 $\pm$0.1 & 7.4 $\pm$ 0.0 \\
\quad W/ BB: $\sigma$=0.4 & \textbf{\textit{1.7}} $\pm$ 0.0 & \textbf{\textit{1.1}} $\pm$ 0.0 & \textbf{\textit{3.9}} $\pm$ 0.0 & \textbf{\textit{2.6}} $\pm$ 0.0 & \textbf{\textit{7.4}} $\pm$ 0.0 & \textbf{\textit{5.4}} $\pm$ 0.0 & \textbf{2.0} $\pm$ 0.0 & \textbf{1.7} $\pm$ 0.0 & \textbf{4.3} $\pm$ 0.0 & \textbf{3.6} $\pm$ 0.0 & \textbf{\textit{8.2}} $\pm$0.1 & 6.9 $\pm$ 0.0 \\
\rowcolor{gray!20} \multicolumn{13}{l}{\textit{Conditional Diffusion Bridge}} \\
\quad $\sigma=0.071$ & 2.2 $\pm$ 0.0 & 1.4 $\pm$ 0.0 & 5.2 $\pm$ 0.0 & 3.2 $\pm$ 0.0 & 9.0 $\pm$ 0.0 & 6.2 $\pm$ 0.0 & 3.1 $\pm$ 0.0 & 2.1 $\pm$ 0.0 & 5.6 $\pm$ 0.0 & 4.1 $\pm$ 0.0 & 8.9 $\pm$ 0.0 & \textbf{\textit{6.6}} $\pm$ 0.0 \\
\quad $\sigma=0.1$ & \textbf{1.7} $\pm$ 0.0 & \textbf{1.1} $\pm$ 0.0 & \textbf{3.8} $\pm$ 0.0 & \textbf{2.5} $\pm$ 0.0 & \textbf{7.1} $\pm$ 0.0& \textbf{5.2} $\pm$ 0.0 & 2.2 $\pm$ 0.0 & \textbf{\textit{1.7}} $\pm$ 0.0 & \textbf{\textit{4.5}} $\pm$ 0.0 & \textbf{\textit{3.7}} $\pm$ 0.0 & \textbf{8.0} $\pm$ 0.0 & \textbf{6.5} $\pm$ 0.0 \\
\quad $\sigma=0.3$ & 4.1 $\pm$ 0.0 & 1.8 $\pm$ 0.0 & 9.5 $\pm$ 0.0 & 4.5 $\pm$ 0.0 & 21.3 $\pm$ 0.2 & 10.3 $\pm$ 0.0 & 2.5 $\pm$ 0.0 & 1.7 $\pm$ 0.0 & 5.8 $\pm$ 0.0 & 4.0 $\pm$ 0.0 & 10.5 $\pm$ 0.1 & 7.1 $\pm$ 0.1 \\
\quad $\sigma=0.4$ & 8.5 $\pm$ 0.0 & 3.3 $\pm$ 0.0 & 19.7 $\pm$ 0.0 & 7.9 $\pm$ 0.0 & 46.4 $\pm$ 0.2 & 18.3 $\pm$ 0.1 & 4.1 $\pm$ 0.0 & 2.3 $\pm$ 0.0 & 9.5 $\pm$ 0.1 & 5.4 $\pm$ 0.0 & 17.5 $\pm$ 0.0 & 9.7 $\pm$ 0.0 \\
\quad $\sigma=0.6$ & 28.0 $\pm$ 0.1 & 9.4 $\pm$ 0.0 & 63.8 $\pm$ 0.1 & 23.1 $\pm$ 0.0 & 140.2 $\pm$ 0.4 & 59.3 $\pm$ 0.3 & 9.8 $\pm$ 0.0 & 4.2 $\pm$ 0.0 & 24.6 $\pm$ 0.1 & 10.4 $\pm$ 0.1 & 50.6 $\pm$ 0.2 & 21.3 $\pm$ 0.2 \\
\rowcolor{gray!20} \multicolumn{13}{l}{\textit{Noise-to-Data Diffusion}} \\
\quad Cos: (3.0, 0.04) & 57.9 $\pm$ 0.1 & 6.0 $\pm$ 0.0 & 108.0 $\pm$ 0.2 & 15.2 $\pm$ 0.0 & 153.4 $\pm$ 0.2 & 27.1 $\pm$ 0.3 & 7.3 $\pm$ 0.1 &2.5 $\pm$ 0.0 & 17.0 $\pm$ 0.1 & 6.2 $\pm$ 0.0 & 26.4 $\pm$ 0.1 & 10.5 $\pm$ 0.1 \\
\quad Cos: (5.0, 0.05) & 133.9 $\pm$ 0.2 & 14.0 $\pm$ 0.0 & 192.7 $\pm$0.0 & 33.4 $\pm$ 0.2 & 227.3 $\pm$ 0.1 & 59.2 $\pm$ 0.3 & 15.3 $\pm$ 0.1 & 3.4 $\pm$ 0.0 & 36.6 $\pm$ 0.1 & 9.1 $\pm$ 0.1 & 58.4 $\pm$ 0.2 & 15.8 $\pm$ 0.2 \\
\quad Exp: (4.0, 2.5) & 85.1 $\pm$ 0.1 & 9.9 $\pm$ 0.0 & 113.8 $\pm$ 0.2 & 24.9 $\pm$ 0.1 & 160.3 $\pm$ 0.3 & 85.5 $\pm$ 0.1 & 9.5 $\pm$ 0.0 & 2.3 $\pm$ 0.0 & 24.1 $\pm$ 0.1 & 7.2 $\pm$ 0.0 & 69.2 $\pm$ 0.2 & 15.5 $\pm$ 0.2 \\
\quad Exp: (5.0, 5.0) & 153.9 $\pm$ 0.1 & 25.9 $\pm$ 0.0 & 163.3 $\pm$ 0.1 & 49.8 $\pm$ 0.1 & 189.7 $\pm$ 0.3 & 92.9 $\pm$ 0.1 & 23.2 $\pm$ 0.0& 4.2 $\pm$ 0.0 & 44.9 $\pm$ 0.1 & 12.5 $\pm$ 0.0 & 80.1 $\pm$ 0.4 & 22.8 $\pm$ 0.2 \\
\bottomrule
\end{tabular}}
\end{table}

\begin{table}[htbp]
\centering
\caption{Sky-timelapse (Non-Causal) --- Overall Average Rank ($\downarrow$)}
\label{tab:sky_non_causal_overall_avg_rank}
\begin{tabular}{lcc}
\toprule
Method & Avg Rank & Rank \\
\midrule
\rowcolor{gray!20} \multicolumn{3}{l}{\textit{ABC}} \\
\quad No BB: $\sigma$=0.4 & \textbf{\textit{2.6}} $\pm$ 0.6 & 2 \\
\quad W/ BB: $\sigma$=0.4 & \textbf{2.2} $\pm$ 0.5 & 1 \\
\rowcolor{gray!20} \multicolumn{3}{l}{\textit{Conditional Diffusion Bridge}} \\
\quad $\sigma=0.071$ & 5.5 $\pm$ 1.5 & 6 \\
\quad $\sigma=0.1$ & 3.3 $\pm$ 1.9 & 4 \\
\quad $\sigma=0.3$ & 3.3 $\pm$ 1.3 & 3 \\
\quad $\sigma=0.4$ & 5.0 $\pm$ 1.0 & 5 \\
\quad $\sigma=0.6$ & 7.5 $\pm$ 1.0 & 7 \\
\rowcolor{gray!20} \multicolumn{3}{l}{\textit{Noise-to-Data Diffusion}}\\
\quad Cos: (3.0, 0.04) & 7.7 $\pm$ 0.3 & 8 \\
\quad Cos: (5.0, 0.05) & 10.0 $\pm$ 0.3 & 10 \\
\quad Exp: (4.0, 2.5) & 8.3 $\pm$ 0.4 & 9 \\
\quad Exp: (5.0, 5.0) & 10.6 $\pm$ 0.2 & 11 \\
\bottomrule
\end{tabular}
\end{table}


\newpage
\clearpage

\subsection{Sky-Timelapse: Causal Results}\label{sec:causal_sky_timelapse}

We use the same frame counts and discretization step counts as in the other experiments. Here, we do causal (regular autoregressive) generation, where we don't prompt with the future, but rather, only give a prefix of 1 or 4 frames.

See Tables \ref{tab:sky_causal_rollout_fvd} (FVD), \ref{tab:sky_causal_rollout_fid} (FID), Figure \ref{tab:pareto_sky_timelapse_causal} (combined ranking). We report mean and std over three random inference seeds. 
Under this setting, \textit{ABC without Brownian Bridge is Pareto-optimal} (Figure \ref{tab:pareto_sky_timelapse_causal}) across FVD and FID scores on causal rollout. As compared to the conditional diffusion bridge methods on the Pareto front, we see that ABC has low discrepancy in its overall FID (3) and FVD (2) rankings, indicating consistency and robustness across multiple aspects of quality measure.

\begin{table}[htbp]
\centering\tiny
\caption{Sky-timelapse --- Causal --- FVD ($\downarrow$)}
\label{tab:sky_causal_rollout_fvd}
\begin{tabular}{lrrrrrrrr}
\toprule
 & \multicolumn{4}{c}{Num Frames = 32} & \multicolumn{4}{c}{Num Frames = 16} \\
\cmidrule(lr){2-5} \cmidrule(lr){6-9}
 & \multicolumn{2}{c}{prefix\_length = 1} & \multicolumn{2}{c}{prefix\_length = 4} & \multicolumn{2}{c}{prefix\_length = 1} & \multicolumn{2}{c}{prefix\_length = 4} \\
\cmidrule(lr){2-3} \cmidrule(lr){4-5} \cmidrule(lr){6-7} \cmidrule(lr){8-9}
\textbf{Method} \ \ \ \ \ \ Steps & \multicolumn{1}{c}{250} & \multicolumn{1}{c}{500} & \multicolumn{1}{c}{250} & \multicolumn{1}{c}{500} & \multicolumn{1}{c}{250} & \multicolumn{1}{c}{500} & \multicolumn{1}{c}{250}& \multicolumn{1}{c}{500} \\
\midrule
\rowcolor{gray!20} \multicolumn{9}{l}{\textit{ABC}} \\
\quad No BB: $\sigma$=0.4 & \textbf{\textit{373.1}} $\pm$ 2.1 & \textbf{\textit{283.3}} $\pm$ 0.9 & \textbf{\textit{359.5}} $\pm$ 0.6 & \textbf{\textit{277.3}} $\pm$ 0.8 & \textbf{\textit{246.4}} $\pm$ 1.5 & \textbf{\textit{209.0}} $\pm$ 0.5 & \textbf{\textit{189.4}} $\pm$ 0.8 & 164.4 $\pm$ 0.3 \\
\quad W/ BB: $\sigma$=0.4 & 462.8 $\pm$ 1.4 & 364.0 $\pm$ 3.0 & 416.3 $\pm$ 1.2 & 330.5 $\pm$ 2.7 & 307.5 $\pm$ 1.5 & 241.0 $\pm$ 0.6 & 225.8 $\pm$ 1.6 & 181.5 $\pm$ 0.5 \\
\rowcolor{gray!20} \multicolumn{9}{l}{\textit{Conditional Diffusion Bridge}} \\
\quad $\sigma=0.071$ & 485.9 $\pm$ 1.1 & 428.4 $\pm$ 1.9 & 438.4 $\pm$ 1.1 & 359.0 $\pm$ 0.9 & 399.0 $\pm$ 1.0 & 354.3 $\pm$ 1.8 & 308.8 $\pm$ 0.2 & 281.7 $\pm$ 1.0 \\
\quad $\sigma=0.1$ & 409.0 $\pm$ 2.2 & 329.3 $\pm$ 0.3 & 389.0 $\pm$ 0.3 & 314.3 $\pm$ 0.5 & 326.4 $\pm$ 1.4 & 274.7 $\pm$ 1.4 & 251.5 $\pm$ 1.7 & 223.0 $\pm$ 0.5 \\
\quad $\sigma=0.3$ & \textbf{303.4} $\pm$ 1.2 & \textbf{225.8} $\pm$ 1.8 & \textbf{278.1} $\pm$ 1.1 & \textbf{194.6} $\pm$ 1.1 & \textbf{206.2} $\pm$ 1.7 & \textbf{185.5} $\pm$ 1.4 & \textbf{156.5} $\pm$ 1.9 & \textbf{137.4} $\pm$ 1.0 \\
\quad $\sigma=0.4$ & 442.0 $\pm$ 1.3 & 311.8 $\pm$ 0.6 & 435.4 $\pm$ 2.5 & 293.6 $\pm$ 2.8 & 269.8 $\pm$ 0.4 & 215.8 $\pm$ 1.8 & 220.0 $\pm$ 1.5 & \textbf{\textit{159.9}} $\pm$ 2.4 \\
\quad $\sigma=0.6$ & 746.4 $\pm$ 3.1 & 559.9 $\pm$ 5.7 & 753.3 $\pm$ 2.6 & 567.8 $\pm$ 3.7 & 456.6 $\pm$ 1.5 & 302.9 $\pm$ 2.2 & 438.3 $\pm$ 4.1 & 279.1 $\pm$ 2.1 \\
\rowcolor{gray!20} \multicolumn{9}{l}{\textit{Noise-to-Data Diffusion}}\\
\quad Cos: (3.0, 0.04) & 698.6 $\pm$ 1.4 & 512.7 $\pm$ 2.9 & 750.4 $\pm$ 4.2 & 508.8 $\pm$ 5.0 & 504.1 $\pm$ 3.7 & 258.3 $\pm$ 2.8 & 475.5 $\pm$ 1.6 & 210.4 $\pm$ 2.7 \\
\quad Cos: (5.0, 0.05) & 862.3 $\pm$ 1.0 & 782.4 $\pm$ 2.1 & 929.5 $\pm$ 2.1 & 842.2 $\pm$ 2.2 & 761.9 $\pm$ 3.3 & 363.0 $\pm$ 2.3 & 752.9 $\pm$ 4.0 & 324.7 $\pm$ 2.2 \\
\quad Exp: (4.0, 2.5) & 764.8 $\pm$ 1.2 & 658.0 $\pm$ 1.1 & 618.5 $\pm$1.0 & 535.2 $\pm$ 2.2 & 621.0 $\pm$ 1.7 & 317.3 $\pm$ 2.7 & 419.8 $\pm$2.0 & 210.8 $\pm$ 0.5 \\
\quad Exp: (5.0, 5.0) & 798.0 $\pm$ 1.7 & 680.1 $\pm$ 1.0 & 644.0 $\pm$0.4 & 580.1 $\pm$ 0.9 & 693.5 $\pm$ 0.9 & 411.1 $\pm$ 3.8 & 471.5 $\pm$1.2 & 300.5 $\pm$ 2.6 \\
\bottomrule
\end{tabular}
\end{table}

\begin{table}[htbp]
\centering\tiny
\caption{Sky-timelapse --- Causal --- FID ($\downarrow$)}
\label{tab:sky_causal_rollout_fid}
\begin{tabular}{lrrrrrrrr}
\toprule
 & \multicolumn{4}{c}{Num Frames = 32} & \multicolumn{4}{c}{Num Frames = 16} \\
\cmidrule(lr){2-5} \cmidrule(lr){6-9}
 & \multicolumn{2}{c}{prefix\_length = 1} & \multicolumn{2}{c}{prefix\_length = 4} & \multicolumn{2}{c}{prefix\_length = 1} & \multicolumn{2}{c}{prefix\_length = 4} \\
\cmidrule(lr){2-3} \cmidrule(lr){4-5} \cmidrule(lr){6-7} \cmidrule(lr){8-9}
\textbf{Method} \ \ \ \ \ \ Steps & \multicolumn{1}{c}{250} & \multicolumn{1}{c}{500} & \multicolumn{1}{c}{250} & \multicolumn{1}{c}{500} & \multicolumn{1}{c}{250} & \multicolumn{1}{c}{500} & \multicolumn{1}{c}{250}& \multicolumn{1}{c}{500} \\
\midrule
\rowcolor{gray!20} \multicolumn{9}{l}{\textit{ABC}} \\
\quad No BB: $\sigma$=0.4 & 19.9 $\pm$ 0.1 & 14.1 $\pm$ 0.0 & 17.3 $\pm$ 0.1 & 12.2 $\pm$ 0.1 & 10.2 $\pm$ 0.1 & 8.5 $\pm$ 0.1 & \textbf{6.6} $\pm$ 0.0 & 5.5 $\pm$ 0.0 \\
\quad W/ BB: $\sigma$=0.4 & 25.1 $\pm$ 0.1 & 17.4 $\pm$ 0.2 & 21.5 $\pm$ 0.1 & 14.7 $\pm$ 0.2 & 12.5 $\pm$ 0.0 & 10.5 $\pm$ 0.1 & 8.0 $\pm$ 0.1& 6.5 $\pm$ 0.1 \\
\rowcolor{gray!20} \multicolumn{9}{l}{\textit{Conditional Diffusion Bridge}} \\
\quad $\sigma=0.071$ & \textbf{13.8} $\pm$ 0.1 & \textbf{10.0} $\pm$ 0.1 & \textbf{15.3} $\pm$ 0.1 & \textbf{10.8} $\pm$ 0.0 & \textbf{8.6} $\pm$ 0.1 & \textbf{6.6} $\pm$ 0.0 & 8.1 $\pm$ 0.0 & 5.9 $\pm$ 0.0 \\
\quad $\sigma=0.1$ & \textbf{\textit{17.8}} $\pm$ 0.1 & \textbf{\textit{13.1}} $\pm$ 0.1 & \textbf{\textit{15.6}} $\pm$ 0.1 & \textbf{\textit{11.2}} $\pm$ 0.1 & \textbf{\textit{9.3}} $\pm$ 0.0 & \textbf{\textit{7.4}} $\pm$ 0.1 & 6.8 $\pm$ 0.1 & \textbf{\textit{5.2}} $\pm$ 0.0 \\
\quad $\sigma=0.3$ & 43.5 $\pm$ 0.2 & 19.9 $\pm$ 0.1 & 30.0 $\pm$ 0.1 &14.3 $\pm$ 0.1 & 11.8 $\pm$ 0.0 & 8.0 $\pm$ 0.0 & \textbf{\textit{6.7}}$\pm$ 0.0 & \textbf{4.7} $\pm$ 0.0 \\
\quad $\sigma=0.4$ & 95.7 $\pm$ 0.5 & 38.2 $\pm$ 0.2 & 68.4 $\pm$ 0.2 &27.4 $\pm$ 0.1 & 20.0 $\pm$ 0.1 & 11.0 $\pm$ 0.1 & 11.3 $\pm$ 0.0 & 6.3$\pm$ 0.1 \\
\quad $\sigma=0.6$ & 220.9 $\pm$ 0.6 & 122.1 $\pm$ 0.5 & 171.2 $\pm$ 0.9 & 95.2 $\pm$ 0.3 & 59.7 $\pm$ 0.1 & 25.1 $\pm$ 0.2 & 35.3 $\pm$ 0.2 & 15.4 $\pm$ 0.1 \\
\rowcolor{gray!20} \multicolumn{9}{l}{\textit{Noise-to-Data Diffusion}}\\
\quad Cos: (3.0, 0.04) & 175.0 $\pm$ 0.4 & 39.1 $\pm$ 0.1 & 140.4 $\pm$0.1 & 28.1 $\pm$ 0.2 & 29.0 $\pm$ 0.2 & 12.1 $\pm$ 0.1 & 16.3 $\pm$ 0.2& 6.6 $\pm$ 0.0 \\
\quad Cos: (5.0, 0.05) & 246.0 $\pm$ 0.2 & 88.0 $\pm$ 0.2 & 203.3 $\pm$0.2 & 68.3 $\pm$ 0.1 & 62.7 $\pm$ 0.3 & 17.3 $\pm$ 0.2 & 37.8 $\pm$ 0.2& 9.8 $\pm$ 0.1 \\
\quad Exp: (4.0, 2.5) & 191.8 $\pm$ 0.3 & 141.2 $\pm$ 0.4 & 159.6 $\pm$0.4 & 116.2 $\pm$ 0.2 & 83.7 $\pm$ 0.2 & 18.5 $\pm$ 0.2 & 50.1 $\pm$ 0.1 & 9.7 $\pm$ 0.0 \\
\quad Exp: (5.0, 5.0) & 211.6 $\pm$ 0.2 & 145.8 $\pm$ 0.1 & 181.5 $\pm$0.1 & 118.1 $\pm$ 0.1 & 90.4 $\pm$ 0.3 & 26.4 $\pm$ 0.2 & 55.4 $\pm$ 0.0 & 14.3 $\pm$ 0.1 \\
\bottomrule
\end{tabular}
\end{table}

\begin{figure}[h]
    \centering
    \includegraphics[width=0.7\linewidth]{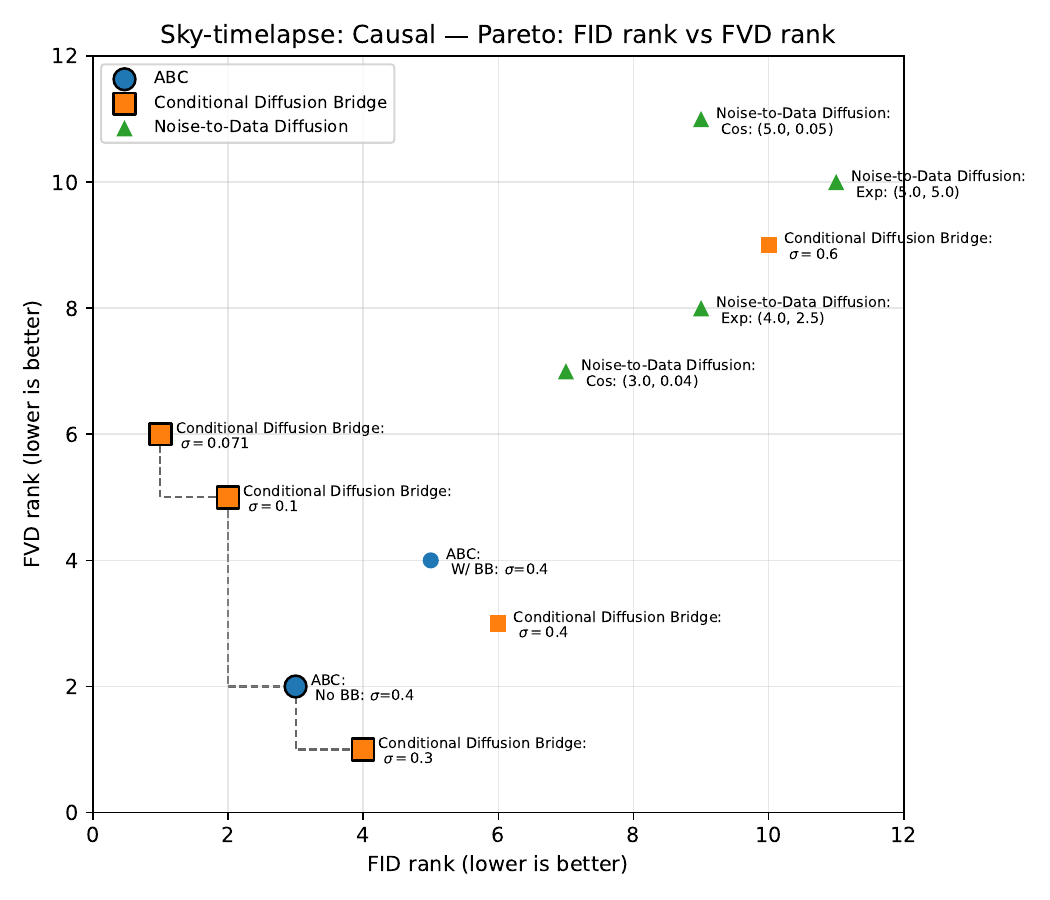}
    \caption{\textbf{Pareto Front for Sky-Timelapse Causal Generation:} ABC is Pareto-optimal across FVD and FID scores for causally-ordered generation on Sky-Timelapse. Corresponds to results in Tables \ref{tab:sky_causal_rollout_fid}, \ref{tab:sky_causal_rollout_fvd}. Horizontal axis is the ordering of average rank in Table \ref{tab:sky_causal_rollout_fid}'s columns, and vertical axis is the ordering of average rank in Table \ref{tab:sky_causal_rollout_fvd}'s columns.}\label{tab:pareto_sky_timelapse_causal}
\end{figure}


\newpage
\clearpage

\subsection{CelebV-HQ: Causal Results}\label{sec:celebv_hq_causal}

We retrained CelebV-HQ on a strictly causal conditioning structure (same frame sampling procedure as before, and then masking out any future conditioning frames that inadvertently got selected), with otherwise the same settings as in the previous experiments. We report mean and std over two random seeds for inference, with each trial calculated over 4096 videos. See Tables \ref{tab:celebvhq_causal_CAUSAL_EVAL_fvd}, \ref{tab:celebvhq_causal_CAUSAL_EVAL_fid}, \ref{tab:celebvhq_causal_overall_avg_rank}. Among the settings we tested, \textit{ABC is clearly better than its conditional diffusion bridge and noise-to-data cousins}.

\begin{table}[htbp]
\centering\tiny
\caption{CelebV-HQ --- Causal --- FVD ($\downarrow$)}
\label{tab:celebvhq_causal_CAUSAL_EVAL_fvd}
\resizebox{\textwidth}{!}{\begin{tabular}{lrrrrrrrr}
\toprule
 & \multicolumn{4}{c}{Num Frames = 32} & \multicolumn{4}{c}{Num Frames = 16} \\
\cmidrule(lr){2-5} \cmidrule(lr){6-9}
 & \multicolumn{2}{c}{prefix\_length = 1} & \multicolumn{2}{c}{prefix\_length = 4} & \multicolumn{2}{c}{prefix\_length = 1} & \multicolumn{2}{c}{prefix\_length = 4} \\
\cmidrule(lr){2-3} \cmidrule(lr){4-5} \cmidrule(lr){6-7} \cmidrule(lr){8-9}
\textbf{Method} \ \ \ \ \ \ Steps & \multicolumn{1}{c}{250} & \multicolumn{1}{c}{500} & \multicolumn{1}{c}{250} & \multicolumn{1}{c}{500} & \multicolumn{1}{c}{250} & \multicolumn{1}{c}{500} & \multicolumn{1}{c}{250} & \multicolumn{1}{c}{500} \\
\midrule
\rowcolor{gray!20} \multicolumn{9}{l}{\textit{ABC}} \\
\quad No BB: $\sigma$=0.5 & \textbf{301.7} $\pm$ 2.1 & \textbf{\textit{278.3}}$\pm$ 3.0 & \textbf{\textit{362.7}} $\pm$ 5.0 & \textbf{\textit{242.1}} $\pm$ 0.5 & \textbf{234.6} $\pm$ 0.5 & 227.2 $\pm$ 0.8 & \textbf{\textit{181.2}} $\pm$ 0.4 & 139.1 $\pm$ 0.4 \\
\quad No BB: $\sigma$=0.75 & \textbf{\textit{331.0}} $\pm$ 0.5 & \textbf{275.6} $\pm$ 2.1 & 410.7 $\pm$ 0.5 & 284.5 $\pm$ 0.1 & 246.5 $\pm$ 1.5 & \textbf{\textit{200.2}} $\pm$ 0.3 & \textbf{177.8} $\pm$ 0.1 & \textbf{133.3} $\pm$ 0.1 \\
\quad No BB: $\sigma$=1.0 & 343.7 $\pm$ 0.6 & 298.6 $\pm$ 2.1 & \textbf{295.7}$\pm$ 1.3 & \textbf{216.5} $\pm$ 0.3 & 281.8 $\pm$ 1.6 & 212.2 $\pm$ 0.1 & 194.9 $\pm$ 0.2 & 138.9 $\pm$ 0.2 \\
\rowcolor{gray!20} \multicolumn{9}{l}{\textit{Conditional Diffusion Bridge}} \\
\quad $\sigma=0.25$ & 897.2 $\pm$ 4.4 & 517.9 $\pm$ 3.3 & 803.2 $\pm$ 0.5 & 490.7 $\pm$ 0.1 & \textbf{\textit{239.2}} $\pm$ 0.8 & \textbf{196.4} $\pm$ 1.3 & 193.8 $\pm$ 0.7 & \textbf{\textit{137.4}} $\pm$ 0.7 \\
\quad $\sigma=0.5$ & 1014.1 $\pm$ 0.6 & 541.5 $\pm$ 2.1 & 889.6 $\pm$ 0.7 & 471.8 $\pm$ 0.6 & 430.9 $\pm$ 2.3 & 273.4 $\pm$ 0.0 & 365.3 $\pm$ 0.4 & 220.2 $\pm$ 0.9 \\
\quad $\sigma=0.75$ & 977.5 $\pm$ 1.0 & 726.1 $\pm$ 1.8 & 887.9 $\pm$ 2.9 & 653.4 $\pm$ 2.3 & 674.7 $\pm$ 0.3 & 426.7 $\pm$ 0.3 & 575.1 $\pm$ 0.2 & 368.0 $\pm$ 1.0 \\
\quad $\sigma=1.0$ & 938.3 $\pm$ 1.5 & 796.6 $\pm$ 0.5 & 764.7 $\pm$ 1.6 & 702.9 $\pm$ 1.1 & 726.5 $\pm$ 1.1 & 535.8 $\pm$ 1.3 & 629.6 $\pm$ 1.9 & 489.2 $\pm$ 2.4 \\
\rowcolor{gray!20} \multicolumn{9}{l}{\textit{Noise-to-Data Diffusion}} \\
\quad Cos: (3.0, 0.04) & 1030.9 $\pm$ 1.4 & 1055.9 $\pm$ 0.5 & 981.4 $\pm$ 1.3& 1059.6 $\pm$ 0.3 & 1059.6 $\pm$ 1.5 & 591.9 $\pm$ 0.8 & 990.0 $\pm$ 0.3 & 458.4 $\pm$ 0.9 \\
\bottomrule
\end{tabular}}
\end{table}

\begin{table}[htbp]
\centering\tiny
\caption{CelebV-HQ --- Causal --- FID ($\downarrow$)}
\label{tab:celebvhq_causal_CAUSAL_EVAL_fid}
\resizebox{\textwidth}{!}{\begin{tabular}{lrrrrrrrr}
\toprule
 & \multicolumn{4}{c}{Num Frames = 32} & \multicolumn{4}{c}{Num Frames = 16} \\
\cmidrule(lr){2-5} \cmidrule(lr){6-9}
 & \multicolumn{2}{c}{prefix\_length = 1} & \multicolumn{2}{c}{prefix\_length = 4} & \multicolumn{2}{c}{prefix\_length = 1} & \multicolumn{2}{c}{prefix\_length = 4} \\
\cmidrule(lr){2-3} \cmidrule(lr){4-5} \cmidrule(lr){6-7} \cmidrule(lr){8-9}
\textbf{Method} \ \ \ \ \ \ Steps & \multicolumn{1}{c}{250} & \multicolumn{1}{c}{500} & \multicolumn{1}{c}{250} & \multicolumn{1}{c}{500} & \multicolumn{1}{c}{250} & \multicolumn{1}{c}{500} & \multicolumn{1}{c}{250} & \multicolumn{1}{c}{500} \\
\midrule
\rowcolor{gray!20} \multicolumn{9}{l}{\textit{ABC}} \\
\quad No BB: $\sigma$=0.5 & \textbf{13.8} $\pm$ 0.1 & \textbf{5.8} $\pm$ 0.0 &\textbf{19.8} $\pm$ 0.1 & \textbf{11.5} $\pm$ 0.0 & \textbf{6.8} $\pm$ 0.0 & \textbf{4.1} $\pm$ 0.0 & \textbf{8.3} $\pm$ 0.0 & \textbf{5.4} $\pm$ 0.0 \\
\quad No BB: $\sigma$=0.75 & \textbf{\textit{27.3}} $\pm$ 0.0 & \textbf{\textit{12.0}} $\pm$ 0.0 & \textbf{\textit{33.5}} $\pm$ 0.0 & \textbf{\textit{17.5}} $\pm$ 0.0 & \textbf{\textit{13.2}} $\pm$ 0.0 & \textbf{\textit{7.2}} $\pm$ 0.0 & \textbf{\textit{9.6}} $\pm$ 0.0 & \textbf{\textit{5.8}} $\pm$ 0.0 \\
\quad No BB: $\sigma$=1.0 & 40.6 $\pm$ 0.1 & 20.8 $\pm$ 0.1 & 34.7 $\pm$ 0.2 &19.1 $\pm$ 0.0 & 21.8 $\pm$ 0.0 & 11.8 $\pm$ 0.0 & 13.0 $\pm$ 0.0 & 7.5 $\pm$ 0.1 \\
\rowcolor{gray!20} \multicolumn{9}{l}{\textit{Conditional Diffusion Bridge}} \\
\quad $\sigma=0.25$ & 117.1 $\pm$ 0.2 & 51.4 $\pm$ 0.1 & 91.4 $\pm$ 0.0 & 43.0$\pm$ 0.1 & 16.4 $\pm$ 0.0 & 8.9 $\pm$ 0.0 & 11.8 $\pm$ 0.0 & 6.7 $\pm$ 0.0 \\
\quad $\sigma=0.5$ & 182.2 $\pm$ 0.2 & 99.0 $\pm$ 0.1 & 132.5 $\pm$ 0.2 & 74.2$\pm$ 0.0 & 64.8 $\pm$ 0.1 & 33.1 $\pm$ 0.0 & 37.5 $\pm$ 0.1 & 19.2 $\pm$ 0.1 \\
\quad $\sigma=0.75$ & 256.9 $\pm$ 0.1 & 169.8 $\pm$ 0.2 & 195.9 $\pm$ 0.3 & 127.7 $\pm$ 0.1 & 122.9 $\pm$ 0.1 & 68.3 $\pm$ 0.1 & 73.9 $\pm$ 0.1 & 40.6 $\pm$ 0.1 \\
\quad $\sigma=1.0$ & 283.7 $\pm$ 0.4 & 232.7 $\pm$ 0.1 & 221.6 $\pm$ 0.4 & 184.2 $\pm$ 0.0 & 170.4 $\pm$ 0.1 & 105.9 $\pm$ 0.1 & 107.6 $\pm$ 0.1 & 65.3 $\pm$0.1 \\
\rowcolor{gray!20} \multicolumn{9}{l}{\textit{Noise-to-Data Diffusion}} \\
\quad Cos: (3.0, 0.04) & 281.0 $\pm$ 0.3 & 185.0 $\pm$ 0.0 & 237.6 $\pm$ 0.2 &143.5 $\pm$ 0.1 & 140.8 $\pm$ 0.1 & 40.0 $\pm$ 0.2 & 83.0 $\pm$ 0.2 & 18.2 $\pm$ 0.1 \\
\bottomrule
\end{tabular}}
\end{table}

\begin{table}[htbp]
\centering
\caption{CelebV-HQ (Causal) --- Overall Average Rank ($\downarrow$)}
\label{tab:celebvhq_causal_overall_avg_rank}
\begin{tabular}{lcc}
\toprule
Method & Avg Rank & Rank \\
\midrule
\rowcolor{gray!20} \multicolumn{3}{l}{\textit{ABC}} \\
\quad No BB: $\sigma$=0.5 & \textbf{1.6} $\pm$ 0.6 & 1 \\
\quad No BB: $\sigma$=0.75 & \textbf{\textit{2.0}} $\pm$ 0.0 & 2 \\
\quad No BB: $\sigma$=1.0 & 3.1 $\pm$ 0.4 & 3 \\
\rowcolor{gray!20} \multicolumn{3}{l}{\textit{Conditional Diffusion Bridge}} \\
\quad $\sigma=0.25$ & 3.4 $\pm$ 0.1 & 4 \\
\quad $\sigma=0.5$ & 5.2 $\pm$ 0.1 & 5 \\
\quad $\sigma=0.75$ & 6.1 $\pm$ 0.1 & 6 \\
\quad $\sigma=1.0$ & 7.2 $\pm$ 0.7 & 7 \\
\rowcolor{gray!20} \multicolumn{3}{l}{\textit{Noise-to-Data Diffusion}} \\
\quad Cos: (3.0, 0.04) & 7.3 $\pm$ 0.6 & 8 \\
\bottomrule
\end{tabular}
\end{table}

\newpage
\clearpage

\subsection{SEVIR: Additional Results}\label{sec:sevir_additional_results}


\subsubsection{Qualitative Comparison}

\begin{figure}[h]
    \includegraphics[width=\textwidth]{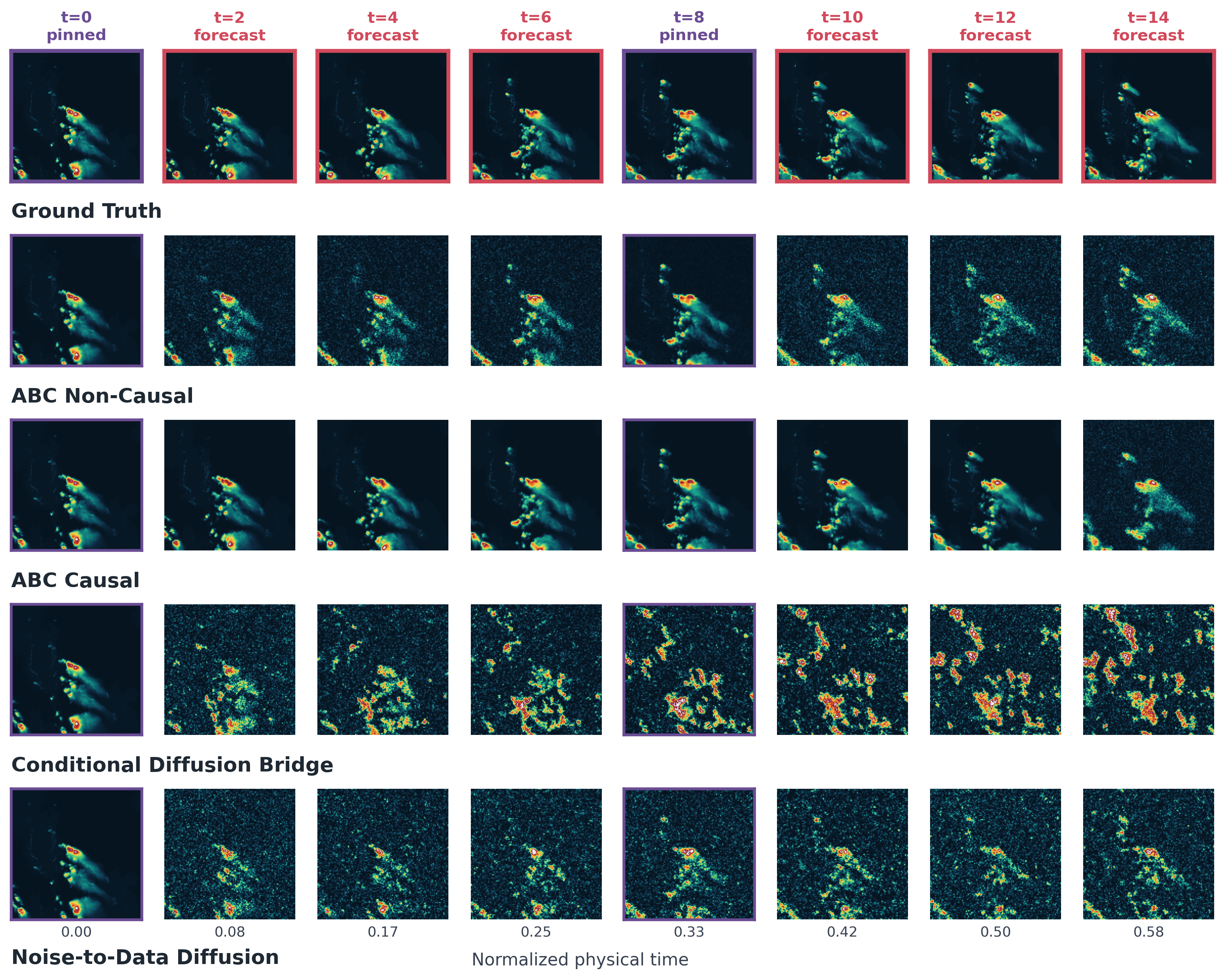}
    \caption{\textbf{SEVIR qualitative comparison.} ABC Non-Causal, Conditional Diffusion Bridge, and Noise-to-Data Diffusion are evaluated under the pinned non-causal protocol, while ABC Causal is shown separately under the stricter past-only protocol. Individual examples may favor either protocol visually, but the aggregate metrics in Table \ref{tab:sevir_main_metrics} 
    favors ABC Non-Causal among the pinned methods.}\label{fig:sevir_timelapse_comparison}
\end{figure}

\subsubsection{SEVIR: Additional Hyperparameter Sweep}\label{sec:more_sevir_hyperparameters}

After the initial round of SEVIR experiments referenced in the main text (which used constant $\sigma=1.0$), we had some extra computational resources and time; so we ran another hyperparameter sweep over $\sigma \in \{0.25, 0.5, 0.75\}$. The results again clearly show ABC's advantage. See Tables \ref{tab:sevir-error-csi-hparam-sweep}, \ref{tab:sevir-pod-far-bias-hparam-sweep}.

\begin{table*}[h]
\centering
\caption{SEVIR results: error and CSI metrics. Best per column in \textbf{bold}, second best in \textit{italic}.}
\label{tab:sevir-error-csi-hparam-sweep}
\resizebox{\textwidth}{!}{%
\begin{tabular}{lccccccc}
\toprule
Method & MAE ($\downarrow$) & RMSE ($\downarrow$) & CSI$_{16}$ ($\uparrow$) & CSI$_{74}$ ($\uparrow$) & CSI$_{133}$ ($\uparrow$) & CSI$_{160}$ ($\uparrow$) & Avg. Rank ($\downarrow$) \\
\midrule
\multicolumn{8}{l}{\cellcolor{gray!25}\textbf{ABC Non-Causal}} \\
$\sigma = 0.25$ & \textbf{14.00} & \textbf{32.21} & \textbf{0.739} & \textbf{0.731} & \textbf{0.662} & \textbf{0.678} & \textbf{1.00} \\
$\sigma = 0.5$ & \textbf{\textit{15.08}} & \textbf{\textit{33.29}} & \textbf{\textit{0.705}} & \textbf{\textit{0.716}} & \textbf{\textit{0.656}} & \textbf{\textit{0.671}} & \textbf{\textit{2.00}} \\
$\sigma = 0.75$ & 17.14 & 34.05 & 0.602 & 0.709 & 0.651 & 0.669 & 3.50 \\
\multicolumn{8}{l}{\cellcolor{gray!25}\textbf{ABC Causal}} \\
$\sigma = 0.25$ & 21.38 & 54.52 & 0.665 & 0.583 & 0.487 & 0.502 & 4.17 \\
$\sigma = 0.5$ & 22.40 & 54.26 & 0.631 & 0.583 & 0.485 & 0.501 & 4.67 \\
$\sigma = 0.75$ & 23.15 & 54.55 & 0.595 & 0.582 & 0.483 & 0.499 & 6.17 \\
\multicolumn{8}{l}{\cellcolor{gray!25}\textbf{Autoregressive Diffusion Bridge}} \\
$\sigma = 0.25$ & 40.87 & 79.26 & 0.423 & 0.349 & 0.313 & 0.326 & 9.50 \\
$\sigma = 0.5$ & 43.45 & 81.05 & 0.381 & 0.346 & 0.254 & 0.245 & 10.67 \\
$\sigma = 0.75$ & 54.83 & 89.82 & 0.303 & 0.231 & 0.155 & 0.142 & 12.00 \\
\multicolumn{8}{l}{\cellcolor{gray!25}\textbf{Noise-to-Data Diffusion}} \\
$\sigma = 0.25$ & 24.44 & 59.37 & 0.633 & 0.535 & 0.382 & 0.363 & 6.50 \\
$\sigma = 0.5$ & 33.16 & 64.93 & 0.462 & 0.399 & 0.206 & 0.145 & 9.33 \\
$\sigma = 0.75$ & 33.13 & 65.36 & 0.424 & 0.476 & 0.339 & 0.326 & 8.50 \\
\bottomrule
\end{tabular}%
}
\end{table*}

\begin{table*}[h]
\centering
\caption{SEVIR results: POD, FAR, and Bias metrics. Best per column in \textbf{bold}, second best in \textit{italic}; for Bias, ``best'' is the value closest to 1.0.}
\label{tab:sevir-pod-far-bias-hparam-sweep}
\resizebox{\textwidth}{!}{%
\begin{tabular}{lcccccccccccc}
\toprule
Method & POD$_{16}$ ($\uparrow$) & POD$_{74}$ ($\uparrow$) & POD$_{133}$ ($\uparrow$) & POD$_{160}$ ($\uparrow$) & FAR$_{16}$ ($\downarrow$) & FAR$_{74}$ ($\downarrow$) & FAR$_{133}$ ($\downarrow$) & FAR$_{160}$ ($\downarrow$) & Bias$_{16}$ & Bias$_{74}$ & Bias$_{133}$ & Bias$_{160}$ \\
\midrule
\multicolumn{13}{l}{\cellcolor{gray!25}\textbf{ABC Non-Causal}} \\
$\sigma = 0.25$ & \textbf{0.947} & \textbf{0.863} & 0.776 & 0.760 & \textbf{\textit{0.230}} & \textbf{0.174} & \textbf{0.181} & \textbf{0.138} & 1.230 & 1.045 & 0.948 & 0.882 \\
$\sigma = 0.5$ & 0.930 & \textbf{\textit{0.855}} & \textbf{0.787} & \textbf{0.774} & 0.255 & \textbf{\textit{0.185}} & 0.203 & 0.165 & 1.249 & 1.050 & \textbf{0.987} & \textbf{0.926} \\
$\sigma = 0.75$ & \textbf{\textit{0.936}} & 0.852 & \textbf{\textit{0.777}} & \textbf{\textit{0.765}} & 0.373 & 0.191 & \textbf{\textit{0.200}} & \textbf{\textit{0.157}} & 1.491 & 1.053 & \textbf{\textit{0.971}} & \textbf{\textit{0.908}} \\
\multicolumn{13}{l}{\cellcolor{gray!25}\textbf{ABC Causal}} \\
$\sigma = 0.25$ & 0.823 & 0.728 & 0.628 & 0.636 & \textbf{0.225} & 0.254 & 0.316 & 0.297 & \textbf{1.062} & \textbf{0.976} & 0.917 & 0.905 \\
$\sigma = 0.5$ & 0.822 & 0.726 & 0.620 & 0.629 & 0.269 & 0.252 & 0.311 & 0.290 & 1.125 & \textbf{\textit{0.971}} & 0.900 & 0.886 \\
$\sigma = 0.75$ & 0.817 & 0.725 & 0.625 & 0.634 & 0.314 & 0.253 & 0.320 & 0.299 & 1.190 & 0.970 & 0.920 & 0.904 \\
\multicolumn{13}{l}{\cellcolor{gray!25}\textbf{Autoregressive Diffusion Bridge}} \\
$\sigma = 0.25$ & 0.670 & 0.582 & 0.579 & 0.613 & 0.466 & 0.534 & 0.595 & 0.589 & 1.255 & 1.249 & 1.431 & 1.491 \\
$\sigma = 0.5$ & 0.697 & 0.545 & 0.461 & 0.457 & 0.544 & 0.513 & 0.640 & 0.654 & 1.529 & 1.118 & 1.279 & 1.320 \\
$\sigma = 0.75$ & 0.654 & 0.427 & 0.313 & 0.279 & 0.640 & 0.665 & 0.765 & 0.775 & 1.813 & 1.275 & 1.334 & 1.240 \\
\multicolumn{13}{l}{\cellcolor{gray!25}\textbf{Noise-to-Data Diffusion}} \\
$\sigma = 0.25$ & 0.800 & 0.656 & 0.487 & 0.461 & 0.249 & 0.255 & 0.361 & 0.369 & \textbf{\textit{1.066}} & 0.881 & 0.762 & 0.732 \\
$\sigma = 0.5$ & 0.770 & 0.512 & 0.255 & 0.167 & 0.463 & 0.355 & 0.481 & 0.478 & 1.435 & 0.794 & 0.492 & 0.319 \\
$\sigma = 0.75$ & 0.826 & 0.621 & 0.477 & 0.461 & 0.534 & 0.330 & 0.459 & 0.474 & 1.775 & 0.926 & 0.882 & 0.878 \\
\bottomrule
\end{tabular}%
}
\end{table*}

\newpage
\clearpage

\subsection{Runtime Performance Profiling}

We ran profiling experiments generating Sky-Timelapse videos from a single prompt frame, and report mean$\pm$SEM across batches. Per setting, we generate 800 videos in batches of 16 videos per GPU (A100, 80GB).

\begin{table}[h]
\centering
\caption{Cost vs. number of conditioning frames (32 generated frames, 250 SDE steps)}
\label{tab:cost-conditioning-frames}
\resizebox{\columnwidth}{!}{%
\begin{tabular}{lrrrrrrrr}
\toprule
Conditioning frames & 1 & 5 & 9 & 13 & 17 & 21 & 25 & 29 \\
\midrule
Denoising steps measured & 400 & 400 & 400 & 400 & 400 & 400 & 400 & 400 \\
GPU time per denoising step (ms) & 108.03 $\pm$ 1.76 & 111.64 $\pm$ 1.74 & 118.57 $\pm$ 1.76 & 131.84 $\pm$ 1.85 & 141.39 $\pm$ 1.79 & 149.08 $\pm$ 1.27 & 157.12 $\pm$ 1.28 & 169.99$\pm$ 1.26 \\
Peak memory (GiB) & 6.440 $\pm$ 0.000 & 6.632 $\pm$ 0.000 & 6.842 $\pm$ 0.000 & 7.056 $\pm$ 0.000 & 7.272 $\pm$ 0.000 & 7.485 $\pm$ 0.000 & 7.696 $\pm$ 0.000 & 7.911 $\pm$ 0.000 \\
\bottomrule
\end{tabular}%
}
\end{table} 

\begin{table}[h]
\centering
\caption{Cost vs. number of generated frames (250 SDE steps)}
\label{tab:cost-generated-frames}
\resizebox{\columnwidth}{!}{%
\begin{tabular}{lrrr}
\toprule
Generated frames & 8 & 16 & 32 \\
\midrule
Sampling GPU time (s) & 23.46 $\pm$ 1.60 & 27.01 $\pm$ 1.37 & 34.74 $\pm$ 1.02 \\
VAE decode GPU time (s) & 1.51 $\pm$ 0.00 & 2.85 $\pm$ 0.00 & 5.53 $\pm$ 0.00 \\
Total wall time (s) & 24.96 $\pm$ 1.60 & 29.85 $\pm$ 1.37 & 40.27 $\pm$ 1.02 \\
Wall time per video (s) & 1.56 $\pm$ 0.10 & 1.87 $\pm$ 0.09 & 2.52 $\pm$ 0.06 \\
Peak memory, sampling (GiB) & 6.21 $\pm$ 0.00 & 6.83 $\pm$ 0.00 & 8.07 $\pm$ 0.00 \\
Peak memory, sampling + decode (GiB) & 12.95 $\pm$ 0.00 & 13.24 $\pm$ 0.00 & 13.82 $\pm$ 0.00 \\
\bottomrule
\end{tabular}%
}
\end{table} 

\begin{table}[h]
\centering
\caption{Cost vs. number of SDE steps (8 generated frames)}
\label{tab:cost-sde-steps}
\resizebox{\columnwidth}{!}{%
\begin{tabular}{lrrr}
\toprule
SDE steps & 125 & 250 & 375 \\
\midrule
Sampling GPU time (s) & 15.00 $\pm$ 0.88 & 23.46 $\pm$ 1.60 & 32.57 $\pm$ 2.10 \\
VAE decode GPU time (s) & 1.51 $\pm$ 0.00 & 1.51 $\pm$ 0.00 & 1.51 $\pm$ 0.00 \\
Total wall time (s) & 16.50 $\pm$ 0.88 & 24.96 $\pm$ 1.60 & 34.07 $\pm$ 2.10 \\
Wall time per video (s) & 1.03 $\pm$ 0.05 & 1.56 $\pm$ 0.10 & 2.13 $\pm$ 0.13 \\
Peak memory, sampling (GiB) & 6.21 $\pm$ 0.00 & 6.21 $\pm$ 0.00 & 6.21 $\pm$ 0.00 \\
Peak memory, sampling + decode (GiB) & 12.95 $\pm$ 0.00 & 12.95 $\pm$ 0.00 & 12.95 $\pm$ 0.00 \\
\bottomrule
\end{tabular}%
}
\end{table} 

\newpage
\clearpage

\section{Architecture}
\label{sec:architecture}

We build on the Diffusion Transformer (DiT)~\cite{peebles2023scalable} and extend it with image-based conditioning, conditioning on the next waypoint $t_\text{next} = t_{i^*+1}$, and an optional Brownian-bridge drift term. The output is the estimated drift from Equation \ref{eqn:simplified_dsm_main}. \textit{All methods compared in the paper use the same architecture.}

\paragraph{Inputs and embeddings.} In addition to the noisy latent $x_t \in \mathbb{R}^{C \times H \times W}$, diffusion timestep $t \in [0,1]$, and class label $y$ used by DiT, our model takes a next-waypoint timestep $t_{\text{next}} \in [0,1]$, a set of $L$ conditioning frames $\{x^{(\ell)}\}_{\ell=1}^{L}$ with associated timesteps $\{t^{(\ell)}\}$, and a boolean validity mask $m \in \{0,1\}^{L}$ (with $m_0 = 1$ always). The conditioning frames are patchified by a dedicated embedder $\mathrm{PatchEmbed}_{\text{cond}}$ and additively tagged with the shared sin-cos positional embedding plus a timestep embedding $\phi_{\text{cond}}(t^{(\ell)})$, yielding tokens $z^{(\ell)} \in \mathbb{R}^{T \times D}$. Timesteps are rescaled by a factor of $1000$ before sinusoidal embedding. The AdaLN conditioning vector becomes
\[
c \;=\; \tfrac{1}{2}\bigl(\phi_t(t) + \phi_{t_{\text{next}}}(t_{\text{next}})\bigr) + \phi_y(y),
\]
where $\phi_t$, $\phi_{t_{\text{next}}}$, and $\phi_{\text{cond}}$ are independent timestep MLPs. 

\paragraph{Block structure.} Each block extends the DiT block with a cross-attention sub-layer inserted between self-attention and the MLP:
\begin{align}
&x \leftarrow x + g_{\text{sa}}\,\mathrm{SA}(\mathrm{mod}(x; s_{\text{sa}}, \sigma_{\text{sa}})), \\
&x \leftarrow x + g_{\text{xa}}\,\mathrm{XA}\bigl(\mathrm{mod}(x; s_{\text{xa}}, \sigma_{\text{xa}}),\; \mathrm{mod}(z; s_z, \sigma_z)\bigr), \\
&x \leftarrow x + g_{\text{mlp}}\,\mathrm{MLP}(\mathrm{mod}(x; s_{\text{mlp}}, \sigma_{\text{mlp}})),
\end{align}

so the AdaLN-Zero MLP emits $11D$ modulation parameters per block (vs.\ $6D$ in DiT). Two additional LayerNorms are introduced for the cross-attention query and key/value streams. We apply RMSNorm \cite{zhang2019root} to the per-head queries and keys in both self- and cross-attention to stabilize training on long contexts (the orignal DiT did not include this, but we found the exclusion to be very unstable).

\paragraph{Variable-length cross-attention.} For cross-attention, we use PyTorch's \texttt{varlen\_attn}, which is based on FlashAttention \cite{dao2022flashattention, dao2023flashattention}. 

\paragraph{Brownian-bridge drift.} 
We have an option to add the Brownian Bridge drift specified in Section \ref{subsec:any_subset}. We add a small epsilon of $10^{-7}$ to the denominator for stability.

\paragraph{Initialization.} We follow DiT's AdaLN-Zero initialization, additionally applying Xavier initialization \cite{glorot2010understanding} to the conditioning patch embedder and $\mathcal{N}(0, 0.02^2)$ to the new timestep MLPs ($\phi_{\text{cond}}$ and $\phi_{t_{\text{next}}}$).

\paragraph{Model Size}
We use the DiT-B model size. With our cross-attention layers and other architectural modifications, this works out to 202,192,912 parameters.

\paragraph{EMA} We use EMA 0.9999, following DiT defaults.

\newpage
\clearpage

\section{Video Experimental Setup Details}\label{sec:video_experimental_setup}

\subsection{Dataset Details}\label{subsec:video_dataset}

We train on 178,688 and evaluate on 2,048 clips (disjoint) of 32 frames each from CelebV-HQ \cite{zhu2022celebv}. We train on 17,152 and evaluate on 1,360 clips (disjoint) of 64 frames each from Sky-Timelapse \cite{zhang2020dtvnet}. We diffuse in the Stable Diffusion latent space \cite{rombach2022high} on 32x32x4 codes, corresponding to 256x256x3 images.

\subsection{Training}\label{subsec:training_details}

We sample the time $t \sim \mathcal{U}[0, 1]$, then clip it so that $|t - t_i| > 10^{-4}$, for every finite-dimensional marginal time $t_i$. For the autoregressively conditioned diffusion bridge (and noise-to-data diffusion), we get the auxiliary time $\tau_i = \frac{t - t_{i^*}}{t_{i^*+1} - t_{i^*}}$, as constructed in Equation \ref{eqn:time_shifted_bridge}.

We have cosine decay learning rate, with linear warmup for 10K steps up to rate $10^{-4}$, then decay for 290K steps. Due to compute limits, we evaluate the model from 240K steps. We use Adam \cite{kingma2014adam} without weight decay, and $\beta = (0.9, 0.999)$. We clip the gradient to maximum norm of $5.0$. We use batch size $256$. We run each training run on 4 A100 80GB GPUs, and each run takes about 36 hours.

\subsection{SDE Parameter Selection}\label{subsec:sde_parameter_selection}

There is a very large design space of SDE parameters that could have been used in Equations \ref{eqn:new_sde} and \ref{eqn:cond_diffusion_bridge}. To isolate the effects of \rqref{time_scaling} (impact of time-dependent quadratic variation), \rqref{data_to_data} (impact of data-to-data bridge), we picked relatively simple choices. Future work can explore more complicated parameterizations.

We always set $t_L = 1$, ensuring that the SDE is defined on $[0, 1]$.

We set the base (non-score driven) drift to be $a(t) = 0$. The score network should be able to learn the drift, without need for an explicit specification.

We use constant volatility $\sigma(t) = \sigma$ in ABC and the conditional diffusion bridge baseline. 
This is because a non-constant volatility (say, small at the finite-dimensional marginal times, but large in between them, as is popular in DDBM \cite{zhou2023denoising}) would mean that the generative process would not have physically meaningful intermediate values in between the times at which we observe the finite-dimensional marginals. This is a non-starter for continuous-time modeling, as it kills any hope of generalization to modeling times that were not on the grid that the training data was sampled from.
In contrast, using a constant volatility decouples the sampling rate of the data from the underlying model of the stochastic process.

We selected the volatility by testing reconstruction fidelity in the Stable Diffusion VAE latent space \cite{rombach2022high} with various Gaussian noise perturbations. We found that perturbations up to $\mathcal{N}(0, 0.1^2 I)$ minimally affected the reconstruction quality. Assuming simulation step sizes up to $\Delta t \leq 0.01$, the upper limit on the volatility introduced from Brownian motion should be $\sigma \sqrt{\Delta t} \leq 0.1 \Rightarrow \sigma \leq 1.0$. 

To select volatility for the baseline conditional diffusion bridge, we tried a few options. 
The first, most naive one was to set the volatility coefficient exactly the same as that of ABC's:
\begin{equation}
    \sigma_\text{cond bridge} = \sigma_\text{ABC}
\end{equation}
Another option was to calculate the volatility coefficient based on the quadratic variation (up to rounding) that would have been introduced by ABC in between frames (given a frame sampling rate): 
\begin{equation}
    \sigma_\text{cond bridge} = \sqrt{t_{i^*+1} - t_{i^*}} * \sigma_\text{ABC},
\end{equation}
in accordance with time-scaling rules for Brownian motion \cite{steele2001stochastic}. 
Finally, when compute resources permitted, we tried a few other volatility values.

For the noise-to-data diffusion baseline, we used exponentially decaying volatility (Section \ref{subsec:exponential_decay_volatility}) and cosine decaying volatility (Section \ref{subsec:cosine_decay_volatility}).

\subsection{Inference}

We also concatenate the final frame from $t = 1$ as a prompt. 

\subsection{Evaluation}\label{subsec:evaluation_details}

FVD and FID are computed on the infilled videos (including the prompt frames), with the same number of videos (1360 in Sky Timelapse, 2048 in CelebV-HQ) for each of the ground truth and fake sets. We exclude the final frame ($t=1$) from computations.
FVD computations use VideoMAE v2 backbone with interpolated time embeddings \cite{wang2023videomae, tong2022videomae, ge2024content}. FID computations use the default Inception net \cite{szegedy2016rethinking}.

\newpage
\clearpage

\section{Other Experimental Setups}

\subsection{Toy Experiment: Incorrect Dynamics in Chained Conditional Bridges}
\label{sec:stitched-vs-conditioned}

We conduct a toy experiment to empirically illustrate Theorem~\ref{thm:non_bijective_path_measure}. 
Although the finite-dimensional marginals of autoregressively concatenated conditional diffusion bridges can agree with those of an underlying stochastic process, the path measures and volatility in general will not coincide. This underscores the need for one continual SDE, as in our ABC framework.

\paragraph{Reference process.}
Let $X(t)$ be standard Brownian motion on $[0,1]$ with $X(0)=x_0$, Doob-conditioned \cite{doob1984classical} to satisfy $X(4/5)=-x_0$ and $X(1)=+x_0$. Its SDE is piecewise \cite{karatzas2014brownian, oksendal2003stochastic, steele2001stochastic}
\begin{equation}
dX(t) =
\begin{cases}
\dfrac{-x_0 - X(t)}{4/5 - t}\,dt + dW(t), & t \in [0,\,4/5],\\[6pt]
\dfrac{\phantom{-}x_0 - X(t)}{1 - t}\,dt + dW(t), & t \in [4/5,\,1].
\end{cases}
\end{equation}

\paragraph{Stitched construction.}
As an alternative, we build $Y(t)$ from two unit-time Brownian bridges
defined on their own time variables. Let $Y_0(\tau_0)$ be a Brownian
bridge on $\tau_0\in[0,1]$ with $Y_0(0)=x_0$, conditioned on
$Y_0(1)=-x_0$, and let $Y_1(\tau_1)$ be a Brownian bridge on
$\tau_1\in[0,1]$ with $Y_1(0)=Y_0(1)=-x_0$, conditioned on
$Y_1(1)=-Y_0(1)=+x_0$. Their SDEs are the standard Brownian-bridge
forms
\begin{align}
dY_0(\tau_0) &= \frac{-x_0 - Y_0(\tau_0)}{1-\tau_0}\,d\tau_0 + dB_0(\tau_0),\\
dY_1(\tau_1) &= \frac{\phantom{-}x_0 - Y_1(\tau_1)}{1-\tau_1}\,d\tau_1 + dB_1(\tau_1),
\end{align}
where $B_0,B_1$ are independent standard Brownian motions in their
respective time variables. We then define
\begin{equation}
Y(t) =
\begin{cases}
Y_0\!\left(\tfrac{5}{4}t\right), & t\in[0,\,4/5],\\[4pt]
Y_1\!\left(5\,(t-4/5)\right),    & t\in[4/5,\,1],
\end{cases}
\end{equation}
so that each unit-time bridge is compressed into its respective
sub-window.

\paragraph{Derivation of the SDE for $Y(t)$.}
Consider the first segment $t\in[0,4/5]$ with $\tau_0=\tfrac{5}{4}t$,
hence $d\tau_0/dt=5/4$. The drift transforms by the chain rule (multiply
by $d\tau_0/dt$):
\begin{equation}
\frac{-x_0 - Y_0(\tau_0)}{1-\tau_0}\cdot\frac{d\tau_0}{dt}
= \frac{5}{4}\cdot\frac{-x_0 - Y(t)}{1-\tfrac{5}{4}t}
= \frac{-x_0 - Y(t)}{4/5 - t}.
\end{equation}
For the diffusion term, $B_0(\tau_0) \sim \mathcal{N}(0, \tau_0) \overset{\text{dist}}{=} \mathcal{N}\left(0, \frac{5}{4}t\right)$, so the diffusion
coefficient is $\sqrt{5/4}$ \cite{karatzas2014brownian, oksendal2003stochastic, steele2001stochastic}. An identical argument on the second segment
with $\tau_1=5(t-4/5)$, $d\tau_1/dt=5$, gives drift
$(x_0 - Y(t))/(1-t)$ and diffusion coefficient $\sqrt{5}$. Hence
\begin{equation}
dY(t) =
\begin{cases}
\dfrac{-x_0 - Y(t)}{4/5 - t}\,dt + \sqrt{\tfrac{5}{4}}\,dW(t),
  & t \in [0,\,4/5],\\[6pt]
\dfrac{\phantom{-}x_0 - Y(t)}{1 - t}\,dt + \sqrt{5}\,dW(t),
  & t \in [4/5,\,1].
\end{cases}
\label{eq:Y-sde}
\end{equation}
The drifts of $X$ and $Y$ coincide, but the diffusion coefficients of
$Y$ are inflated by $\sqrt{d\tau/dt}$ on each segment---a factor that is
easy to overlook when ``reusing'' bridges across rescaled
time windows.

\paragraph{Empirical Results.}
$X$ and $Y$ have identical finite-dimensional marginals at the checkpoint set
$\{0,\,4/5,\,1\}$ (all three are point masses at $x_0,-x_0,x_0$), yet
their path measures do not coincide: their quadratic variations
differ pathwise, e.g.\ $\langle Y\rangle_t=\tfrac{5}{4}t$ versus
$\langle X\rangle_t=t$ on the first segment. Figure~\ref{fig:insufficient_stiched_bridge}
confirms this empirically with $1000$ simulated trajectories ($4000$ discretization steps): the two
processes coincide in distribution at the pinned times $t\in\{4/5,1\}$
but differ markedly at intermediate times $t\in\{1/2,9/10\}$, with $Y$
exhibiting the predicted variance inflation.
In this experiment, we use the analytic form of the drift (with small $\epsilon$ in denominator to promote numerical stability) to isolate the effect of neural network learning from the fundamental issue of path measures coinciding.

\subsection{Weather Forecasting Experimental Setup}\label{sec:sevir_additional_details}

For SEVIR, we use the VIL modality \cite{veillette2020sevir} and construct a benchmark protocol from 25-frame clips taken from the 2019 radar sequences, with a center clip start at frame 12. This yields 13 context frames and 12 forecast frames per clip. We split events at 2019-06-01, reserve 10\% of the pre-cutoff events for validation, and obtain 2,729 training events, 334 validation events, and 847 test events. Training is performed on resized $128\times128$ single-channel VIL fields, while evaluation is reported at native resolution by upsampling predictions before scoring. 
We train for 20 epochs at a constant learning rate of $10^{-4}$, with a batch size of 4, on a single GPU: this takes around three hours. We gradient clip the norms to maximum $1$, and use EMA 0.999.

We compare four methods: ABC Non-Causal, ABC Causal, Conditional Diffusion Bridge, and Noise-to-Data Diffusion. For the SEVIR experiments mentioned in the main text and the qualitative figure (Table \ref{tab:sevir_main_metrics}, 
Figure \ref{fig:sevir_timelapse_comparison}), all the benchmarked methods used the same constant volatility schedule of $\sigma=1.0$. For completeness, we also try different hyperparameters (Section \ref{sec:more_sevir_hyperparameters}).

In the pinned non-causal protocol, the model conditions on the first, every 8th, and final frame of the clip and predicts the remaining frames; in the causal protocol, it conditions on the first 13 frames and predicts the final 12. All SEVIR results use 250 sampling steps and one sample per test example. We report MAE, MSE, RMSE, and thresholded CSI, POD, FAR, and bias at VIL thresholds $\{16,74,133,160\}$, computed only on the unobserved/forecast frames.

\newpage
\clearpage

\section{SDE Kernel Choices}\label{sec:sde_kernel_choices}

\subsection{Exponentially Decaying Volatility}\label{subsec:exponential_decay_volatility}

We set
\begin{align}
    A(t) &= A \\
    \sigma(t) &= Ke^{-Bt},
\end{align}
where $A, B, K \in \mathbb{R}_{+0}$ are positive constants. This gives us an analytic form for the Gaussian process statistics (Equation \ref{eqn:phi}):
\begin{align}
    \Phi(s, t) &= e^{-A(t-s)},\label{eqn:new_resolvent_kernel}
\end{align}
where $s < t$. Furthermore, plugging Equation \ref{eqn:new_resolvent_kernel} into Equation \ref{eqn:cov_kernel},
\begin{align}
    C_{t_{i^*}}(\tau_a, \tau_b) &= \int_{t_{i^*}}^{\text{min}(\tau_a, \tau_b)} e^{-A(\tau_a - s)}e^{-A(\tau_b - s)}K^2 e^{-2Bs} ds \\
    &= K^2 e^{-A(\tau_a + \tau_b)} \int_{t_{i^*}}^{\text{min}(\tau_a, \tau_b)} e^{2(A-B)s}ds \\
    &= K^2 e^{-A(\tau_a + \tau_b)} \frac{1}{2(A-B)}\left[e^{2(A-B)s}\right]_{s=t_{i^*}}^{s=\text{min}(\tau_a, \tau_b)} \\
    &= K^2 e^{-A(\tau_a + \tau_b)} \frac{1}{2(A-B)}\left[e^{2(A-B)\text{min}(\tau_a, \tau_b)} - e^{2(A-B)t_{i^*}}\right] \\
\end{align}

\subsection{Periodic Volatility}\label{subsec:periodic_volatility_kernel}
\begin{equation}
    \sigma(t) = \frac{\alpha}{2}\left(1 - \text{cos}\left(2\pi kt\right)\right) + \epsilon
\end{equation}
\begin{align}
    A(t) &= 0 \\
    \Phi(s, t) &= 1
\end{align}

\begin{align}
    C_{t_{i^*}}(\tau_a, \tau_b) &= \int_{t_{i^*}}^{\text{min}(\tau_a, \tau_b)}1 * 1 * \sigma(s)^2ds \\
    &= \left[\left(\frac{3\alpha^2}{8} + \alpha\epsilon + \epsilon^2\right)s - \frac{\alpha(\alpha + 2\epsilon)}{4\pi k}\text{sin}(2\pi k s) + \frac{\alpha^2}{32\pi k}\text{sin}(4\pi k s)\right]_{s = t_{i^*}}^{s = \text{min}(\tau_a, \tau_b)}
\end{align}

\subsection{Constant Volatility}
Simply set $\alpha=0, k=1$ in the periodic volatility kernel (Section \ref{subsec:periodic_volatility_kernel}).

\subsection{Cosine Decaying Volatility}\label{subsec:cosine_decay_volatility}
Use $\sigma_\text{cos decay} = \sigma_\text{cos periodic}(t - 1)$, where $k = 0.5$ (from Section \ref{subsec:periodic_volatility_kernel}).

\subsection{General Case}

In general, we should set the volatility to match the dynamics of the underlying process, which may result in a non-monotonic function. In this case, we may have to integrate numerically.

\newpage
\clearpage

\section{Related Works}\label{sec:related_work}

\subsection{Non-Markovian Diffusions}

Our method is non-Markovian, as it makes use of path history. There is also recent work in non-Markovian diffusion modeling \cite{nobis2024generative, nobis2025fractional} that approximately parameterizes the driving noise as fractional Brownian motion. Our framework actually supersedes such frameworks, assuming we have to discretize the SDE to simulate it. This is because many non-Markovian processes (including fractional Brownian Motion) can be written as Volterra SDEs \cite{burton2005volterra}:
\begin{equation}\label{eqn:volterra}
    dX(t) = -AX(t)dt + \left(\int_{0}^{t}\sigma(t, s)dB(s)\right) dt + \sigma(t, t)dB(t)
\end{equation}
In principle, if we constructed our martingale to consider every point on the discretization grid, our score network would take this entire history as input. It could then learn the middle integral over past noise history as part of the score-driven drift, assuming the underlying time series data actually has Volterra dynamics.

In fact, our framework is strictly more flexible, because it can learn from data what the best way to incorporate the path history is when generating the next sample. In contrast, the fractional Brownian Motion diffusion models use a handcrafted kernel to incorporate path history.
Furthermore, our framework allows us to take the intermediate steps of the generative process as valid samples, whereas the other works do not.

Daems et al \cite{daems2023variational} also explore variational inference for (approximate) fractional Brownian motion driven SDEs to generate videos. However, they do not make the connection to any-subset autoregressive models nor score matching.

\subsection{Time-Series Modeling}
\paragraph{Longitudinal Flow Matching}
Longitudinal flow matching \cite{islam2025longitudinal} also uses diffusion/flow processes to generate time series data, where the time variable in the generative process corresponds to the actual time in the physical time series. However, their work has a few key differences from ours:
(1) We explicitly unify diffusion (bridge) models and any-subset autoregressive models (Section \ref{subsec:any_subset}). Their procedure does not support any-subset conditioning (in the future or arbitrary past), nor is it even autoregressive, as it only conditions on the most recent observation, making their predictions Markovian.
This stems from their choice to focus on multi-marginal optimal transport (rather than learning interdependencies along the trajectory). 
(2) Their loss function and simulation require three objects (velocity, score, volatility) parameterized by learned neural networks. This complicates the learning problem with three losses: flow matching (on velocity), score matching (on score), confidence estimation (on volatility; using the velocity and scores to form a distillation-like target). We just require one item learned with one loss: the path-dependent drift (Equation \ref{eqn:simplified_dsm_main}).
(3) Our method is rigorously derived from stochastic calculus, measure theory, and Gaussian process regression. Their method is derived from different ideas, such as flow matching and optimal transport.
(4) The scope of their experiments is mostly limited to medical imaging; we tackle video generation and weather forecasting (Section \ref{sec:results}).
(5) They do not provide detailed analysis on the benefits of time-adaptive volatility, and how that affects learning of path measures (\textit{i.e.}, correctness of trajectories). We provide both theoretical (Theorem \ref{thm:non_bijective_path_measure}) and empirical evidence (Figure \ref{fig:insufficient_stiched_bridge}) of this phenomenon. See Section \ref{sec:cond_ablation} for an empirical comparison.

\paragraph{Time Series Schrodinger Bridge}
SBTS diffusion \cite{hamdouche2023generative, alouadi2025robust} proposes to frame time series generation as a Schrodinger bridge problem between path measures, and also uses a similar measure-theoretic argument with Girsanov's Theorem to derive their method. However, their work has a few key limitations:
(1) They assume a fixed discretization, and only model causally through time. As such, they do not support any-subset autoregressive modeling \cite{guo2025reviving}, either with respect to the past conditioning or future conditioning.
(2) Their learning objectives for the drift do not scale to high-dimensions. In contrast, we show that the change-of-measure and Gaussian process regression can ultimately be used to derive a tractable, scalable denoising score matching objective \cite{vincent2011connection}; they fail to see this connection.
(3) They do not show success on any high-dimensional tasks at scale (largest is MNIST).
(4) They were generally unable to leverage the power of modern deep learning architectures, at most using LSTMs on low-dimensional problems.

\paragraph{Multi-Marginal Schrodinger Bridges}
DMSB \cite{chen2023deep}, SF$^2$M \cite{tong2023simulation}, 3MSBM \cite{theodoropoulos2025momentum}, and SSB \cite{garg2024soft} (this one lays out the theory, but no experiments) learn multi-marginal Schrodinger Bridges, but (1) only respect the marginal rather than joint distributions, (2) do not make the connection to any-subset autoregressive modeling, 
(3) have limited experimental validation on high-dimensional problems. Smooth Schrodinger Bridges \cite{hong2025trajectory} also tackles trajectory inference, but (1) their belief-propagation-based method does not scale to high-dimensional problems, (2) they do not use neural networks, 
(3) while they do draw some connections to autoregressive modeling and Gaussian processes like our work, they do not show any connection to any-subset autoregressive models. 
Chen et al \cite{chen2023provably} also tried to generate time series with Schrodinger bridges, but they modeled the time series as the high-dimensional endpoint distribution of the SDE, rather than a distribution over intermediate variables in the generative SDE itself.

\paragraph{Multi-Marginal Flow Matching}

Multi-Marginal flow matching also tries to respect certain marginal distributions at intermediate times in flow matching trajectories \cite{kviman2025multi}. However, they do not actually learn the joint distribution (but only the marginals), and do not account for any-subset inference (in discrete or continuous time).

\paragraph{Trajectory Flow Matching (TFM)} Like ABC, TFM~\cite{zhang2024trajectory} attempts to model time series and indeed has a time-sensitive velocity. However, they do not actually guarantee to learn the joint time series distribution. 

They use a reweighted flow-matching objective (TFM Sec. 3.2) to train the SDE drift, which learns a history-conditioned point estimate of the next state. Their learned “volatility” (TFM Sec. 3.4), as they note, is more like an epistemic uncertainty estimate for their next-state predictor of $x_\text{next}$. 

They admit their model only learns the correct joint probability distribution when the trajectory is effectively deterministic given observation history (TFM Prop. 3.2, (A1)-(A3)). Intuitively, the loss described in their Prop. 3.2 only achieves zero if there is no stochasticity in the target. They say they need a Monge map (deterministic OT) in order for their model to preserve the couplings.

In contrast, ABC is explicitly designed to model the correct joint distributions of a time series, and can handle distributions where the trajectory is \textit{not} deterministic given the input. 

\paragraph{SDE Matching} Bartosh et al. \cite{bartosh2025sde} also use an SDE to model time-series. Their method is quite different from ours. 
Their score function estimator relies on a network mapping from noise to latent state (their Section 4.2, Eqs. 17-20) to parameterize an auxiliary ODE that becomes part of the SDE construction. They generally require invertible parameterizations (\textit{e.g.}, Real NVP~\cite{dinh2017density}, affine maps) for which the log-determinant of the Jacobian matrix is efficiently computable, restricting their architectural choices. In contrast, we do not need an auxiliary ODE nor noise-to-latent mapping, and we can use more flexible architectures to directly learn scores.

\paragraph{Probabilistic Forecasting with Stochastic Interpolants}

See Section \ref{sec:pfi_comparison}.

\subsection{Video Generative Models}

There is a rich body of work in video generative models. Generally, the strategy is to generate groups of frames autoregressively \cite{ho2022video}. Within this paradigm, there is work on efficiently compressing the frame context \cite{zhang2025pretraining, zhang2025frame, cai2025mixture}; architectural improvements \cite{ma2024latte}; noising schedule \cite{chen2024diffusion, sun2025ar}.
Fundamentally, they \textit{all stay within the noise-to-data diffusion paradigm}: we introduce a new paradigm of a continual SDE data-to-data bridge. Furthermore, they largely treat videos as discrete-time signals, rather than our continuous-time treatment.

\textit{Our contribution is on the probabilistic modeling side, and therefore orthogonal to these engineering improvements; it could, in principle, be applied jointly with many of them. We accordingly design our experiments to isolate methodological gains via controlled ablations of our own method, rather than to chase SOTA against models that differ along multiple architectural and scaling axes.}  

\subsection{Diffusion Bridges}
See Section \ref{sec:comparison} for detailed comparison. The popular denoising diffusion models/score-based generative models \cite{song2020score, ho2020denoising} are special cases of denoising diffusion bridge models \cite{zhou2023denoising}. 
While these works only account for bridges between two endpoints (\textit{e.g.}, noise to data), we extend the bridging methodology to the non-Markovian (any-subset, non-causal) case.
Another difference between these works and ours is that they derive their method through Anderson's time-reversal \cite{anderson1982reverse}, while we only deal with the forward-time process via a change-of-measure to force it to generate data; this results in slightly different parameterizations of the score function.

There are also works that have a similar formulation of forward-in-time diffusion bridges \cite{liu2022let, peluchetti2023diffusion, peluchetti2023non}. \cite{liu2022let} actually does consider non-Markovian stochastic processes in their initial general formulation, but they unfortunately do not take this idea further. Ultimately, they resort to matching individual marginals $\prod_1^L p(x_{t_i})$, rather than the finite-dimensional marginals $p(x_{t_1}, \ldots, x_{t_L})$. Furthermore, their work only considers \textit{two-endpoint} bridges, \textit{i.e.}, they only care about the sample at $t = 0$ and $t = 1$, and completely ignores time series data that could be modeled by an SDE.

\subsection{Autoregressive Models}

\paragraph{Continuous-Space Autoregressive Models}
Our method also is related to a line of work that uses diffusion models to parameterize the next-token prediction in continuous-space autoregressive models (\textit{e.g.}, image generation) \cite{li2024autoregressive, ho2022cascaded}. However, we have a few key differences: 
(1) Our method operates in continuous time, whereas the previous methods conceptually run in discrete time with respect to the autoregressive generation (\textit{e.g.}, they predict a fixed number of tokens at fixed "grid" points). This means that our method can predict arbitrarily many points in the time series, whereas the previous ones cannot.
(2) Our method uses diffusion bridges to parameterize the model, which is conceptually closer to how certain time series (\textit{e.g.}, dynamical systems) operate. Whereas previous methods transform noise into data, we have the inductive bias that we really only need to make some small perturbations to the previous data to get the next data point.
(3) The time index of our sampling process also corresponds to the physical time index of the generated time series data, since we have the property that intermediate distributions generated by the SDE adhere to the data distribution.

\paragraph{Any-Order and Any-Subset Autoregressive Models}
Our method is related to any-order and any-subset autoregressive models \cite{shih2022training, guo2025reviving}, as we can infill trajectories and condition on arbitrary subsets of the trajectory. However, our method advances the paradigm in a few substantial ways: (1) We operate in continuous time. (2) We unearth a novel parameterization of any-subset autoregressive rollouts as diffusion bridge modeling. (3) Any-subset autoregressive models typically operate on discrete data, not continuous data.

\subsection{Meta Flow Maps}

Meta Flow Maps \cite{potaptchik2026metaflowmapsenable} are a recently introduced framework that touches on some similar ideas: namely, conditioning differential equation-based generative models on some intermediate point in the trajectory. That is, they generate samples from the posterior $p_{1|t}(|x_t)$, where $1$ is the time at which we get data, and $t$ is an intermediate timestep we want to condition on. 

However, we have a few crucial differences: 
(1) When simulating the conditional flow, their method does not actually hit the intermediate trajectory point $x_t$ that the model was conditioned on: it only produces samples from the posterior $p_{1 | t}(| x_t)$. Therefore, the modeling choice is fundamentally different from the dynamics it hopes to simulate, and lacks interpretability when the paths are examined. In contrast, our method is specifically trained to hit all the conditioning points over the course of the trajectory, as our generative time index is actually the same as the time series index. This is a consequence of our choice to model with a stochastic diffusion bridge, rather than their deterministic flow.
(2) Although they briefly mentioned time series data and multiple conditioning inputs, they did not further implement or explore the connection to autoregressive models, let alone the continuous time or any-subset variants.



\end{document}